\documentclass{article}

\usepackage{microtype}
\usepackage{graphicx}
\usepackage{subfigure}
\usepackage{booktabs} 
\usepackage{multirow}
\usepackage{hyperref}

\usepackage[accepted]{icml2025}

\usepackage{amsmath}
\usepackage{amssymb}
\usepackage{mathtools}
\usepackage{amsthm}
\usepackage{nicefrac}
\usepackage{tikz}
\usetikzlibrary{arrows.meta} 
\usepackage{pgfplots}

\usepackage[capitalize,noabbrev]{cleveref}

\theoremstyle{plain}
\newtheorem{theorem}{Theorem}[section]

\newtheorem{lemma}[theorem]{Lemma}
\newtheorem{corollary}[theorem]{Corollary}
\theoremstyle{definition}

\theoremstyle{remark}
\newtheorem{remark}[theorem]{Remark}
\usepackage{enumitem}
\PassOptionsToPackage{normalem}{ulem}
\usepackage{ulem}

\usepackage{listings}
\usepackage{xcolor}

\setlist[itemize]{noitemsep, left=0pt}
\definecolor{codeblue}{rgb}{0.1, 0.1, 0.6}
\definecolor{codegreen}{rgb}{0.0, 0.6, 0.0}
\definecolor{codegray}{rgb}{0.5, 0.5, 0.5}
\definecolor{codepurple}{rgb}{0.58, 0.0, 0.82}
\usepackage[textsize=tiny]{todonotes}

\icmltitlerunning{Subset-Norm and Subspace-Momentum}

\begin{document}

\twocolumn[
\icmltitle{Lean and Mean Adaptive Optimization via Subset-Norm and Subspace-Momentum with Convergence Guarantees}



\icmlsetsymbol{note}{*}
\begin{icmlauthorlist}
\icmlauthor{Thien Hang Nguyen}{neu,ctx,note}
\icmlauthor{Huy Le Nguyen}{neu}
\end{icmlauthorlist}

\icmlaffiliation{neu}{Khoury College of Computer Sciences, Northeastern University, Boston.}
\icmlaffiliation{ctx}{Contextual AI}

\icmlcorrespondingauthor{Thien Nguyen}{nguyen.thien@northeastern.edu; thien.nguyen@contextual.ai}
\icmlcorrespondingauthor{Huy Nguyen}{hu.nguyen@northeastern.edu}

\icmlkeywords{Machine Learning, ICML}

\vskip 0.3in
]



\printAffiliationsAndNotice{\textsuperscript{*}Work done at Northeastern University. } 
\global\long\def\E{\mathbb{\mathbb{E}}}%
\global\long\def\F{\mathcal{F}}%
\global\long\def\R{\mathbb{R}}%
\global\long\def\tn{\widetilde{\nabla}f}%
\global\long\def\n{\nabla f}%
\global\long\def\indicator{\mathbf{1}}%
\global\long\def\mf{f(x^{*})}%
\global\long\def\breg{\mathbf{D}_{\psi}}%
\global\long\def\dom{\mathcal{X}}%
\global\long\def\norm#1{\left\lVert #1\right\rVert }%
\global\long\def\nf{\nabla f}%
\global\long\def\eps{\epsilon}%
\global\long\def\hn{\widehat{\nabla}}%
\global\long\def\red#1{\textcolor{red}{#1}}%
\global\long\def\green#1{\textcolor[rgb]{0.0,0.5,0.3}{#1}}%
\global\long\def\blue#1{\textcolor{blue}{#1}}%
\providecommand{\tabularnewline}{\\}

\begin{abstract}
We introduce two complementary techniques for efficient optimization
that reduce memory requirements while accelerating training of
large-scale neural networks. The first technique, \emph{Subset-Norm} step size, generalizes AdaGrad-Norm and AdaGrad(-Coordinate) through step-size sharing. Subset-Norm (SN) reduces
AdaGrad's memory footprint from $O(d)$ to $O(\sqrt{d})$, where $d$ is the model size. For non-convex
smooth objectives under coordinate-wise sub-gaussian noise, we show
a noise-adapted high-probability convergence guarantee with improved
dimensional dependence of SN over existing methods. Our second technique,
\emph{Subspace-Momentum}, reduces the momentum state's memory footprint by
restricting momentum to a low-dimensional subspace while performing SGD in the
orthogonal complement. We prove a high-probability convergence result for Subspace-Momentum under
standard assumptions. Empirical evaluation on pre-training and fine-tuning LLMs  demonstrates the effectiveness of our methods. For instance, combining Subset-Norm with Subspace-Momentum achieves Adam's validation perplexity for LLaMA 1B in approximately \textit{half} the training
tokens (6.8B vs 13.1B) while reducing Adam's optimizer-states memory footprint by more than 80\% with minimal additional hyperparameter tuning. 
\end{abstract}

\section{Introduction}

\begin{figure}[t]
\centering
\includegraphics[width=.95\columnwidth]{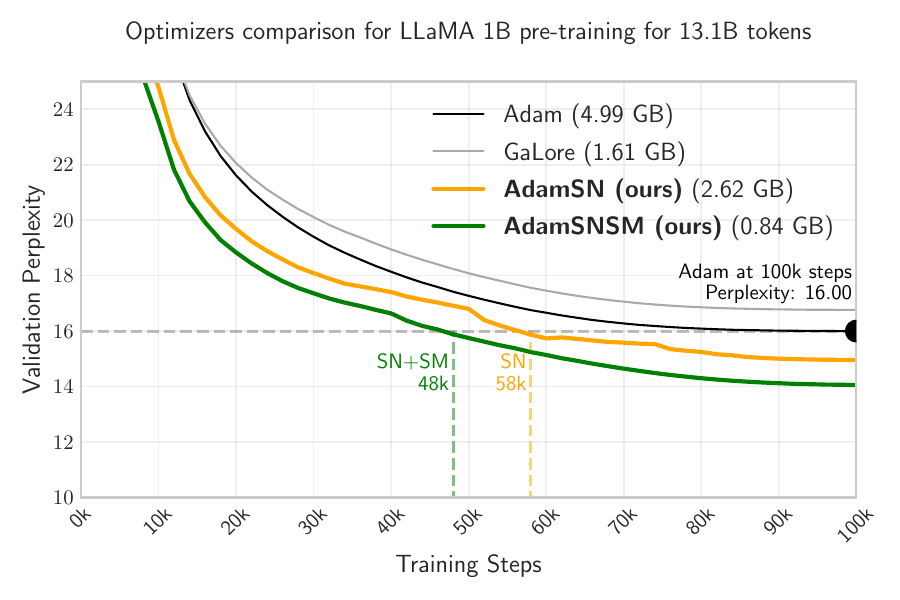}
\caption{Validation perplexity for Adam, GaLore \cite{zhao2024galore},
\textbf{AdamSN}, and \textbf{AdamSNSM} (\textbf{ours}) during LLaMA 1B model training for 13.1B
tokens (100K steps). Optimizer memory footprint is shown in parentheses. Adam achieves a perplexity
of 16.00 at 100,000 steps, while AdamSN and AdamSNSM exhibit lower
perplexity earlier in training at 58,000 and 48,000 steps.
\label{fig:speed-comparison-learning-curve}}
\end{figure}

Adaptive optimizers like Adam \cite{kingma2014adam}, AdaGrad \cite{duchi2011adaptive},
and RMSProp \cite{tieleman2012lecture} are de facto methods for training
large-scale deep neural networks. However, the optimizer states for
the momentum and second moment (or adaptive step size) terms are memory
intensive, consuming as much as twice the size of the model. As deep
neural networks continue to grow in the era of large-language models
(LLMs), concerns that were previously overlooked, such as the memory
consumption of optimizer states, have become an active area of research.
Indeed, numerous methods have recently emerged to reduce the memory
footprint of optimizer states (e.g. Adam's momentum and second moment
terms) with approaches ranging from quantization \cite{li2024memory,dettmers20218,dettmers2024qlora},
low-rank decomposition \cite{hu2021lora,lialin2023relora,zhao2024galore,shazeer2018adafactor},
sketching-based dimensionality reduction \cite{muhamed2024grass,hao2024flora},
etc. Existing methods either lacks theoretical guarantees, requires
strong assumptions, trades too much performance, or requires expensive
additional tuning for the memory saving, especially in pretraining
tasks. 


\paragraph{Our contributions.}

We aim to reduce memory consumption while maintaining strong performance and theoretical guarantees. To this end, we introduce two memory-efficient optimization algorithms for large-scale
DNN training: \textbf{Subset-Norm} (\textbf{SN}) for adaptive
step-size memory reduction (Section \ref{sec:AdaGrad-Subset-Norm}) and \textbf{Subspace-Momentum} (\textbf{SM})
for momentum compression (Section \ref{sec:Subspace-Momentum}). While existing approaches trade performance
for memory savings, our theoretically-grounded methods achieve both
a reduced memory footprint and faster training:
\begin{itemize}[topsep=0pt]
\item \textbf{Subset-Norm (SN)}: A memory-efficient adaptive step-size algorithm
with high-probability convergence guarantees for non-convex objectives
under coordinate-wise sub-gaussian noise. By unifying AdaGrad-Coordinate's
and AdaGrad-Norm's analysis, we show that the SN adaptive step size
(Algorithm \ref{alg:adagrad-subsetnorm}) achieves improved dimensional
dependence, while reducing the memory footprint from $O(d)$ to roughly $O(\sqrt{d})$.
On LLaMA models' pretraining tasks, SN step sizes achieves better perplexity than
coordinate-wise step size across a range of optimizers and model sizes, while using significantly less
memory and introducing minimal additional hyperparameters.\footnote{Although the subset size can be tuned (Section \ref{subsec:Subset-size-ablation}),
we provide a heuristic in Section \ref{subsec:SN-Implementation} that
works effectively across model sizes, eliminating the need for additional
tuning.}
\item \textbf{Subspace-Momentum (SM)}: A momentum compression method that
applies momentum in a chosen subspace and SGD in the orthogonal complement
with high-probability convergence guarantees under sub-gaussian
noise for non-convex smooth objectives. When combined with SN, for some selected dimension$k$ less than $d$,\footnote{Typically, $k$ is chosen to be around $d/4$.\label{fn:rank-galore}}  our
method (SNSM) reduces the memory footprint of Adam and AdaGrad+momentum from $2d$
to $k+\sqrt{d}$ (see Table \ref{tab:Optimizer-states-memory})
while delivers improved training speed and performance.
\end{itemize}
Empirical evaluations on LLaMA models from 60M to 1B parameters
demonstrate that our algorithms scale effectively and attain better
performances than existing optimizers. 
Our proposed methods are simple
to implement, require minimal additional hyperparameter tuning, and are compatible with modern distributed training frameworks like FSDP \cite{zhao2023pytorch,rajbhandari2020zero}.
We provide an implementation in PyTorch at \url{https://github.com/timmytonga/sn-sm}.

\section{Preliminaries}

 \subsection{Common Optimizers and Memory Footprint}
\begin{algorithm}[tb]
   \caption{Generic Template for Stochastic Adaptive Optimizers with Momentum}
   \label{alg:Generic-Adaptive-Optimizer}
\begin{algorithmic}
   \STATE {\bfseries Input:} Initial point $x_{1} \in \mathbb{R}^{d}$, base step size $\eta > 0$, constant $\epsilon > 0$.
   \FOR{$t=1$ {\bfseries to} $T$}
      \STATE Obtain stochastic gradient $\hn f(x_t)$
      \STATE $m_t = \text{update\_momentum}(\hn f(x_t); m_{t-1})$ 
      \STATE $v_t^2 = \text{update\_adaptive\_stepsize}(\hn f(x_t); v_{t-1}^2)$
      \STATE $x_{t+1} = x_t - \eta \cdot \frac{m_t}{v_t + \epsilon}$ \hfill \COMMENT{Update step}
   \ENDFOR
\end{algorithmic}
\end{algorithm}

\begin{table}[ht]
\centering
\scriptsize 
\caption{Update rules for common optimizers in the framework of Algorithm \ref{alg:Generic-Adaptive-Optimizer}.
We omit bias correction terms and numerical stabilizer $\epsilon$
for simplicity. Memory for optimizer state is shown for model of size (and memory footprint)
$d$. \label{tab:Update-rules-framework}}
\begin{tabular}{lll}
\toprule
\textbf{Optimizer}  & \textbf{Memory} & \textbf{Update Rules}  \tabularnewline
\midrule 
\multirow{2}{*}{Adam} & \multirow{2}{*}{$2d$} & $m_t =  \beta_{1}m_{t-1}+(1-\beta_{1})\hn f(x_{t})$ \\
 & & $v_t^2 =  \beta_{2}v_{t-1}^{2}+(1-\beta_{2})\cdot\hn f(x_{t})^{2}$ \\ \midrule
\multirow{2}{*}{SGDm} & \multirow{2}{*}{$d$} & $m_t =  \beta m_{t-1}+(1-\beta)\hn f(x_{t})$ \\
 & & $v_t^2 =  \text{ID}$ \\ \midrule
\multirow{2}{*}{AdaGrad-Coord} & \multirow{2}{*}{$d$} & $m_t =  \hn f(x_{t})$ \\
 & & $v_t^2 =  v_{t-1}^{2}+\hn f(x_{t})^{2}$ \\ \midrule
\multirow{2}{*}{RMSProp} & \multirow{2}{*}{$d$} & $m_t =  \hn f(x_{t})$ \\
 & & $v_t^2 =  {\beta_{1}v_{t-1}^{2}+(1-\beta_{1})\cdot\hn f(x_{t})^{2}}$ \\ \midrule
 \multirow{2}{*}{AdaGrad-Norm} & \multirow{2}{*}{$1$} & $m_t =  \hn f(x_{t})$ \\
 & & $v_t^2 =  v_{t-1}^{2}+\norm{\hn f(x_{t})}^{2}$ \\ \midrule
\multirow{2}{*}{SGD} & \multirow{2}{*}{$1$} & $m_t =  \hn f(x_{t})$ \\
 & & $v_t^2 =  \text{ID}$ \\
\bottomrule 
\end{tabular}
\end{table}

Consider the generic template in Algorithm \ref{alg:Generic-Adaptive-Optimizer},
which captures a broad range of first-order optimizers that leverage
either momentum or adaptive step sizes. Many standard optimizers can
be represented within this framework by varying choices of momentum and adaptive step-size terms, as shown in Table \ref{tab:Update-rules-framework}. Generally,
optimizers with higher memory requirements, such as Adam, tend to
outperform more memory-efficient alternatives like SGD and RMSProp. We aim to design principled algorithms that achieve the best of both worlds: strong performance and memory-efficient.

\subsection{Assumptions and Notations}
For our theoretical analysis, we consider the unconstrained non-convex stochastic optimization problem
$\min_{x\in\R^{d}}f(x)$ where $f:\R^{d}\to\R$ is the objective function.
We assume access to an history independent, non-biased gradient estimator
$\hn f(x)$ for any $x\in\dom$, that is $\E\left[\hn f(x)\mid x\right]=\n(x)$.
Furthermore, we assume that $f$ is an $L$-smooth
function: $$\left\Vert \n(x)-\n(y)\right\Vert \le L\left\Vert x-y\right\Vert, 
\text{ for all } x,y\in\R^{d}. $$
Smoothness implies the following quadratic
upperbound that we will extensively utilize: for all $x,y\in\R^{d}$ we have $f(y)-f(x)\le\left\langle \nabla f(x),y-x\right\rangle +\frac{L}{2}\left\Vert y-x\right\Vert ^{2}.$

\paragraph{Notations.}

Let $v_{i}$ denote the $i$-th coordinate of a vector $v\in\R^{d}$.
If a vector $x_{t}$ is already indexed as part of a sequence of vectors
(where $x_{t}$ denotes the $t$-th update) then we use $x_{t,i}$
to denote $x_{t}$'s $i$-th coordinate and $x_{t,\Psi}\in\R^{k}$
to denote the indexing with respect to an ordered subset $\Psi\subseteq[d]$
of size $k$ where $\left(x_{t,\Psi}\right)_{k}=x_{t,\Psi^{(k)}}$
with $\Psi^{(k)}$ denoting the $k$-th element of $\Psi$. For gradients,
we let $\nabla_{i}f(x):=\frac{\partial f}{\partial x_{i}}$ denote
the partial derivative with respect to the $i$-th coordinate. Similarly,
for stochastic gradients $\hn f(x)$, we let $\hn_{i}f(x)$ denotes
its $i$-th coordinate. If $a,b\in\R^{d}$, then $ab$ and $a/b$
denotes coordinate-wise multiplication and division, respectively
i.e. $(ab)_{i}=a_{i}b_{i}$ and $(a/b)_{i}=a_{i}/b_{i}$.

\paragraph{Coordinate-wise sub-gaussian noise assumption.}
A random variable $X$ is
$\sigma$-sub-Gaussian \cite{vershynin2018high} if
\[
\E\left[\exp\left(\lambda^{2}X^{2}\right)\right]\leq\exp\left(\lambda^{2}\sigma^{2}\right)\text{ for all }\lambda\text{ such that }\left|\lambda\right|\leq\frac{1}{\sigma}.
\]
If we denote the stochastic gradient noise as $\xi_{t}:=\widehat{\nabla}f(x_{t})-\nabla f(x_{t})$
and $\xi_{t,i}$ as the $i$-th coordinate of $\xi_{t}$, then we
assume the noise is per-coordinate subgaussian i.e. there exists $\sigma_{i}>0$
for $i\in[d]$ such that $\xi_{t}$ satisfies
\begin{equation}
\E\left[\exp\left(\lambda^{2}\xi_{t,i}^{2}\right)\right]\leq\exp\left(\lambda^{2}\sigma_{i}^{2}\right),\forall\left|\lambda\right|\leq\frac{1}{\sigma_{i}},\forall i\in\left[d\right].\label{eq:coordinate-subgaussian-noise}
\end{equation}
Note that $\left\Vert \xi_{t}\right\Vert $ being $\sigma$-subgaussian
implies that each $\xi_{t,i}$ is also $\sigma$-subgaussian, so coordinate-wise
sub-gaussian is more general than standard scalar sub-gaussian noise
assumption. Furthermore, when $\norm{\cdot}$ is used without explicitly
specifying the norm, we assume it is the $\ell_{2}$ norm $\norm{\cdot}_{2}$.
We also use the 0-indexing convention i.e. $[n]:=\left\{ 0,1,\dots,n-1\right\} $
for integer $n\in\mathbb{N}$. 

\section{Subset-Norm (SN) Adaptive Step Size\label{sec:AdaGrad-Subset-Norm}}

\begin{figure}[ht]
    \centering
    \includegraphics[width=0.9\columnwidth]{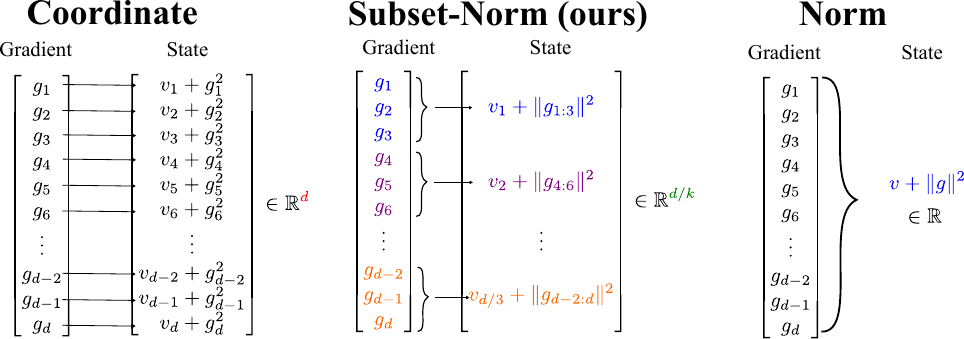}
    \caption{AdaGrad variants: Coordinate, Subset-Norm, and Norm. Subset-Norm generalizes Coordinate ($k=1$) and Norm ($k=d$).}
    \label{fig:adagrad-variants}
\end{figure}
\begin{algorithm}[ht]
\small
   \caption{SGD with Subset-Norm Adaptive Step Size}
   \label{alg:adagrad-subsetnorm}
\begin{algorithmic}
   \STATE {\bfseries Input:} Initial point $x_1 \in \mathbb{R}^{d}$, base step size $\eta > 0$, partition function $\psi:[d] \to [c]$ that splits the coordinates into $c$ subsets $\Psi_i = \psi^{-1}(i) \subset [d]$, where $\coprod_{i=1}^{c} \Psi_i = [d]$, and $b_{0,i} > 0$ for $i \in [c]$
   \FOR{$t=1$ {\bfseries to} $T$}
      \STATE Obtain stochastic gradient $\hn f(x_t)$
      \STATE $b_{t,i}^2 = b_{t-1,i}^2 + \norm{\hn_{\Psi_i} f(x_t)}^2, \quad \text{for } i \in [c]$ \hfill \COMMENT{Update accumulated gradient norms}
      \STATE $x_{t+1,k} = x_{t,k} - \frac{\eta}{b_{t,\psi(k)}} \hn_k f(x_t), \quad \text{for } k \in [d]$ 
   \ENDFOR
\end{algorithmic}
\end{algorithm}

We compress the second moment adaptive step size by partitioning parameters
into subsets for which they share the same adaptive step size as AdaGrad-Norm
\cite{McMahanS10,ward2019adagrad}. Formally, we need to specify a
\textit{partition function} $\psi:[d]\twoheadrightarrow[c]$ that splits
the $d$ coordinates into $c$ non-empty subsets $\Psi_{i}=\psi^{-1}(i)\subset[d]$,
where $\coprod_{i=1}^{c}\Psi_{i}=[d]$. For example, one can pick
$\psi(j)=(j/c)\mod k$ to get consecutive equipartitioned subsets $\Psi_{i}=\left\{ ik,ik+1,\dots,ik+(k-1)\right\} $
for some subset-size $k\in\mathbb{N}$ so that $kc=d$.\footnote{We use this strategy in all our implementations for simplicity.}

Given a stochastic gradient $\hn f(x_{t})\in\R^{d}$ at time $t$
for parameter $x_{t}$, we denote $\hn_{\Psi_{i}}f(x_{t})\in\R^{k}$
to be the subset of the coordinates of the stochastic gradient with
respect to the subset $\Psi_{i}$. For example, given $\psi(j)=(j/c)\mod k$
as above, we have $\left(\hn_{\Psi_{i}}f(x_{t})\right)_{j}=\hn_{ik+j-1}f(x_{t})$.
Similarly, we can define $\nabla_{\Psi_{i}}f(x_{t})\in\R^{|\Psi_{i}|}$
to be $\frac{\partial f(x_{t})}{\partial x_{\Psi_{i}}}$. 

Now, we define
the \emph{subset-norm (SN) adaptive step size} $b_{t,i}$ for subset $\Psi_{i}$
and the update rule for $x_{t+1}$ (see Figure \ref{fig:adagrad-variants}): 
\begin{align}
b_{t,i}^{2} & =b_{t-1,i}^{2}+\norm{\hn_{\Psi_{i}}f(x_{t})}^{2},
\ \text{for }i\in [c] \\
x_{t+1,j} & =x_{t,j}-\frac{\eta}{b_{t,\psi(j)}}\hn_{j}f(x_{t}),\ \text{for }j \in [d].\label{eq:adagrad-subset-norm}
\end{align}
Note that choosing $c=d$ and $c=1$ recovers AdaGrad-Coordinate and
AdaGrad-Norm, respectively. We now show a convergence guarantee on arbitrary partitions that will inform us on how to select a good partition strategy. 

\subsection{High-Probability Convergence of Subset-Norm }

We show the following high-probability convergence result for the
subset-norm adaptive step size:
\begin{theorem}
\label{thm:main-thm-simplified}Suppose that $f:\R^{d}\rightarrow\R$
is $L$-smooth and lower bounded by $f_{*}$. Given unbiased stochastic
gradients $\widehat{\nabla}f(x_{t})$ with stochastic gradient noise
$\xi_{t}:=\widehat{\nabla}f(x_{t})-\nabla f(x_{t})$ that is $\sigma_{i}$-per-coordinate
subgaussian for $i\in[d]$. For partitions of the parameters into $c\in \mathbb{N}_+$
disjoint subsets $[d]=\bigcup_{i=0}^{c-1}\Psi_{i}$ with $\Psi_{i}\cap\Psi_{j}=\emptyset,\ \text{for }i\neq j$,
the iterates $x_{t}$ given by Algorithm \ref{alg:adagrad-subsetnorm}
satisfies the following with probability at least $1-O(c\delta)$
(for failure probability $\delta>0$)
\begin{align*}
&\frac{1}{T}\sum_{t=1}^{T}\norm{\nabla f(x_{t})}_{2}^{2} \le G(\delta)\cdot \tilde{O}\biggl(  \frac{\sum_{i=0}^{c-1}\norm{\sigma_{\Psi_{i}}}_2}{\sqrt{T}}+ \frac{N(\delta)}{T}\biggl),\\
&\text{where }  \\
&G(\delta)  :=\tilde{O}\biggl(\sum_{i=0}^{c-1}\norm{\sigma_{\Psi_{i}}}_2^{4} +\norm{\sigma}_\infty (\norm{\sigma}_{2}^{2}+c^{3/2})+cL\biggl), \\ 
&N(\delta) := \norm{\sigma}_{2}^{2}+ \sum_{i=0}^{c-1}\norm{\sigma_{\Psi_{i}}}_2+Lc.
\end{align*}
\end{theorem}
Polylog terms are hidden in Theorem \ref{thm:main-thm-simplified}
for simplicity. The full result, Theorem \ref{thm:full-thm}, and
proofs are presented in Appendix \ref{sec:Full-Theorem-and}. Theorem
\ref{thm:main-thm-simplified} provides guarantee for all partitions
of the parameters into arbitrary disjoint subsets and generalizes
AdaGrad-Norm ($c=1$) and AdaGrad-Coordinate ($c=d$) results. The
result is noise-adapted: if $\sum_{i=0}^{c-1}\norm{\sigma_{\Psi_{i}}}_2$
is small enough, the rate becomes the optimal deterministic rate of
$O(\frac{1}{T})$ regardless of the base step size $\eta$. The next section explores implications of Theorem
\ref{thm:main-thm-simplified} and strategies for selecting subsets.

\subsection{Coordinate-Noise Density and Subset-Norm's Improved Dimensional Dependency \label{subsec:Coordinate-noise-sparsity-and}}

Theorem \ref{thm:main-thm-simplified} presents trade-offs between
the partition strategies and the stochastic gradient noise, where we need to balance between the number of subsets $c$ and noise-reduction benefits of parameters-grouping e.g., $\norm{x}_2 \le \norm{x}_1$. 

\paragraph{Coordinate-noise density $d^{\beta}$.}

To make the intuition above concrete, consider a scenario with various
coordinate-noise density rate: fix a rate $\beta\in[0,1]$, some $d^{\beta}$
coordinates have noise $\alpha>0$ while the rest are $0$. The rate $\beta$ controls the density of coordinate noise. When $\beta=0$,
only 1 coordinate have noise. When $\beta=1$, all coordinates have
noise. To get a feel for $\beta$'s relationship to the fraction of coordinates containing noise, half the coordinates contain noise when $\beta\approx0.96$ for $d=60\text{M}$ and $\beta\approx 0.97$ for $d=10\text{B}$ and $\beta\approx 0.98$ for $d=10^{15}$ (See also Figure \ref{fig:coordinate-noise-density-over-fraction-of-noisy-coords}). See Figure \ref{fig:noise-density-10} for noise density of LLaMA 60M (details in Appendix \ref{subsec:Empirical-validation-noise-sparse}). 
Furthermore, $\alpha$ upper bounds all coordinate noise, i.e. $\|\sigma\|_{\infty}\le\alpha$,
which is common in coordinate-wise analysis \cite{defossez2020simple}. 

\paragraph{Derivation of convergence rate given coordinate noise density $d^{\beta}$. }

Given $\beta\in[0,1]$, we can obtain a concrete expression for the
convergence rates of various methods (different subset sizes) from
Theorem \ref{thm:main-thm-simplified}. For SGD with Subset-Norm,
we consider an \textit{equipartition strategy}, where we divide
the coordinates into $c=d^{1-\beta}k$ subsets of size $d^{\beta}/k$
each with the $d^{\beta}$ noisy coordinates into just $k$ subsets
so that the rest of the $c-k$ subsets have no noisy coordinate. We
defer the derivation details to Appendix \ref{sec:coordinate-noise-sparsity-derivation}
and summarize the results in the first row of Table \ref{tab:dimension-convergence-rate}. 

\begin{table*}[ht]
\centering
\caption{Algorithms comparison between dimensional dependencies and convergence
rates under different coordinate-noise density settings. Given a density
rate $\beta$, convergence rates' dimensional dependency are highlighted
in red and green to denote the worst and best dependency on the dimension.
Note that memory usage of AdaGrad-Coordinate is $O(d)$ while SGD with Subset-Norm
(with the partition strategy presented here) is $O(d/k)$, where $k=d^{1.4\beta-0.6}$
is chosen as an optimal noise dependent subset size. \label{tab:dimension-convergence-rate}}
\begin{tabular}{llll}
\hline 
Density rate & AdaGrad-Coordinate & AdaGrad-Norm & Subset-Norm (equipartition subsets)\tabularnewline
\hline 
\multirow{2}{*}{$\beta\in[0,1]$} & \multirow{2}{*}{$\tilde{O}\left(\nicefrac{d^{1.5+\beta}}{\sqrt{T}}+\nicefrac{d^{2.5}}{T}\right)$} & \multirow{2}{*}{$\tilde{O}\left(\nicefrac{d^{2.5\beta}}{\sqrt{T}}+\nicefrac{d^{3\beta}}{T}\right)$} & $\tilde{O}\left(\nicefrac{d^{0.3+1.8\beta}}{\sqrt{T}}+\nicefrac{d^{\beta+1}}{T}\right)\text{if }\text{\ensuremath{\beta\in[0,\nicefrac{2}{3}]}}$\tabularnewline
 &  &  & $\tilde{O}\left(\nicefrac{d^{0.3+1.8\beta}}{\sqrt{T}}+\nicefrac{d^{1.6\beta+0.6}}{T}\right)\text{if }\beta\in[\nicefrac{2}{3},1]$\tabularnewline
\hline 
$\beta=0$ & $\tilde{O}\left(\nicefrac{\red{d^{1.5}}}{\sqrt{T}}+\nicefrac{\red{d^{2.5}}}{T}\right)$ & \textbf{$\tilde{O}\left(\nicefrac{\green 1}{\sqrt{T}}+\nicefrac{\green 1}{T}\right)$} & $\tilde{O}\left(\nicefrac{d^{0.3}}{\sqrt{T}}+\nicefrac{d}{T}\right)$\tabularnewline
$\beta=0.5$ & $\tilde{O}\left(\nicefrac{\red{d^{2}}}{\sqrt{T}}+\nicefrac{\red{d^{2.5}}}{T}\right)$ & $\tilde{O}\left(\nicefrac{d^{1.25}}{\sqrt{T}}+\nicefrac{\green{d^{1.5}}}{T}\right)$ & $\tilde{O}\left(\nicefrac{\green{d^{1.2}}}{\sqrt{T}}+\nicefrac{\green{d^{1.5}}}{T}\right)$\tabularnewline
$\beta=0.9$ & $\tilde{O}\left(\nicefrac{\red{d^{2.4}}}{\sqrt{T}}+\nicefrac{d^{2.5}}{T}\right)$ & $\tilde{O}\left(\nicefrac{d^{2.25}}{\sqrt{T}}+\nicefrac{\red{d^{2.7}}}{T}\right)$ & $\tilde{O}\left(\nicefrac{\green{d^{1.92}}}{\sqrt{T}}+\nicefrac{\green{d^{2.04}}}{T}\right)$\tabularnewline
$\beta=1$ & $\tilde{O}\left(\nicefrac{\red{d^{2.5}}}{\sqrt{T}}+\nicefrac{d^{2.5}}{T}\right)$ & $\tilde{O}\left(\nicefrac{\red{d^{2.5}}}{\sqrt{T}}+\nicefrac{\red{d^{3}}}{T}\right)$ & $\tilde{O}\left(\nicefrac{\green{d^{2.1}}}{\sqrt{T}}+\nicefrac{\green{d^{2.2}}}{T}\right)$\tabularnewline
\hline 
\end{tabular}
\end{table*}

\paragraph{Subset Selection.}
In Table \ref{tab:dimension-convergence-rate}, the equal subset-size
partition strategy for Subset-Norm has better dependency on the dimension
$d$ when the noise is not completely sparse i.e. $\beta=0$. Hence,
if we expect the actual noise density $\beta$ to be around\footnote{Figure \ref{fig:Aggragated-noise-density} shows that overall noise
is quite sparse but varies more when limited to a particular layer as in Figure \ref{fig:noise-density-10}.
See Section \ref{subsec:Subset-size-ablation} for more experiments on subset size selection.} $0.75$ to $0.90$, then compressing with a subset size of around
$d^{0.45}$ to $d^{0.66}$ is optimal. The dependency on $d$ is important
for modern neural network, since the number of parameters $d$ is
typically greater than or on the same order as the total number of iterations $T$. 
\begin{figure}[ht]
\begin{center}
\includegraphics[width=0.95\columnwidth]{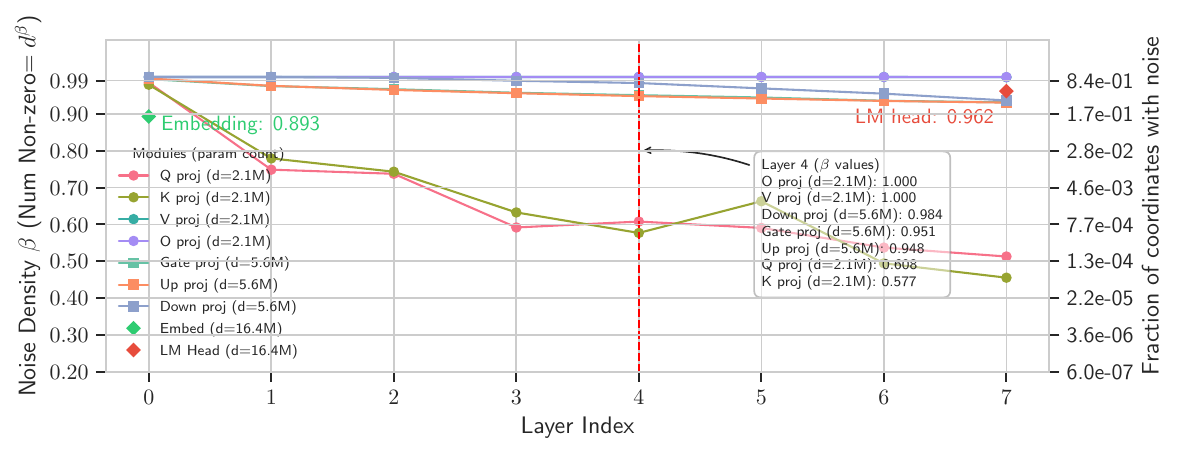}
\par\end{center}
\caption{Noise density per parameter across layers for LLaMA 60M on pre-training task after 100
steps. \label{fig:noise-density-10}}
\end{figure}

\paragraph{Subset-size heuristics to avoid additional hyperparameters.}\label{par:subset-size-heuristics}

Providing a useful and robust default setting for an algorithm is  important to justify claims of reduced costs. We provide
a simple partitioning scheme for SN for  2D parameters $p\in\R^{m\times n}$: we simply group along the smaller dimension. For example, if $p$ is of shape $(2048, 1024)$, we group by rows to get a state of size $2048$. 
This is a natural grouping scheme that groups the latent dimension together and aims for the rough $d^{0.45}$
subset size discussed in the previous paragraph.
Another simplification is that subset-norm is applied only on \emph{linear}
modules, since 2D linear modules makes up the vast majority of parameters
in transformers. This means we compress all the attention, MLP, and
final LM head weights. This implementation is presented in more details
in Appendix \ref{sec:Adam-Subset-Norm-Implementation}. Section \ref{subsec:Subset-size-ablation}
shows that this heuristic grouping, while simple, is not optimal and can
be improved by tuning the subset size, but we opt for simplicity over performance-tuning in our experiments. 

\textbf{Generic Implementation. \label{subsec:SN-Implementation}}  
We provide pseudocode for the generic equipartition strategy of Algorithm \ref{alg:adagrad-subsetnorm}
in Section \ref{subsec:Generic-Subset-Norm-Adaptive} that we use for
the subset sizes ablations in Section \ref{subsec:Subset-size-ablation}.

Furthermore, in contrast to methods like AdaFactor or GaLore that
are limited to 2D parameters, the generic subset-norm algorithm is coordinate-wise 
and admits an easy implementation to FSDP \cite{zhao2023pytorch,rajbhandari2020zero}, where parameters are flattened
to 1D tensors for efficient communication.

\section{Subspace-Momentum \label{sec:Subspace-Momentum}}
\begin{figure}[ht]
\centering
    \includegraphics[width=0.6\columnwidth]{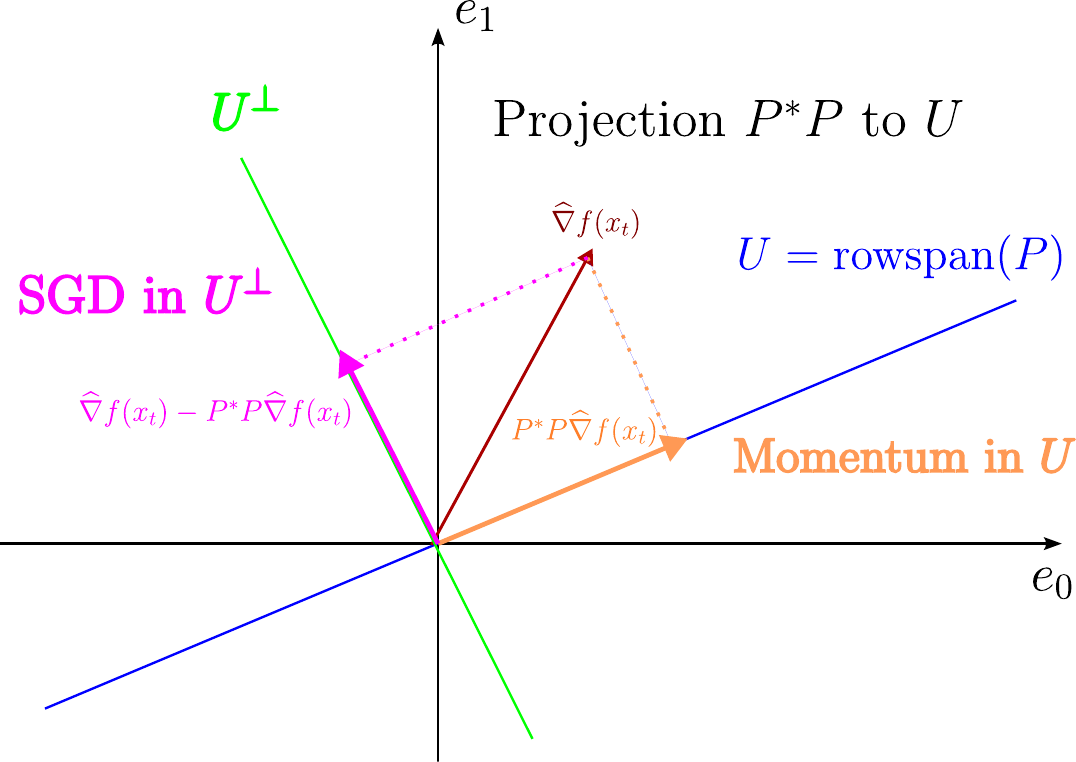}
    \caption{Subspace Momentum Illustration.}
    \label{fig:subspace-momentum}
\end{figure}
Existing algorithmic compression approaches like GaLore \cite{zhao2024galore},
GRASS \cite{muhamed2024grass}, and FLORA \cite{hao2024flora} project
the gradient to a lower dimensional space $\R^{k}$ for updating the
optimizer state via some bounded linear operator $P:\R^{d}\rightarrow\R^{k}$
such that $P^{*}P:\R^{d}\rightarrow\R^{d}$ is a projection, i.e. $(P^{*}P)^2=P^{*}P$, where
$P^{*}:\R^{k}\rightarrow\R^{d}$ is the adjoint operator of $P$. More concretely,
given a stochastic gradient $\hn f(x_{t})\in\R^{d}$ at time $t$,
a low-dimensional version $c_{t}:=P\hn f(x_{t})\in\R^{k}$ is computed
that is used to update the states before projecting back to $\R^{d}$
for update:
\begin{align}\label{eq:galore-update}
m_{t}&=\beta_{1}m_{t-1}+\left(1-\beta_{1}\right)c_{t} \notag \\ 
v_{t}^{2}&=\beta_{2}v_{t-1}^{2}+\left(1-\beta_{2}\right)c_{t}^{2}  \\ 
x_{t+1}&=x_{t}-P^{*}\left(m_{t}/v_{t}\right).\notag
\end{align}
This update performs adaptive optimization in the row
span $U\subseteq\R^{d}$ of $P$ when viewed as a linear operator,
with $\dim(U)=k$. For example, GaLore \cite{zhao2024galore} utilizes
the top $k$ singular vectors of stochastic gradients,
and FLORA \cite{hao2024flora} simply projects to a random subspace
using dense Gaussian matrices. Due to the optimization operating 
in a low rank subspace, convergence is not guaranteed unless stronger
conditions are assumed. 

We propose Subspace Momentum (SM) -- presented in Algorithm \ref{alg:Subspace-momentum} and illustrated in Figure \ref{fig:subspace-momentum} -- where SM guarantees convergence by incorporating the orthogonal
complement of $P^{*}P\hn f(x_{t})\in U$ that lives in the orthogonal
complement $U^{\perp}$ of $U$ (with $U\oplus U^{\perp}=\R^{d}$). We can compute the orthogonal complement of $\hn f(x_{t})$ via $\left(\hn f(x_{t})-P^{*}P\hn f(x_{t})\right)\in U^{\perp}$.

\begin{algorithm}[tb]
\small
   \caption{SGD with Subspace Momentum (SM)}
   \label{alg:Subspace-momentum}
\begin{algorithmic}
   \STATE {\bfseries Input:} Bounded Linear Operator $P : \mathbb{R}^d \to \mathbb{R}^k$, such that $P^*P$ is a projection, where $P^*$ is $P$'s adjoint.
   \FOR{$t = 1, 2, \ldots, T$}
      \STATE Obtain stochastic gradient $\hn f(x_t)$
      \STATE $m_t = \beta_1 m_{t-1} + (1 - \beta_1) P \hn f(x_t)$ \hfill \COMMENT{Momentum in subspace $U=\text{rowspan}(P)$}
      \STATE $r_t = \hn f(x_t) - P^* P \hn f(x_t)$ \hfill \COMMENT{Component in $U^{\perp}$}
      \STATE $x_{t+1} = x_t - \eta (P^* m_t + r_t)$ \hfill \COMMENT{Step in both spaces}
   \ENDFOR
\end{algorithmic}
\end{algorithm}

Subspace Momentum maintains the
same memory footprint, $O(k)$, as existing low-rank optimizers. However, SM's
 update step is full rank: it uses momentum only in $U:=\text{rowspan}(P)$
while performs SGD in $U^{\perp}$. Unlike joint compression techniques
like GaLore (\ref{eq:galore-update}), SM only affects the momentum
term. Hence, SM is modular and fits into the framework of Algorithm
\ref{alg:Generic-Adaptive-Optimizer}: there, we can combine it
with different adaptive step sizes such as subset-norm.\footnote{Section \ref{subsec:Step-sizes-and-momentum} contains a detailed ablation on different momentum and adaptive step sizes combinations. } 


\subsection{High-Probability Convergence of Subspace-Momentum \label{subsec:Convergence-of-Subspace-Momentum}}

We show that Subspace-Momentum, Algorithm \ref{alg:Subspace-momentum},
converges with high-probability under the standard assumptions of smoothness
and $\sigma$-subgaussian gradient noise:
\begin{theorem}
\label{thm:SGDSM-highprob}Suppose that $f:\R^{d}\rightarrow\R$ is
$L$-smooth and lower bounded by $f_{*}$. Assume unbiased stochastic
gradients $\widehat{\nabla}f(x_{t})$ with $\sigma$-subgaussian stochastic
gradient noise. Then, the iterates $x_{t}$ given by SGD with Subspace-Momentum
(Algorithm \ref{alg:Subspace-momentum}) with step size $\eta=\min\left\{ \frac{1}{2\alpha}; \sqrt{\frac{\Delta_1}{\sigma^2 \alpha T}} \right\}$
for $\alpha:=\frac{(3-\beta)L}{2\left(1-\beta\right)}$ satisfies
the following with probability at least $1-\delta$
\[
\frac{1}{T}\sum_{i=1}^{T}\norm{\nabla_{t}}_2^{2}\le\frac{8\Delta_{1}\alpha}{T}+\frac{7\sigma\sqrt{\alpha\Delta_{1}}}{\sqrt{T}}+\frac{48\sigma^{2}\log\left(1/\delta\right)}{T},
\]
where $\Delta_{1}:=f(x_{1})-f_{*}$ is the initial function gap.
\end{theorem}

We observe that Theorem \ref{thm:SGDSM-highprob} has a similar
rate to vanilla SGD. Unlike adaptive algorithms, we need to know the problem parameters to get an adaptive convergence rate. The proof is presented in Appendix \ref{subsec:Subspace-Momentum-Convergence-proof},
where we also provide some intuition for the algorithm.

\subsection{Subspace Selection, Subspace Switching, and Projection Updates}


In our experiments, we use the top-$k$ singular vectors of a stochastic gradient snapshot as our main subspace, similarly to GaLore \cite{zhao2024galore}. 

Algorithm \ref{alg:Subspace-momentum} and the accompanying theory
in Section \ref{subsec:Convergence-of-Subspace-Momentum} are only
for a fixed projection. However, from our experiments, we find that
performing subspace switching every $G$ steps (as in GaLore) can be
beneficial, especially for smaller ranks. Section \ref{subsec:SM-rank-gap-ablation}
contains an ablation studies on this. We incorporate
projection updates in our main algorithms by picking a projection
update gap and then fully resetting the momentum term to \emph{zero}
when we switch (in contrast to GaLore's accumulated statistics when
switching subspace).

\section{Experiments\label{sec:Experiments}}

We evaluate Subset-Norm (SN) and Subspace-Momentum (SM) on LLM pretraining and supervised fine-tuning 
tasks, where memory is often a bottleneck. We compare against several
baselines, with memory estimates given for parameters of size
$m\times n$, where we assume WLOG $m\geq n$.

 \textbf{Baselines.} We consider \textbf{AdaGrad} \cite{duchi2011adaptive}, \textbf{AdaGradm} where we incorporate momentum $0.9$ to AdaGrad, \textbf{Adam} \cite{kingma2014adam}, and \textbf{RMSProp} \cite{tieleman2012lecture} as standard optimizers. We also consider \textbf{GaLore} \cite{zhao2024galore} as a recent memory-efficient method that projects the optimizer states into a low-rank
subspace (typically rank $n/4$), using $2(mn/4)$ memory but requiring
6 hyperparameters including subspace rank, projection update frequency,
and scaling parameters.

\textbf{Our methods.} We incorporate SN and SM to AdaGrad, AdaGradm, Adam and RMSProp. \textbf{SN}
reduces the adaptive step size (e.g. Adam's second moment term) memory from $mn$ to $m$ for a parameter of size $m\times n$.
\textbf{SNSM} further compresses the momentum term of momentum methods like Adam and AdaGradm by
adding SM with SVD at the cost of additional hyperparameters (See Algorithm \ref{alg:adamsnsm} for the full
implementation used in our experiments). \textbf{RMSPropSN} and \textbf{AdaGradSN} achieves
minimal memory footprint of just $m$ while requiring only 2 hyperparameters.

\subsection{LLM Pre-Training Experiments}
We test our method on the task of pre-training LLaMA models \cite{dubey2024llama,touvron2023llama}
on the C4 dataset \cite{raffel2023c4dataset} with a standard setup -- details in Appendix \ref{pretraining-setup}. Table \ref{tab:pretraining-table} presents the main pre-training
results and Table \ref{tab:Optimizer-states-memory} shows the memory footprint\footnote{The memory
footprint is the total parameters in
the optimizer states multiplied by 16 bits. See Listing \ref{lst:get_optimizer_state_size}
for more details.}
of different optimizers across a range of model sizes. 

\textbf{Additional Baselines.} We provide additional comparisons with FLORA \cite{hao2024flora}, LoRA \cite{hu2021lora}, and ReLoRA \cite{lialin2023relora} in Table \ref{tab:additional-comparisons-pretrain}. Note that these memory-efficient methods sacrifice performance (over Adam) to save memory while our method, AdamSNSM, achieves the best of both worlds. 

\begin{table*}[ht]
\centering
\small
\caption{Final perplexity (``Perpl.'') along with the  number of tokens
 in parentheses of different optimizers on pretraining
LLaMA models task. \textbf{Bolded methods} are ours. Columns LR and
\#TP denote the learning rate and the number of tunable parameters of the corresponding
method, respectively. We only tune for the base learning and set other parameters as in previous implementations. The memory column shows the optimizer's states
memory consumption given a parameter of shape $m\times n$ with
$m\ge n$. Red LR highlights instability.
\label{tab:pretraining-table}}
\begin{tabular}{lllrrrrrrrr}
\toprule 
\multirow{2}{*}{Methods} & Memory & \multirow{2}{*}{\#TP} & \multicolumn{2}{c}{\textbf{\uline{60M}} \uline{(1.38B)}} & \multicolumn{2}{c}{\textbf{\uline{130M}} \uline{(2.62B)}} & \multicolumn{2}{c}{\textbf{\uline{350M}} \uline{(7.86B)}} & \multicolumn{2}{c}{\textbf{\uline{1B}} \uline{(13.1B)}} \tabularnewline
 & (for $m\times n$) &  & Perpl. & \small LR & Perpl. & \small LR & Perpl. & \small LR & Perpl. & \small LR \tabularnewline
\midrule 
Adam & $\red{2mn}$ & $3$ & 30.46 & {\tiny 0.005} & 24.60 & {\tiny 0.005} & 18.67 & {\tiny \red{0.001}} & 16.00 & {\tiny \red{0.0005}} \tabularnewline
\textbf{AdamSN} & $mn+m$ & $3$ & {29.75} & {\tiny 0.05} & 22.90 & {\tiny 0.05} & 17.49 & {\tiny 0.05} & 14.96 & {\tiny 0.05} \tabularnewline
\textbf{AdamSNSM} & $\green{rn+m}$ & $5$ & \uline{29.74} & {\tiny 0.05} & \textbf{22.43} & {\tiny 0.05} & \textbf{16.91} & {\tiny 0.05} & \uline{14.05} & {\tiny 0.05} \tabularnewline
\midrule 
{AdaGradm} & $\red{2mn}$ & $ 2$ & {30.40} & {\tiny 0.10} & {24.86} & {\tiny 0.10} & {18.30} & {\tiny 0.10} & {17.42} & {\tiny 0.10} \tabularnewline
\textbf{AdaGradmSN} & $mn+m$ & $ 2$ & \textbf{29.73} & {\tiny 2.00} & \uline{22.58} & {\tiny 2.00} & {17.14} & {\tiny 2.00} & {14.48} & {\tiny 2.00} \tabularnewline
\textbf{AdaGradSNSM} & $rn+m$ & $4$ & {29.81} & {\tiny 1.00} & \textbf{22.43} & {\tiny 1.00} & \uline{16.99} & {\tiny 1.00} & \textbf{13.96} & {\tiny 1.00} \tabularnewline
\midrule 
{AdaGrad} & ${mn}$ & $\green 1$ & {37.12} & {\tiny 0.05} & {25.76} & {\tiny 0.05} & {18.14} & {\tiny 0.05} & {15.25} & {\tiny \red{0.01}} \tabularnewline
\textbf{AdaGradSN} & $\green m$ & $\green 1$ & {29.85} & {\tiny 2.00} & {24.19} & {\tiny 1.00} & {17.72} & {\tiny 1.00} & {14.82} & {\tiny 1.00} \tabularnewline
\midrule
RMSProp & $mn$ & $2$ & 35.51 & {\tiny 0.001} & 25.94 & {\tiny 0.001} & 20.01 & {\tiny 0.001} & {17.03} & {\tiny 0.001} \tabularnewline
\textbf{RMSPropSN} & $\green m$ & $2$ & 34.57 & {\tiny 0.01} & 25.67 & {\tiny 0.01} & 18.72 & {\tiny 0.01} & {15.97} & \red{\tiny  0.001} \tabularnewline \midrule
GaLore \cite{zhao2024galore} & $2rn$ & $\red 6$ & 34.73 & {\tiny 0.01} & 25.31 & {\tiny 0.01} & 18.95 & {\tiny 0.01} & 16.76 & {\tiny \red{0.001}} \tabularnewline
\midrule
\multicolumn{3}{c}{\textbf{Rank $r$ / Dimension $m$}} & \multicolumn{2}{c}{128/512} & \multicolumn{2}{c}{256/768} & \multicolumn{2}{c}{256/1024} & \multicolumn{2}{c}{512/2048} \tabularnewline
\bottomrule
\vspace{-0.5cm} 
\end{tabular}
\end{table*}

\subsubsection{Discussions}

\textbf{Subset-Norm (SN) improves upon all existing adaptive methods while reducing memory.} Modifying Adam, AdaGradm, AdaGrad, and RMSProp with the SN adaptive step size not only reduces memory footprint but improves their performance across different scales. Notably, AdaGrad and AdaGradm benefit the most from the SN step size, providing empirical support for the theoretical benefits of SN presented in Section \ref{sec:AdaGrad-Subset-Norm}. 

\textbf{Combining Subspace-Momentum (SM) with SN further improves performance while saving additional memory.} Perhaps surprisingly, limiting the use of momentum to a subspace \textit{improves} performance in SN-adaptive step sizes rather than degrading it. Our experiments show that SNSM, combining SN and SM, gives the best performance for the least amount of memory across model sizes. While adding SM introduces additional hyperparameters, Section \ref{subsec:SM-rank-gap-ablation} suggests that these parameters are not too sensitive.

Furthermore, Section \ref{subsec:Subspace-Momentum-ablations} shows that the choice of the subspace matters i.e. the subspace spanned by a top-$k$ singular vectors of a snapshot of a stochastic gradient seems to be the most beneficial for momentum as opposed to simpler choices like a random subspace. Our current guarantee for SM, presented in Section \ref{sec:Subspace-Momentum-Convergence},  does not yet explain why or when subspace momentum is useful, and theoretical understanding of (EMA style) momentum in stochastic optimization is still limited \cite{kidambi2018insufficiencyexistingmomentumschemes}. We believe this could be related to how momentum is beneficial when noise is low (and harmful when noise is high) and the choice of the subspace could correlate to the amount of gradient noise or optimization landscape that harm or benefit momentum \cite{wang2024the, gitman2019understanding}.

\textbf{Hyperparameter robustness.} In Table \ref{tab:pretraining-table}, the best learning rate (LR) found via grid search is displayed and is highlighted in red as the best LR changes across scales. This indicates potential sensitivity to tuning for each respective algorithm. We see that Adam requires smaller LR for larger models, but using SN and SNSM does not. AdaGradm seems less sensitive to the base LR overall.

\textbf{Closing the theory-practice gap.} While there is a non-trivial performance gap between Adam and AdaGrad(m) for larger models, using the SN step size closes this gap across scales. This shows that AdaGrad style algorithms can be competitive to Adam when using the SN step size. Interestingly, vanilla AdaGrad seems to perform well as model size increases. This is important because AdaGrad enjoys stronger theoretical understanding than Adam and has one fewer parameter -- $\beta_2$ -- to tune.

\begin{table}
\tiny 
\caption{Optimizer states memory footprint (in GB for BF16 dtype) for different
LLaMA models. Our methods, AdamSN, AdamSNSM, and RMSPropSN (RMSPSN), are modifications
of Adam and RMSProp (RMSP) to utilize Subset-Norm (SN) and Subspace-Momentum
(SM). For GaLore and AdamSNSM, the subspace is of dimension\protect\textsuperscript{\ref{fn:rank-galore}}
$d/r$, where the memory accounts for additional space for storing
the projection matrices.\label{tab:Optimizer-states-memory}}

\centering{}%
\begin{tabular}{lrrrrrr}
\toprule 
\textbf{Opt.} & \textbf{AdamW} & \textbf{AdamSN} & \textbf{RMSP} & \textbf{GaLore} & \textbf{AdamSNSM} & \textbf{RMSPSN} \\
\midrule 
\textbf{Mem.} & $2d$ & $d+\sqrt{d}$ & $d$ & $4d/r$ & $2\nicefrac{d}{r}+\sqrt{d}$ & $\sqrt{d}$\tabularnewline
\midrule
60M & 0.22 & 0.14 & 0.11 & 0.15 & 0.08 & 0.03\tabularnewline
130M & 0.50 & 0.30 & 0.25 & 0.29 & 0.16 & 0.05\tabularnewline
350M & 1.37 & 0.75 & 0.69 & 0.53 & 0.28 & 0.06\tabularnewline
1B & 4.99 & 2.62 & 2.49 & 1.61 & 0.84 & 0.12\tabularnewline
3B & 10.01 & 5.16 & 5.00 & 2.96 & 1.52 & 0.15\tabularnewline
7B & $\red{25.10}$ & 13.04 & 12.55 & 7.01 & 2.73 & 0.49\tabularnewline
\bottomrule
\end{tabular}
\vspace{-0.5cm} 

\end{table}

\subsection{LLMs Supervised Fine-Tuning (SFT) Experiments}
We further evaluate on a supervised-fine-tuning task, where we fine-tune a pre-trained LLaMA 7B model on the UltraFeedback dataset \cite{cui2023ultrafeedback} using the chosen responses with max sequence length of 1024. We train for 1 epoch with linear decay and gradient clipping of 1. Table \ref{tab:sft} contains the result with the time and memory of one training epoch on a single A100-80GB GPU. Note SNSM's $r$ denotes the dimension of SM but the optimization is full-rank. 
\begin{table}[ht]
\tiny
\centering
\caption{Last and minimum validation perplexity for SFT of LLaMA 7B on the UltraFeedback dataset between Adam, LoRA, and AdamSNSM for 2 different ranks. We also show the wall-clock time and peak memory for batchsize 1 for these optimizers. \label{tab:sft}}
\begin{tabular}{lrrrrr}
\toprule
\textbf{} & \textbf{Adam} & \textbf{LoRA (r=64)} & \textbf{AdamSNSM (r=64)} & \textbf{ SNSM (r=32)} \\
\midrule
Last &  2.622       & 2.632                   &   2.584 &       \textbf{2.580}                \\
Min.  &   2.401      & 2.410                   & 2.392 & \textbf{2.390}                        \\ \midrule
Time (min.)           &   266      &      249             &    303 & 301                     \\
Memory (GB)              &    77.11     &    20.75                & 42.89 & 42.89                         \\
\bottomrule
\end{tabular}
\label{tab:comparison}
\end{table}

\textbf{Discussion.} We observe similar improvement over Adam as in pre-training tasks. Surprisingly, the smaller rank (for momentum) is more beneficial than the larger rank. In contrast to LoRA, since we report peak-memory here, due to the full parameter training of SNSM, the primary memory bottlenecks are gradients and activations. Furthermore, we note that the primary contributor to SNSM's slower wall clock time is the SVD computation on large dimension. We try larger projection update gaps in Table \ref{tab:larger-updategap} which reduce this cost while maintaining good performance for our methods. Furthermore, we discuss potential more efficient alternatives in Section \ref{par:Projection-selection.} and leave further exploration to future works.

\textbf{GLUE Fine-tuning.}
Additional results on fine-tuning on GLUE tasks with BERT models are in Appendix \ref{subsec:glue-finetuning}.

\subsection{Ablation Studies}
In this section, we present ablation studies on various parameters
of SN and SM.

\textbf{Subset-Norm's subset size ablation. \label{subsec:Subset-size-ablation}}
While we use a simple scheme to compress the adaptive step size
of linear modules in the previous experiments, Table \ref{tab:dimension-convergence-rate} suggests that there is an optimal subset
size that depends on the noise. Figure \ref{fig:AdamSN-Subset-size-ablation} shows performance for various subset-size selection. Since the step size scales with the subset size,  the optimal base LR should be decreased as we decrease the subset size closer towards Adam. We include additional results for 130M model in Figure \ref{fig:AdamSN-Subset-size-ablation-130M}.

\begin{figure}
\begin{center}
\includegraphics[width=0.9\columnwidth]{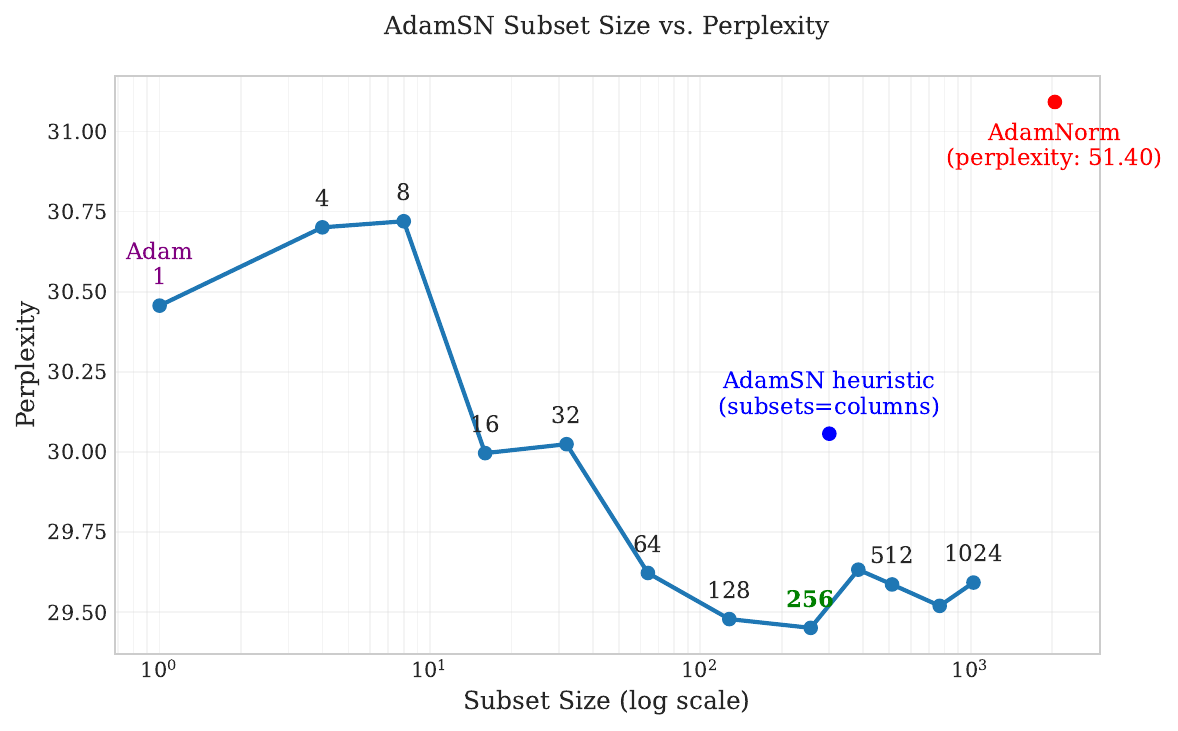}
\par\end{center}
\caption{Subset size ablation for AdamSN on LLaMA 60M trained for 1.38B tokens
(batch size of 512 of max length 256 for 10,000 steps). The higher
the subset size, the smaller the memory footprint of the second moment
optimizer state. \label{fig:AdamSN-Subset-size-ablation}}
\end{figure}


While one can use the heuristics discussed on models where linear modules make up the vast majority, for arbitrary models with
weights of $d$ elements, we found that a subset size of $\sqrt{d}/2$ is probably
a reasonable choice. If more resources are available, the subset
size can also be tuned. 

\textbf{Subspace selection.} \label{par:Projection-selection.} While the top-$k$ singular vectors of stochastic gradients gives a subspace with strong performance, performing SVD can be expensive for larger models and storing the dense projection consumes non-negligible memory for large ranks. Gradient-independent projections
like random gaussian as in FLORA \cite{hao2024flora} avoids
SVD and can save memory by storing the pseudorandom seed
(at the cost of recomputating the projection at every step). One can
further speed up the random projection by using a fast
subspace embedding like the Subsampled-Randomized Hadamard Transform
(SRHT) used in the Fast-JL transform \cite{ailon2009fast}. Random
projections like SRHT can also be used to approximate SVD (Appx-SVD)
computation \cite{halko2011finding} that can be much faster than
full SVD. Finally, the cheapest projection is a subspace
of random standard bases. Recently, GRASS \cite{muhamed2024grass}
explores this idea and tests sampling random bases with large gradient
norms. We examine different choices for the subspace and
compare their time, space, and performance in Appendix \ref{subsec:Subspace-Momentum-ablations} (Table \ref{tab:Different-Projection-Ablations}).  

\begin{table}[ht]
\centering
\scriptsize
\caption{Effects of less frequent subspace update schedule (gap). Compared to Table \ref{tab:pretraining-table} where the gap is fixed to 200 across all scales. \label{tab:larger-updategap}
}
\begin{tabular}{lrrrrrrrr}
\toprule
 Model Size & \textbf{60M} & \textbf{130M} & \textbf{350M} & \textbf{1B} \\
\midrule
 Gap/Steps (5\%) & 200/10K &  1K/20K & 3K/60K & 5K/100K \\
 \midrule
 AdamSNSM    & 29.84 & 22.71 & 18.43 & 15.28 \\
AdaGradSNSM & 30.28 & 22.76 & 17.02 & 13.90 \\
GaLore      & 36.69 & 29.37 & 21.27 & 19.14 \\
\midrule
\textit{Fixed} Subspace & 10K/10K &  20K/20K & 60K/60K & 100K/100K \\
 \midrule
 AdamSNSM    & 30.65 & 23.65 & 18.94 & 15.16 \\
AdaGradSNSM & 31.43 & 24.85 & {18.04} & 14.62 \\
GaLore      & 37.95 & 26.63 & 21.49 & 27.11 \\
\bottomrule
\end{tabular}
\end{table}

\textbf{Larger projection update gaps.} 
Frequently updating the projection map using SVD can be expensive, especially for larger models. Furthermore, updating the projection every 200 steps can be arbitrary. In Table \ref{tab:larger-updategap}, we examine more structured schedules: (1) updating every 5\% of the total training steps (corresponding to 200/10K steps for the 60M model) and (2) only using a fixed subspace at the start. Compared to Table, \ref{tab:pretraining-table} where a fixed gap of $200$ is used across scales, we see SNSM's performance stay relatively similar when we increase the update gap to $5\%$ of the total training steps, whereas GaLore's performance suffers more.

\begin{table}[ht]
\centering
\tiny
\caption{Fixed Subspace Choices on LLaMA 60M. We examine GaLore and SNSM with top-$k$ singular vectors projections (SVD) and random subspaces (Random) using dense gaussian projections.\label{tab:fixedsubspace-proj}}
\begin{tabular}{l|cc|cc}
\toprule
    & \textbf{GaloreSVD} & \textbf{GaloreRandom} & \textbf{SNSM+SVD} & \textbf{SNSM+Random} \\
\midrule
Perplexity & 37.95 & 38.23 & 30.65 & 40.15 \\
\bottomrule
\end{tabular}
\end{table}

Interestingly, for fixed subspace (100\% gap), GaLore still achieves decent performance even though the optimization only happens in a small subspace up until the 1B model, where the training stops improving after 50K/100K steps. In Table \ref{tab:fixedsubspace-proj}, we see that a random subspace seems to work decently well too. This suggests that a majority of progress can be made in a small subspace in smaller models. In contrast, this is not the same for restricting momentum to a subspace. Furthermore, we notice that there are training loss spikes at the times when we switch subspace for GaLore that impacts training with 5\% gap, most likely due to incompatible optimizers' statistics between subspaces. This could explain why GaLore's 100\% gap performs similarly or even better than 5\% gap for certain run. Finally, we note that AdaGradSNSM performs the best here with the larger gaps as the dimension increases.

\textbf{Additional experiments and ablations.}
We provide additional experiments and detailed ablations on wall-clock time speedup, peak-memory savings, the effect of clipping, batch sizes, random seeds, combinations between adaptive step sizes and momentum, and more in Appendix \ref{sec:additional-ablations}. 

\section{Related Works}

As model sizes grow, memory-efficient training techniques have become
crucial. Following up on AdaFactor \cite{shazeer2018adafactor}, low-rank
methods like Galore \cite{zhao2024galore}, LoRA \cite{hao2024flora},
and ReLORA \cite{lialin2023relora} reduce memory usage by approximating
large weight matrices with low-rank representations. Projection-based
approaches, such as GRASS \cite{muhamed2024grass} and FLORA \cite{hao2024flora},
compress gradients or combine low-rank ideas with projections to reduce
memory requirements. Recently, AdaMeM \cite{vyas2024adamem} proposes to incorporate the orthogonal subspace to the AdaFactor optimizer; this is related to but different from our simpler SM algorithms, where we use subspace decompositions to decouple the momentum and SGD. 
 BAdam \cite{luo2024badammemoryefficientparameter},
a block coordinate descent method that utilizes Adam as an inner solver,
has been proposed for fine-tuning large language models. 
In contrast to our proposed methods, these methods are largely heuristic-driven and often lack convergence
guarantees under standard assumptions. 
On the other hand, methods like SM3 \cite{anil2019memory},
which uses subset (cover) statistics to show convergence in online
learning, and MicroAdam \cite{modoranu2024microadam}, which provides
convergence guarantees for a gradient compression scheme with error
correction, offer theoretical guarantees. 

Additional approaches to reducing memory during training include optimizer
quantization \cite{li2024memory,dettmers20218,dettmers2024qlora},
attention computation compression/optimization \cite{wu2022memorizing,dao2022flashattention,dao2023flashattention,shah2024flashattention},
activation checkpointing \cite{chen2016training}, and distributed
training \cite{rajbhandari2020zero}. For inference, compression techniques
are also actively being explored \cite{sakr2024espace,dettmers20228bitquantization,xiao2024smoothquantaccurateefficientposttraining,lin2024awqactivationawareweightquantization,frantar2023gptqaccurateposttrainingquantization}. These are orthogonal directions to our work and can be combined. 
Another orthogonal direction is approximated second-order optimization, where one aims to approximate the Hessian preconditioner using only first-order information in order to achieve faster convergence. Some works in this area include \cite{gupta2018shampoo,liu2023sophia, vyas2024soap}. These methods typically demonstrate faster training but at the cost of super-linear memory and additional computational overhead.  

 Convergence analysis of non-convex optimization methods has seen significant
progress, with recent works providing convergence proofs for adaptive
algorithms like Adam \cite{li2024convergence,defossez2020simple}.
Numerous studies have explored convergence properties of various adaptive
and stochastic gradient methods \cite{chen2018convergence,defossez2020simple,ene2021variational,liu2023high,liu2023on,ward2019adagrad,zou2019sufficient,reddi2018convergence,nesterov1983method},
while lower bound analyses \cite{arjevani2023lower} have highlighted
fundamental limits in non-convex optimization. Here, obtaining convergence results for EMA updates (Adam style) for subset-norm and under further relaxed assumptions like affine smoothness \cite{wang2023convergence, attia2023sgd}, affine noise \cite{hong2024adagradmomentum,faw2022power}, heavy-tailed noise \cite{zhang2019gradient,zhang2020adaptive,nguyen2023improved,nguyen2023high} are of great interest. 


\textbf{Comparison with Adam-mini.} Very recently, Adam-mini \cite{zhang2024adamminiusefewerlearning} also uses shared step sizes as Subset-Norm; however, the partition strategy is quite different from ours. While Adam-mini also employs a grouping strategy for the adaptive step size, it is primarily motivated empirically and lacks a general grouping strategy for general parameters. 
In contrast, our theory results show that grouping by noise magnitude leads to improvement. 
In experiments, our AdamSNSM uses less memory than Adam-mini, due to the fact that Adam-mini uses full momentum while we use momentum only in a subspace (which outperforms full momentum in many cases given a good choice of subspace). Furthermore, in terms of perplexity, Adam-mini performs very closely to AdamW while our methods outperform Adam (which performs similarly to AdamW) on a range of language tasks and model sizes.
\section{Conclusion and Future Works }
In this paper, we introduce two principled optimizer states' memory reduction methods ---Subset-Norm
(SN) and Subspace-Momentum (SM)---designed to address the high memory
costs associated with adaptive optimizers in large-scale deep learning. SN and SM achieve memory savings without
compromising performance and admit high-probability convergence guarantees under relaxed
assumptions. Extensive experiments pre-training and fine-tuning LLMs validate our methods' effectiveness and efficiency.

\textbf{Future works.}
Promising directions include exploring SN and SM on additional domains like Reinforcement Learning, where high memory and high noise are also bottlenecks when scaling up to large models. Fully generalizing our methods involves developing general projections (beyond matrix decomposition) for SM on higher order tensors for use on additional architectures like CNNs.  A more in-depth investigation to more optimal subset partition strategies for SN is also an interesting open question, since our analysis in Section \ref{subsec:Coordinate-noise-sparsity-and} only applies to equipartition subsets.  Further theoretical understanding for SM's subspace dependency for improving the subspace selection for SM is desirable. Furthermore, the convergence of SNSM is still unknown. Finally, the benefits of momentum in  stochastic optimization in general is still a mystery; using SM  to study effects of momentum on particular subspaces could open doors to obtain a more fine-grained understanding for why or when momentum helps. 




\section*{Acknowledgement}

This work was supported by NSF CCF 2311649.
We thank Alina Ene, Themistoklis Haris, and Duy Nguyen  for insightful discussions and Hieu Nguyen for assistance with the experiments.
This work was completed in part using the Discovery cluster, supported by Northeastern University’s Research Computing team.

\section*{Impact Statement}

This paper presents work whose goal is to advance the field of 
Machine Learning. There are many potential societal consequences 
of our work, none which we feel must be specifically highlighted here.

\bibliography{ref}
\bibliographystyle{icml2025}

\newpage
\appendix
\onecolumn

\section{Experimental Details}

In this section, we provide hyperparameters details, implementation
details (pseudocode), and other practical considerations. 

\subsection{Experimental details}\label{pretraining-setup}
 All of our pre-training experiments
are conducted on NVIDIA RTX4090/3090 GPUs. Unless specified otherwise, we run all experiments on BF16
format, weight decay of 0, gradient clipping of 1.0, cosine learning
rate decay to 10\% of the max learning rate with 10\% linear warmup
steps, and batch size of 512 (similarly to \cite{zhao2024galore}
and \cite{touvron2023llama,dubey2024llama}).\footnote{Note that these addition improve the performance for all baselines.
See Appendix \ref{subsec:Gradient-clipping.}.} For all our experiments, we use the default $\left(\beta_{1},\beta_{2}\right)=\left(0.9,0.999\right)$
and only tune for the base learning rate within a grid of $\left\{ 0.5,0.1,0.05,0.01,0.005,0.001\right\} $.\footnote{Except AdaGradSNm where we find higher learning rates in $\left\{ 0.5,1,2,5\right\} $
to be better. We tune the lr on the 60M model and use the same learning
rate for the larger model, where the base learning rate is only reduced
if the method fails to converge.} We train for 1.38B, 2.62B, 7.86B, and 13.1B tokens for models of
sizes 60M, 130M, 350M, and 1B parameters, respectively, following \cite{zhao2024galore}
and matches roughly the scaling laws in \cite{hoffmann2022training}. 

For GaLore, we use the same hyperparameters as in \cite{zhao2024galore},
where we use rank 128/512, 256/768, 256/1024, and 512/2048 for the
60M, 130M, 350M, and 1B models, respectively (Table 2 of \cite{zhao2024galore}).\footnote{Note that our reproduced results for GaLore and baselines are similar to \cite{zhao2024galore}.}
For AdamSNSM, we use the same ranks and projection update gap (of
200) as GaLore for all models.\footnote{Note that a smaller gap is more expensive than a larger gap. Our experiments
below show that we can increase the projection update gap without
much performance loss. If data is not limited, one could use a larger
gap to speed up training. However, if data is limited, then a smaller
gap to converge in fewer tokens is potentially more desirable.} However, we do not tune for an additional scaling parameter unlike
GaLore, and we compresses the LM head (final linear layer) with SN
and SM also.\footnote{Existing methods typically do \emph{not} compress the embedding layer
and final LM head, while our methods seem robust to this choice. Compressing these layers save additional memory.}



\subsection{Adam-Subset-Norm Implementation\label{sec:Adam-Subset-Norm-Implementation}}

Algorithm \ref{alg:Adam-Subset-Norm} presents the pseudocode for
Adam-Subset-Norm as mentioned in Section \ref{subsec:SN-Implementation} where we partition the coordinates (for each parameter) into subsets of equal sizes. 

\begin{algorithm}[tb]
   \caption{Adam-Subset-Norm with a Simple Partitioning Scheme}
   \label{alg:Adam-Subset-Norm}
\begin{algorithmic}
   \STATE {\bfseries Input:} Learning rate $\eta$, EMA parameters $\beta_1$ and $\beta_2$, $\epsilon > 0$, optional weight decay $wd \geq 0$
   \FOR{each $p \in \mathbb{R}^{m \times n}$ in params}
      \STATE $\text{grad} \gets p.\text{grad}$
      \STATE $r \gets 0$ if $m \geq n$ else $1$
      \STATE $k \gets p.\text{shape}[r]$ \hfill \COMMENT{Set $k = m$ if $r = 0$, else $k = n$}
      \STATE $\text{gradN} \gets \text{grad.norm(dim=} 1-r\text{)} \in \mathbb{R}^k$ \hfill \COMMENT{Subset norm}
      \STATE $m \gets \beta_1 m + (1 - \beta_1) \cdot \text{grad} \in \mathbb{R}^{m \times n}$
      \STATE $v \gets \beta_2 v + (1 - \beta_2) \cdot \text{gradN}^2 \in \mathbb{R}^k$ \hfill \COMMENT{Omitting bias correction terms}
      \STATE $p \gets p + \eta \frac{m}{\sqrt{v} + \epsilon}$ \hfill \COMMENT{Broadcast division}
      \STATE $p \gets p - \eta \cdot wd$ \hfill \COMMENT{Weight decay}
   \ENDFOR
\end{algorithmic}
\end{algorithm}

\subsection{Generic Subset-Norm Adaptive Step Size Implementation\label{subsec:Generic-Subset-Norm-Adaptive}}

The heuristic implementation in Section \ref{sec:Adam-Subset-Norm-Implementation} is simple and does not require any tuning.
However, to modify existing algorithms to work with arbitrary subsets, one
could utilize reshape as in Algorithm \ref{alg:generalized-subset-norm} as an example.

\begin{algorithm}[tb]
   \caption{Generic Subset-Norm Adaptive Step Size Update Rule (PyTorch-y Notation)}
   \label{alg:generalized-subset-norm}
\begin{algorithmic}
   \STATE {\bfseries Input:} Parameter $P \in \mathbb{R}^d$, step size $\eta > 0$, $\beta$, $\epsilon > 0$, and partition size $k$ such that $k$ divides $d$
   \STATE $R \gets (\nabla P).\text{reshape}(d/k, k)$ \hfill \COMMENT{Reshape gradient into shape $\frac{d}{k} \times k$}
   \STATE $V \gets \beta V + (1 - \beta) \cdot ((R\text{**2}).\text{sum(dim=1)}) \in \mathbb{R}^{d/k}$ \hfill \COMMENT{Update state $V$ via subset norm reduction on dim 1}
   \STATE $U \gets \frac{R}{\sqrt{V} + \epsilon} \in \mathbb{R}^{d/k \times k}$ \hfill \COMMENT{Broadcast addition and division for update step}
   \STATE $P \gets P - \eta \cdot U.\text{view}(d)$ \hfill \COMMENT{Reshape $U$ back to $\mathbb{R}^d$ and update $P$}
\end{algorithmic}
\end{algorithm}

\subsection{AdamSNSM Implementation Details }

Algorithm \ref{alg:adamsnsm} provides the pseudocode and implementation
details for the version of AdamSNSM with SVD subspace momentum and
heuristics subset-norm (as described in Section \ref{subsec:SN-Implementation})
used in our experiments.

\begin{algorithm}[tb]
   \caption{AdamSNSM with Subspace Momentum via Top-$k$ Singular Vectors from SVD}
   \label{alg:adamsnsm}
\begin{algorithmic}
   \STATE {\bfseries Input:} Learning rate $\eta$, rank $k$, update gap $G$, momentum parameters $\beta_1, \beta_2 \in (0,1)$, and stability parameter $\epsilon$
   \FOR{$t = 1, \dots, T$}
      \STATE Obtain stochastic gradient $g_t \in \mathbb{R}^{m \times n}$ \hfill \COMMENT{WLOG, assume $m \geq n$}
      \IF{$t \mod G = 0$}
         \STATE $U, S, V = \text{SVD}(g_t)$ \hfill \COMMENT{Compute singular value decomposition}
         \STATE $P = U[:, :k] \in \mathbb{R}^{m \times k}$ \hfill \COMMENT{Extract top $k$ singular vectors}
      \ENDIF
      \STATE $m = \beta_1 m + (1 - \beta_1) P^T g_t \in \mathbb{R}^{k \times n}$ \hfill \COMMENT{Update subspace momentum}
      \STATE $r = g_t - P P^T g_t$ \hfill \COMMENT{Compute orthogonal SGD component}
      \STATE $s = \text{sum}(g_t, \text{dim}=1) \in \mathbb{R}^n$ \hfill \COMMENT{Sum all columns for subset-norm heuristic}
      \STATE $v = \beta_2 v + (1 - \beta_2) s^2 \in \mathbb{R}^n$ \hfill \COMMENT{EMA of subset-norm}
      \STATE $x_t = x_{t-1} + \eta \frac{P m + r}{\sqrt{v} + \epsilon}$ \hfill \COMMENT{Update with subspace momentum and subset-norm step size}
   \ENDFOR
\end{algorithmic}
\end{algorithm}

\subsection{Measuring Memory Footprint of Optimizers}

In PyTorch, we can obtain the number of parameters in optimizer states
using the code in Listing \ref{lst:get_optimizer_state_size}.

\begin{lstlisting}[float, caption={PyTorch function to calculate optimizer state size}, label={lst:get_optimizer_state_size}]
def get_optimizer_state_size(optimizer) -> Tuple[int, Dict[str, int]]:
    total_state_size = 0
    state_size_breakdown = {}
    for group in optimizer.param_groups:
        for p in group['params']:
            state = optimizer.state[p]
            for state_key, state_value in state.items():
                if torch.is_tensor(state_value):
                    if state_value.numel() == 1:
                        # we do not count singleton
                        continue
                    total_state_size += state_value.numel()
                    if state_key not in state_size_breakdown:
                        state_size_breakdown[state_key] = 0
                    state_size_breakdown[state_key] += state_value.numel()
    return total_state_size, state_size_breakdown 
\end{lstlisting}

\subsection{Peak memory measurement during training for different optimizers}

We measure peak memory consumption directly via running nvidia-smi
in Figure \ref{fig:Peak-GPU-Memory} while training as oppose to controlled
measurement as in Table \ref{tab:Optimizer-states-memory}. Note that
these peak measurements incur additional memory from gradient computation
and algorithms' overhead. 

\begin{figure}

\begin{centering}
\includegraphics[width=1\textwidth]{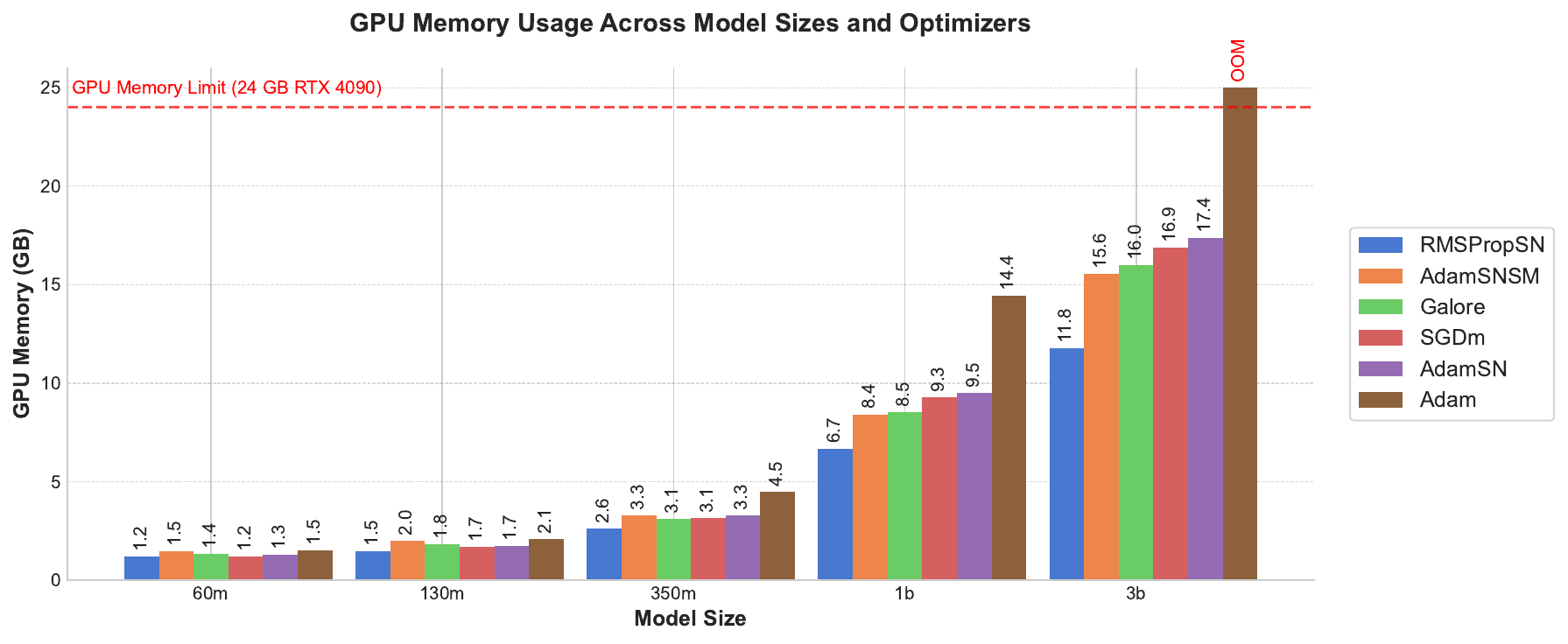}
\par\end{centering}
\caption{Peak GPU Memory Usage (Gb) for various model sizes, obtained with
batch size 1 and activation checkpointing to measure the optimizer
state footprint. \label{fig:Peak-GPU-Memory}}

\end{figure}

\section{Additional Experiments and Ablation Studies}\label{sec:additional-ablations}

\subsection{Additional Comparisons with Memory Efficient Optimizers for Pre-Training LLaMA Models}\label{subsec:additional-pretraining-comparison}
Table \ref{tab:additional-comparisons-pretrain} extends Table \ref{tab:pretraining-table} to compare against other recent methods. 

\begin{table}[ht]
    \centering
    \caption{Additional comparisons with other memory efficient optimizers}
    \label{tab:additional-comparisons-pretrain}
    \begin{tabular}{lrrrr}
        \toprule
        \textbf{Method} &
        \textbf{LLaMA 60M} &
        \textbf{LLaMA 130M} &
        \textbf{LLaMA 350M} &
        \textbf{LLaMA 1B}  \\
        \midrule

        AdamW            & 30.46 & 24.60 & 18.67 & 16.00 \\ \midrule
        AdamSNSM (\textbf{ours}) & 29.74 & 22.43 & 16.91 & 13.96 \\ \midrule
        GaLore    \cite{zhao2024galore}       & 34.73 & 25.31 & 20.51 & 16.76 \\
        FLORA   \cite{hao2024flora}         & 32.52 & --    & 23.69 & --    \\
        LoRA   \cite{hu2021lora}          & 34.99 & 33.92 & 25.58 & 19.21 \\
        ReLoRA  \cite{lialin2023relora}         & 37.04 & 29.37 & 29.08 & 18.33 \\
        \bottomrule
    \end{tabular}
\end{table}

\subsection{Wall-clock speedup and peak memory}
We provide the per iteration time, peak memory (via nvidia-smi), and time to Adam’s val perplexity after 100K steps for the 1B model for each method on a 2x4090 machine with the same setup as in Table \ref{tab:pretraining-table} (seq length 256, total batch size 512, micro batchsize 16) in Table \ref{tab:wallclock-mem}. 

\begin{table}[ht]
    \centering
    \scriptsize
    \caption{Per iteration time, peak memory (via nvidia-smi), and time to Adam’s val perplexity after 100K steps for the 1B model for each method on a 2x4090 machine with the same setup as in Table \ref{tab:pretraining-table} (seq length 256, total batch size 512, micro batchsize 16) in Table \ref{tab:wallclock-mem}}
    \label{tab:wallclock-mem}
    \begin{tabular}{llllll}
        \toprule
        & \textbf{Adam} & \textbf{AdamSNSM (Gap=5000)} & \textbf{AdamSNSM (Gap=200)} & \textbf{AdamSN} & \textbf{GaLore (Gap=200)} \\
        \midrule
        Time for 1K iters         & 7426 s   & 7465 s   & 7624 s   & 7399 s   & 7827 s   \\
        Time per iteration        & 7.43 s/it & 7.47 s/it & 7.62 s/it & 7.39 s/it & 7.83 s/it \\
        Time to perplexity $< 16$   & $\sim$206.4 hrs (100K iters) & $\sim$136.9 hrs ($<$66K iters) & $\sim$101.6 hrs ($<$48K iters) & $\sim$118.9 hrs ($<$58K iters) & $>$217 hrs ($>$100K iters) \\
        Peak mem                  & 21.554 GB/GPU & 16.642 GB/GPU & 16.642 GB/GPU & 19.193 GB/GPU & 18.187 GB/GPU \\
        \bottomrule
    \end{tabular}
\end{table}

\subsection{Vision Tasks}\label{subsec:vision-tasks}

    \textbf{Diffusion Transformers.} While our main focus is on large models that are more typical to language models where memory is often a bottleneck, vision models are also increasing in size. Hence, we conduct further evaluations using the DiT-L/2 model (458M)\footnote{\url{https://github.com/facebookresearch/DiT}} on a setup with batch size 2048, image size 64, and 8×A6000 GPUs. We compared our method (SNSM) with Adam. As shown in Table \ref{tab:ditl2-fid}, SNSM outperforms Adam in FID similarly to LLM tasks.

\begin{table}[ht]
    \centering
    \footnotesize
        \caption{FID scores over training iterations for the DiT-L/2 model (458M parameters) on $64\times64$ images.}
    \begin{tabular}{lccccc}
        \toprule
        \textbf{FID / Iter} & \textbf{200k} & \textbf{300k} & \textbf{400k} & \textbf{500k} & \textbf{700k} \\
        \midrule
        Adam      & 56.69 & 56.63 & 40.69 & 41.15 & 39.61 \\
        AdamSNSM  & 66.76 & 66.31 & 34.05 & 32.31 & 32.26 \\
        \bottomrule
    \end{tabular}

    \label{tab:ditl2-fid}
\end{table}

\textbf{CIFAR10 and CIFAR100.} We further evaluate Adam, AdamSN, and AdamSNSM (rank 64 and no update gap) by training \texttt{vit\_base\_patch16\_224}\footnote{\url{https://huggingface.co/timm/vit_base_patch16_224.augreg2_in21k_ft_in1k}} \cite{dosovitskiy2020imagevit} (around 85M params) from the \texttt{timm} library\footnote{\url{https://timm.fast.ai/}} on the CIFAR10 and CIFAR100 \cite{krizhevsky2009learningcifar} datasets for 10 epochs with a batch size of 64 and weight decay 0 on a 2x4090 machine. We tune the lr across \{1e-3, 5e-3, 1e-4, 5e-4, 5e-5, 1e-5\} grid for all methods. The results are shown in Table \ref{tab:cifar10-100}.

\begin{table}[h]
\centering
\caption{Performance Comparison of Optimizers on Vision Transformers for CIFAR10 and CIFAR100}
\label{tab:cifar10-100}
\begin{tabular}{lccc}
\toprule
\textbf{Best val accuracy (10 epochs)} & \textbf{Adam} & \textbf{AdamSN} & \textbf{AdamSNSM (r=64, g=1000)} \\
\midrule
CIFAR100 & 43.30\% & 45.20\% & 45.60\% \\
CIFAR10  & 69.02\% & 69.18\% & 71.21\% \\
\midrule
\textbf{Peak Mem (bs 64)} & 9.288GB & 8.886GB & 8.878GB \\
\bottomrule
\end{tabular}
\end{table}

These preliminary experiments show promising results for the application of our methods to vision tasks where models are becoming larger. 

\subsection{Fine-tuning on GLUE Tasks}\label{subsec:glue-finetuning}

Table \ref{tab:Fine-tuning-GLUE} presents results for fine-tuning
on GLUE dataset for various methods. The SN step size maintains good
performance while reducing the memory footprint.

\begin{table}
\centering{}\caption{Performance metrics across GLUE tasks. QQP, RTE, SST-2, MRPC, STSB,
QNLI, and MNLI use accuracy as the metric, while CoLA uses the Matthews
correlation coefficient. The \textbf{best} and \uline{runner-up}
results for each task and the average score are highlighted.\label{tab:Fine-tuning-GLUE}}
\begin{tabular}{lrrrrrrrrr}
\toprule 
\textbf{Method} & \textbf{QQP} & \textbf{RTE} & \textbf{SST2} & \textbf{MRPC} & \textbf{STSB} & \textbf{QNLI} & \textbf{MNLI} & \textbf{COLA} & \textbf{Avg}\\
\midrule 
Adam & \textbf{92.0} & 77.9 & 94.9 & 89.2 & 90.5 & \uline{93.0} & \textbf{87.6} & \textbf{65.4} & 86.3\\
GaLore ($r=4$) & 90.9 & 79.4 & \textbf{95.2} & 88.7 & \textbf{90.8} & 92.4 & 86.9 & 61.9 & 85.8\\
RMSProp & \uline{91.9} & 79.4 & \textbf{95.2} & \textbf{91.4} & 90.3 & 92.8 & \textbf{87.6} & \uline{65.1} & \textbf{86.7}\\
\midrule 
RMSPropSN & \uline{91.9} & \textbf{80.1} & \uline{95.1} & 90.0 & \uline{90.7} & \textbf{93.1} & \uline{87.5} & 63.8 & \uline{86.5}\\
AdamSN & 91.2 & 74.4 & 94.5 & 89.5 & 90.4 & 92.0 & 86.7 & 64.4 & 85.4\\
\bottomrule
\end{tabular}
\end{table}

\subsection{AdaGrad, AdaGrad-Norm, and AdaGrad-Subset-Norm}

We examine the subset-norm step size for AdaGrad in Figure \ref{fig:adagrad-experiment}.
We again see that subset-norm is slightly better than the full coordinate
version while using a lot less memory. This is consistent with our
observations for Adam and RMSProp when we replace the standard coordinate-wise
step size with the subset-norm adaptive step size.

\begin{figure}[tb]
\includegraphics[width=0.49\textwidth]{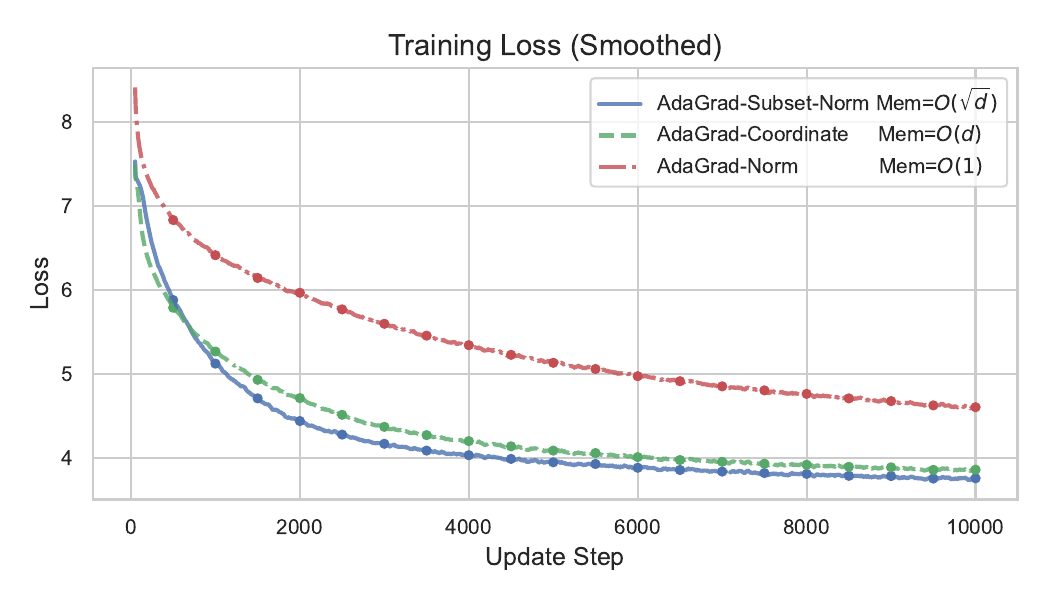}
\includegraphics[width=0.49\textwidth]{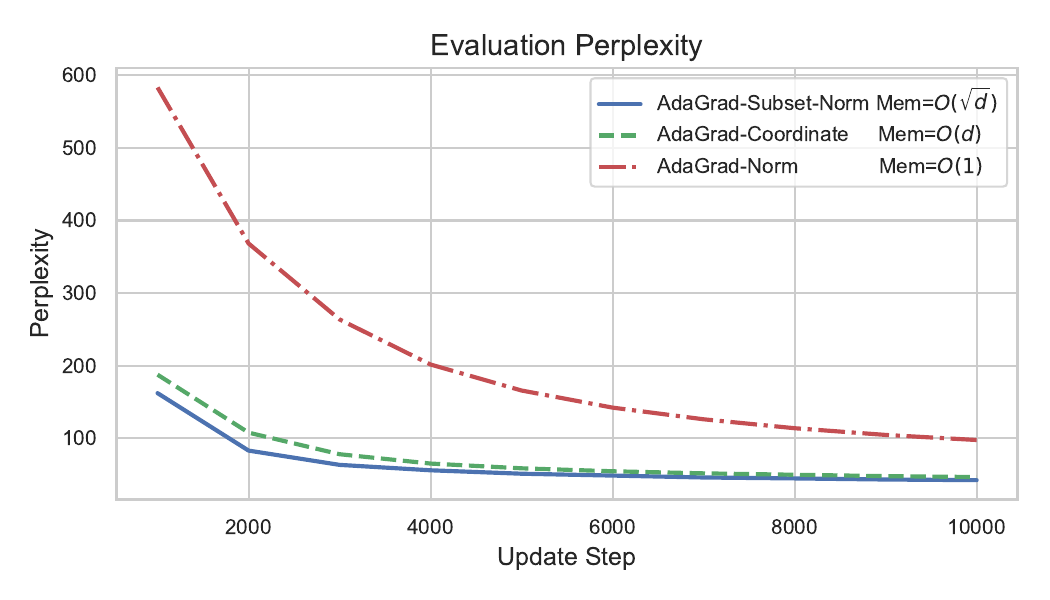}\caption{\label{fig:adagrad-experiment} Pretraining LLaMA 60M on the C4 dataset
for AdaGrad variants. Memory consumption estimate as a function of
parameter count $d$ is shown in the legend.}
\end{figure}



\subsection{Additional Subset-Size Experiments for 130M model}
We provide additional subset-size experiments similar to the ones in Section \ref{subsec:Subset-size-ablation} for LLaMA 130M in Figure \ref{fig:AdamSN-Subset-size-ablation-130M}. 

\subsection{Subspace-Momentum Rank and Gap Ablations}\label{subsec:SM-rank-gap-ablation}

\begin{figure}[t]
\begin{minipage}[t]{0.48\columnwidth}%
\begin{center}
\includegraphics[width=1\textwidth]{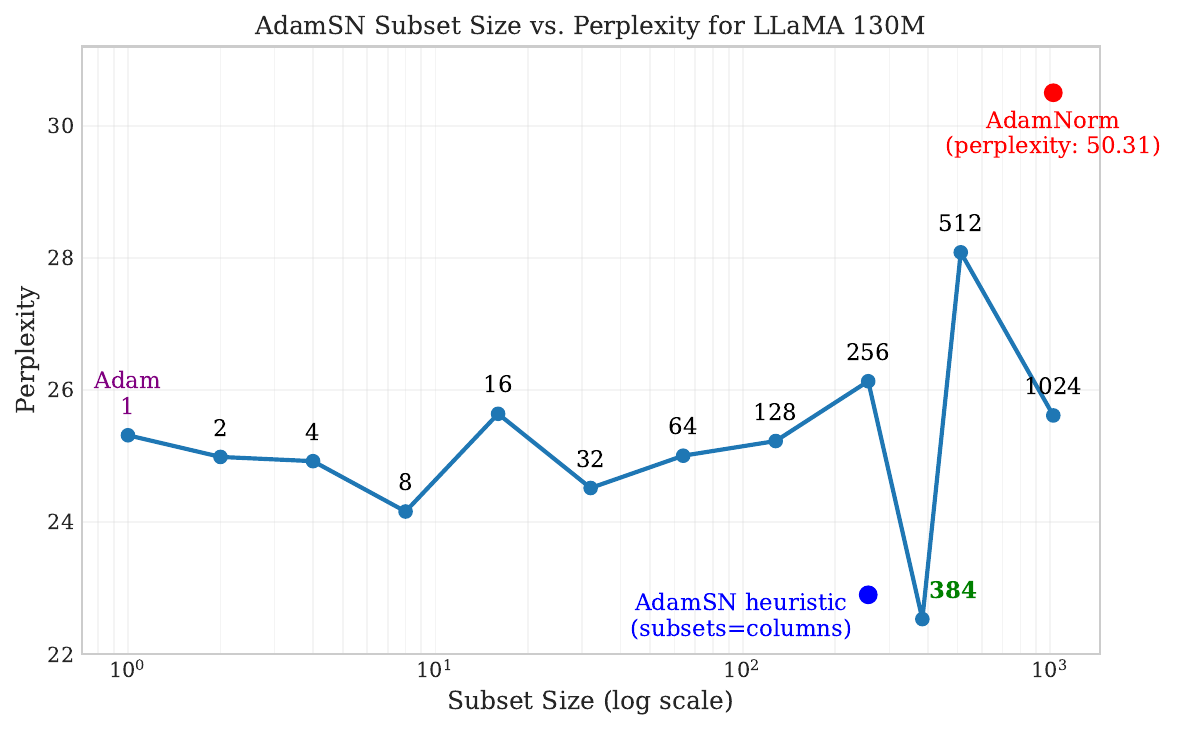}
\par\end{center}
\caption{Subset size ablation for AdamSN on LLaMA 130M trained for 2.62B tokens
(batch size of 512 of max length 256 for 20,000 steps). The higher
the subset size, the smaller the memory footprint of the second moment
optimizer state. \label{fig:AdamSN-Subset-size-ablation-130M}}
\end{minipage}\hfill{}%
\begin{minipage}[t]{0.5\columnwidth}%
\begin{center}
\includegraphics[width=1\textwidth]{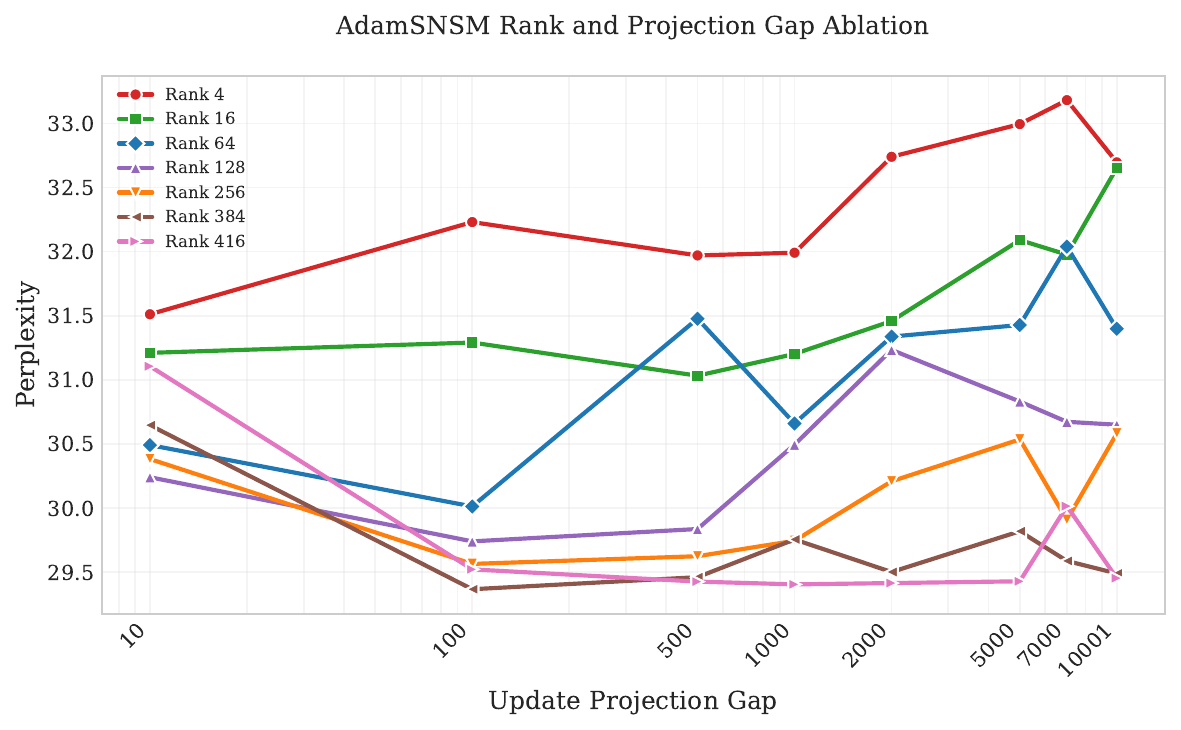}
\par\end{center}
\caption{Rank and gap ablation for AdamSNSM on LLaMA 60M for 10,000 steps.
The lower the rank, the less memory consumption used by the momentum
state. The higher the projection gap, the less SVD computation is
performed which is cheaper. \label{fig:Rank-and-gap-ablation}}
\end{minipage}
\end{figure}

\paragraph{Rank and gap ablations.}

We examine the impact of varying rank and update gap of subspace momentum,
similarly to \cite{zhao2024galore}, in Figure \ref{fig:Rank-and-gap-ablation}.
There, we see that the higher the rank, the better the results. For
the update gap, it seems like there is an optimal choice. However, due to the SVD computation, a larger gap will
be cheaper than a more frequent gap. 

\subsection{Subspace-Momentum Projection Choice Ablations\label{subsec:Subspace-Momentum-ablations}}

\paragraph{Projection types.}

Table \ref{tab:Different-Projection-Ablations} tests different choices for projection in SM discussed in Section
\ref{par:Projection-selection.}.
Note that for memory storage, SVD, Random Projection via dense Gaussian projection (Gaussian), and Approximated-SVD (Appx-SVD)
need to store the $r\times n$ projection matrix (unless we recompute at every step).
The remaining methods only need to store the indices for sampling and/or the
random seed to regenerate any random choices. 

Note that the choice of the projection is important as some projections are more computationally and memory expensive than other, although trading other qualities for given the cost. Simple projections like selecting a subset of coordinates for momentum (Subset-Momentum) are not only faster but enables simple distributed training like FSDP unlike more complex subspace selection mechanism that requires additional priors about the parameters (shape, low-rank, etc.) that might not always satisfied. 

\paragraph{Online $k$-PCA and Streaming $k$-PCA for Up-to-date Subspace.} Computing subspace from stochastic gradient snapshots can be noisy. Recently, \cite{liang2024opca} proposes a formulation of online-PCA to handle the problem of staled top-$k$ components as the stochastic gradients evolve. We test this algorithm in the OPCA column. Another natural algorithm to ensure the top-$k$ components stay up-to-date is Oja's algorithm for streaming $k$-PCA  \cite{huang2021streamingojakpca}. We also test this algorithm in Table \ref{tab:Different-Projection-Ablations}. While we can maintain up to date projection using these schemes, more frequent updates suffer from the same issue of transferring optimization statistics from one subspace to another. We only test for not resetting the statistics in this setting and leave additional investigation for future works. Furthermore, these schemes are more expensive computationally due to additional computation requirement at every step. OPCA further uses Adam for inner optimization which incurs additional memory.

\begin{table*}
\caption{Different projections selection for Subspace-Momentum and validation
perplexity. All methods are evaluated on LLaMA 60M with rank $128/512$
and a projection update gap of $200$. Time and space rows denote time and space to compute and store the projection.\label{tab:Different-Projection-Ablations}}
\scriptsize
\centering
\begin{tabular}{lrrrrrr|rr} \toprule
AdamSNSM's projection & SVD & Approx-SVD & Gaussian & SRHT & Top-$k$ & Random-rows   & OPCA & Oja\\ \midrule
Time (for $m\times n$) & $O(mn^{2})$ & $O(mn\log k+kn^{2})$ & $O(kn)$ & $O(\max(m,n))$ & $O(mn)$ & $O(k)$ & $O(kn)$ & $O(kn)$ \\
Space (for rank $k$) & $O(kn)$ & $O(kn)$ & $O(kn)$ & $O(k)$ & $O(k)$ & $O(k)$ & $O(kn)$& $O(kn)$\\
\midrule
Validation Perplexity & 29.74 & 31.51 & 42.48 & 33.33 & 31.42 & 33.17 & 29.63 & 30.69 \\ \bottomrule
\end{tabular}
\end{table*}

\subsection{Step Sizes and Momentum Choices Full Ablations\label{subsec:Step-sizes-and-momentum}}

We investigate various combinations of momentum and adaptive step
size approaches in Table \ref{tab:stepsize-momentum-ablations}. For
adaptive methods, we compare EMA, which uses exponential moving average
to accumulate the second moment ($v_{t}^{2}=\beta v_{t-1}^{2}+(1-\beta)g_{t}^{2}$),
with AdaGrad's cumulative accumulation approach ($b_{t}^{2}=b_{t-1}^{2}+g_{t}^{2}$).
Methods with the SN suffix utilize subset norm for parameter grouping,
contrasting with per-coordinate approaches that are standard.
While EMA momentum follows the standard momentum implementation, subspace
momentum employs a reduced rank approximation with rank 128 for this
model size.

\begin{table}[ht]
\centering{}\caption{Different combinations of momentum (columns) and adaptive step-size (rows)
and the effect of the learning rate schedule on each combination (cosine learning
rate decay schedule with warmup ``coslr'' or  constant learning
rate ``lr.'').
Memory footprint for each adaptive step size and/or
momentum are shown. Green and red highlight runs with perplexity below
30 and above 50 respectively. 
\label{tab:stepsize-momentum-ablations}}
\begin{tabular}{llll}
\toprule 
\multicolumn{1}{l}{Final eval perplexity (lr)} & \textbf{No momentum} & \textbf{EMA momentum} & \textbf{Subspace momentum}\\
\multicolumn{1}{c}{LLaMA 60M for 1.31B tokens} & Mem = 0 & Mem = $m\cdot n$ & Mem = max$(m,n)\cdot\text{rank}$\\
\midrule 
 & \uline{SGD} & \uline{SGDm} & \uline{SGD+SM}\\
\textbf{No Adaptive Step-size} & $\red{86.60}$ (coslr=1e-3) & $\red{55.76}$ (coslr=1e-3) & $\red{89.97}$ (coslr=1e-3)\\
Mem = 0 & $\red{100.04}$ (lr=1.0) & $\red{56.07}$ (lr=1.0) & $\red{213.21}$ (lr=5e-4)\\
\midrule
 & \uline{RMSProp} & \uline{Adam} & \uline{AdamSM}\\
\textbf{EMA Coordinate} & $35.01$ (coslr=1e-3) & $30.46$ (coslr=5e-3) & $32.34$ (coslr=1e-3)\\
Mem = $m\cdot n$ & $36.46$ (lr=5e-4) & $33.47$ (lr=1e-2) & $32.25$ (lr=5e-4)\\
\midrule
 & \uline{RMSPropSN} & \uline{AdamSN} & \uline{AdamSNSM}\\
\textbf{EMA Subset-Norm} & $34.86$ (coslr=1e-2) & $\green{29.75}$ (coslr=5e-2) & $\green{29.74}$ (coslr=5e-2)\\
Mem = $\max(m,n)$ & $34.57$ (lr=1e-2) & $33.69$ (lr=1e-2) & $32.49$ (lr=1e-2)\\
\midrule
 & \uline{AdaGrad} & \uline{AdaGradm} & \uline{AdaGradSM}\\
\textbf{AdaGrad Coordinate} & $37.12$ (coslr=5e-3) & $31.48$ (coslr=5e-2) & $30.99$ (coslr=5e-2)\\
Mem = $m\cdot n$ & $46.47$ (lr=5e-4) & $43.99$ (lr=1e-2) & $41.32$ (lr=5e-4)\\
\midrule
 & \uline{AdaGradSN} & \uline{AdaGradSNm} & \uline{AdaGradSNSM}\\
\textbf{AdaGrad Subset-Norm} & $33.19$ (coslr=5e-3) & $\green{29.73}$ (coslr=5e-3) & $\green{29.81}$ (coslr=5e-3)\\
Mem = $\max(m,n)$ & $41.23$ (lr=0.1) & $44.98$ (lr=0.1) & $40.11$ (lr=0.1)\\
\bottomrule
\end{tabular}
\end{table}

\paragraph{Discussions.}
From Table \ref{tab:stepsize-momentum-ablations},  Subset
norm (SN) step sizes consistently outperform coordinate-wise implementations
while requiring less memory.
Adaptivity proves crucial for optimization effectiveness, where the first row without adaptivity perform consistently poorly.
The addition of momentum is beneficial in all configurations while
SM is more beneficial for adaptive step sizes. The impact of learning
rate scheduling is also evident across configurations, with cosine
decay consistently outperforming constant learning rates. Notably, we observe varying degrees of
learning rate sensitivity: adaptive methods demonstrate greater robustness
to learning rate selection, while non-adaptive methods require more
precise tuning.

\subsection{Gradient Clipping\label{subsec:Gradient-clipping.}}

Gradient clipping is standard in training LLMs for many open source
models like LLaMA, DeepSeek, OPT, etc. \cite{deepseekai2024deepseekv2strongeconomicalefficient,touvron2023llama,workshop2022bloom,zhang2022opt,chowdhery2023palm,ding2023longnet}.
Clipping has a strong connection to stochastic gradient noise being
\emph{heavy-tailed} \cite{zhang2019gradient} and many theoretical results
have been shown to suggest some form of clipping is beneficial when
the noise could follow a heavy-tail distribution\cite{cutkosky2021high,gorbunov2020stochastic,li2022high,nguyen2023high,nguyen2023improved}.
We present the results with clipping equal to 1.0 for each method in
Table \ref{tab:clipping-ablations}.

\begin{table}
\begin{centering}
\begin{tabular}{lrrrr}
\toprule
Method  & 60M (no clipping)  & 60M (with clipping)  & 130M (no clipping)  & 130M (with clipping)\\
\midrule
Adam  & 30.58  & 30.46  & 25.07  & 25.07 \\
AdamSN  & \textbf{30.06}  & \textbf{29.75}  & 23.54  & 22.89\\
GaLore  & 34.91  & 34.73  & 25.43  & 25.31\\
\bottomrule
\end{tabular}
\par\end{centering}
\caption{Pre-training LLMs ablation experiments for gradient clipping. We compare
validation perplexity between LLaMA 60M and 130M with and without
clipping. We use the same hyperparameters as in Section \ref{pretraining-setup}
but just add clipping. \label{tab:clipping-ablations}.}
\end{table}

In Table \ref{tab:clipping-ablations}, we see that gradient clipping
indeed helps most of the methods achieve slightly better perplexity. In our
experiments, we notice that adding some form of gradient clipping
produces more stable training. 

\subsection{Batch Sizes and Random Seeds \label{subsec:Batch-size-ablation}}

\paragraph{Fixed number of steps.} We measure the impact of different batch sizes on pre-training LLaMA
60M for 10,000 steps in Table \ref{tab:Batch-size-ablation}.\footnote{This reduces the amount of total tokens trained. However, we only
compare optimizers against one another. To compare the same optimizer
against different batch sizes, one should train for the same amount
of tokens. } We use the same configuration as in other experiments. Typically, smaller
batch sizes require smaller learning rates, but curiously, AdamSNSM
seems to be stable with the choice of learning rates. Even more interestingly,
AdamSNSM's final performance seems to be affected less by the smaller
batch size as opposed to other methods, especially GaLore. 

\begin{table}
\caption{Batch size ablation for various optimizers along with optimal learning
rate. \label{tab:Batch-size-ablation}}

\begin{centering}
\begin{tabular}{lrrrrrrrr}
\toprule 
\multirow{2}{*}{Batch size} & \multicolumn{2}{c}{\textbf{Adam}} & \multicolumn{2}{c}{\textbf{GaLore}} & \multicolumn{2}{c}{\textbf{AdamSN}} & \multicolumn{2}{c}{\textbf{AdamSNSM}}\\
 & Perpl. & {\small{}LR} & Perpl. & {\small{}LR} & Perpl. & {\small{}LR} & Perpl. & {\small{}LR}\\
\midrule 
1024 & 27.94 & {\small{}0.005} & 32.75 & {\small{}0.01} & \textbf{27.68} & {\small{}0.05} & 28.02 & {\small{}0.05}\\
512 & 30.46 & {\small{}0.005} & 34.73 & {\small{}0.01} & 29.75 & {\small{}0.05} & \textbf{29.74} & {\small{}0.05}\\
256 & 36.65 & {\small{}0.001} & 44.71 & {\small{}0.001} & 37.03 & {\small{}0.001} & \textbf{32.82} & {\small{}0.05}\\
128 & 41.72 & {\small{}0.001} & 49.75 & {\small{}0.001} & 42.04 & {\small{}0.001} & \textbf{36.82} & {\small{}0.05}\\
\bottomrule
\end{tabular}
\par\end{centering}
\end{table}

\paragraph{Fixed data quantity.}
In the previous section, we compare the performances on different batch sizes fixing the same number of steps. In this section, we fix the amount of data to 1.3B tokens for pre-training LLaMA 60M. Hence, adjusting the batch size would also adjust the number of steps. Table \ref{tab:batchsize-ablation-fixed-data} contains the result where SNSM shows consistently better performance than Adam across different batch sizes. 

\paragraph{Random seeds.} Throughout our experiments, we fix the random seed for all runs within a same table. In Table \ref{tab:batchsize-ablation-fixed-data}, we investigate the effects of random seeds by running each batch size on 3 random seeds and report the mean and standard deviation. We see that SNSM has better variance than Adam for many batch sizes overall. We also examine the random variation on the 130M model in Table \ref{tab:seed130m}.
\begin{table}[ht]
\footnotesize
\centering
\caption{Mean and standard deviation (in parentheses) evaluation perplexities of Adam and AdamSNSM optimizers when pretraining LLaMA 60M for 1.3B tokens over 3 random seeds. SNSM rank = 128 and gap = 200. Learning rates were tuned over a grid for each batch size.}
\label{tab:batchsize-ablation-fixed-data}

\begin{tabular}{lccccccccc}
\toprule
Batch size &          1024 &          512  &          256  &          128  &          64   &          32   &          16   &          8    &          4    \\
\midrule
Adam      &  31.80 {\tiny (1.87)} &  30.46 {\tiny (0.29)} &  32.11 {\tiny (1.32)} &  34.57 {\tiny (0.16)} &  36.34 {\tiny (0.16)} &  38.91 {\tiny (0.12)} &  43.12 {\tiny (0.26)} &  48.88 {\tiny (0.17)} &  57.28 {\tiny (0.80)} \\
AdamwSN    &  30.11 {\tiny (0.15)} &  29.81 {\tiny (0.12)} &  30.32 {\tiny (0.07)} &  31.30 {\tiny (0.02)} &  32.72 {\tiny (0.11)} &  35.38 {\tiny (0.11)} &  40.46 {\tiny (0.97)} &  45.81 {\tiny (0.11)} &  51.01 {\tiny (0.25)} \\
AdamSNSM &  31.39 {\tiny (0.17)} &  29.93 {\tiny (0.07)} &  30.08 {\tiny (0.19)} &  30.57 {\tiny (0.08)} &  32.35 {\tiny (0.14)} &  34.51 {\tiny (0.14)} &  37.05 {\tiny (0.20)} &  39.39 {\tiny (0.02)} &  44.27 {\tiny (0.10)} \\
\bottomrule
\end{tabular}
\end{table}

\begin{table}[th]
    \centering
    \begin{tabular}{lrrrr}
        \toprule
        & Adam & AdamSN & Adagrad & AdaGradSN \\
        \midrule
        Mean    & 24.69  & 22.98  & 25.95  & 24.57  \\
        Stdev & 0.07 & 0.07 & 0.16 & 0.37 \\
        \bottomrule
    \end{tabular}
    \caption{Mean and standard deviation across 3 runs for different optimizers on pretraining LLaMA 130M task.}
    \label{tab:seed130m}
\end{table}
\section{Coordinate-Noise Density }

This section further examine the coordinate-noise density model by
providing additional empirical results across the train progress.
We also provide the full derivation for the convergence of AdaGrad
algorithms under various noise density rate. 

\subsection{Empirical Validation\label{subsec:Empirical-validation-noise-sparse} }
\paragraph{Coordinate-noise density experiments.}

\begin{figure}
    \centering
    \begin{tikzpicture}
  \begin{axis}[
      ylabel={Fraction of coordinates with noise},
      xlabel={Noise density rate $\beta$},
      legend pos=north west,
      domain=1e-6:1,
      samples=200,
      width=7cm,
      height=5cm,
      grid=both,
      xtick distance=0.1,
      ytick distance=0.2,
  ]
    \addplot [
      blue,
      thick,
      domain=1e-6:1,
      samples=200,
      variable=\t
    ] ({ln(60000000*\t)/ln(60000000)}, {\t});\addlegendentry{\(d=60\)M}
    \addplot [
      red,
      thick,
      domain=1e-6:1,
      samples=200,
      variable=\t
    ] ({ln(1e9*\t)/ln(1e9)}, {\t});
    \addlegendentry{\(d=1\)B}
  \end{axis}
\end{tikzpicture}
    \caption{Fraction of coordinates with noise over noise density rate.}
    \label{fig:coordinate-noise-density-over-fraction-of-noisy-coords}
\end{figure}
To validate the coordinate-noise density model, we sample stochastic
gradients repeatedly (via different mini batches) to obtain a sample
variance estimate for the true sub-gaussian parameter $\sigma_{i}$
for each coordinate: if $g_{1},\dots,g_{n}\in\R^{d}$
are independent stochastic gradient samples, we can calculate the
sample variance $S^{2}$ as an estimator for $\sigma^{2}$ as $S^{2}=\frac{1}{n-1}\sum_{i=1}^{n}\left(g_{i}-\bar{g}\right)^{2},$
where $\bar{g}=\frac{1}{n}\sum_{i=1}^{n}g_{i}$ is the sample mean.
We pick $n=200$ samples (with batch size equals 128) for estimating
coordinate-noise on LLaMA 60M across various steps during the training
process.
Figure \ref{fig:Aggragated-noise-density} shows the aggregated noise
distribution across \emph{all} parameters for LLaMA 60M after 100 training
steps. There, the noise is quite low for the vast
majority of coordinates except for some outliers. While the noise
seems sparse in aggragate, a more fine-grained analysis, presented
in Figure \ref{fig:noise-density-10}, shows that noises are dense
per parameter, except for the $Q$ and $K$ attention projections
in the deeper layers. Figures \ref{fig:Normalized-noise-density-0}
to \ref{fig:Noise-density-9999} in Appendix \ref{subsec:Empirical-validation-noise-sparse}
present more noise density rates across various parameters throughout
different points of the training progress.

\begin{figure}[ht]
\centering{}%
\begin{center}
\includegraphics[width=0.65\columnwidth]{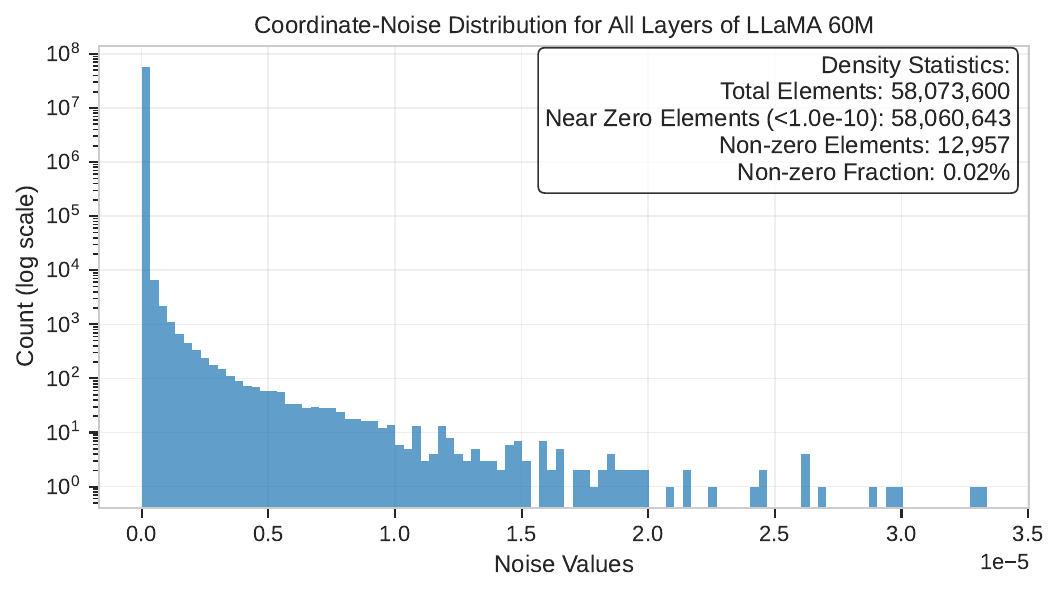}
\par\end{center}
\caption{Aggregated noise distribution across \emph{all} parameters after 100 steps of training. \label{fig:Aggragated-noise-density}}
\end{figure}

Figure \ref{fig:Normalized-noise-density-0} to \ref{fig:Normalized-noise-density-5000}
show the normalized noise density ratio for different parameters of
LLaMA 60M as described in Section \ref{subsec:Coordinate-noise-sparsity-and}.
The noise patterns show a clear layer-dependent structure, where early
layers (like layer 0) maintain consistently high density (close to
1.0) throughout training, while deeper layers start very sparse and
gradually become denser as training progresses. Notably, the embedding
layer shows an opposite trend, starting relatively dense and becoming
increasingly sparse by step 5000, suggesting different dynamics for
embedding updates compared to attention layers. The middle layers
show an interesting transition pattern, starting sparse but rapidly
becoming dense after about 1000 steps, indicating a potential critical
phase in training where these layers become more actively involved
in learning.

\begin{figure}
\begin{centering}
\includegraphics[width=0.99\textwidth]{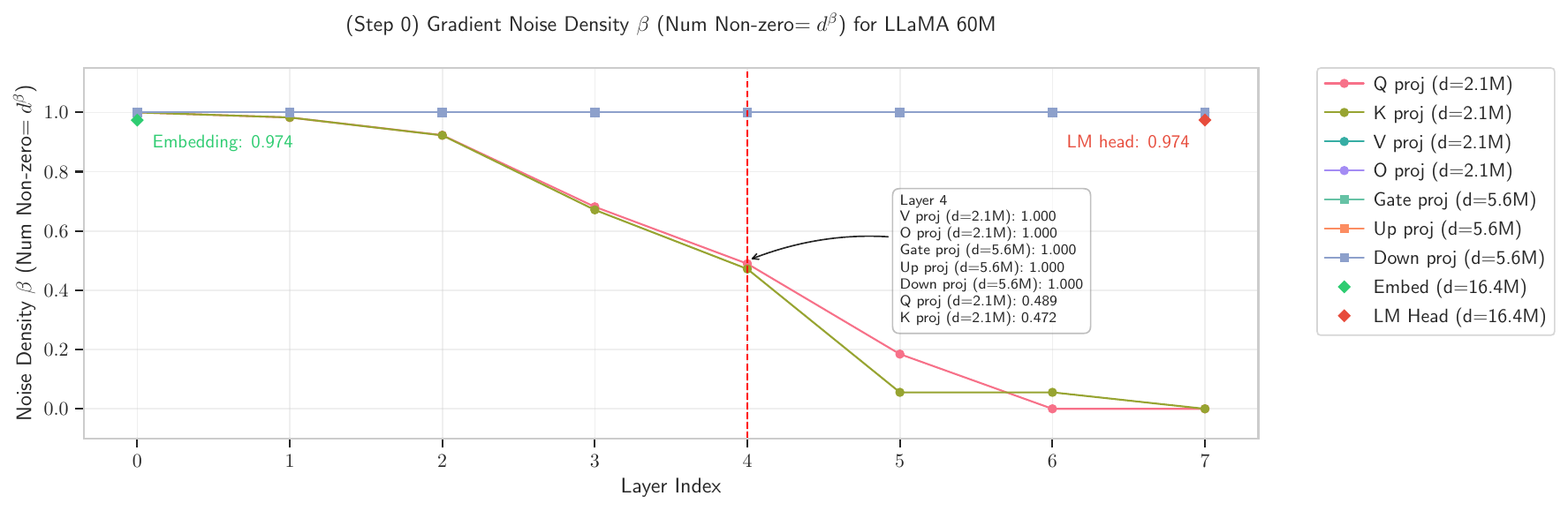}
\par\end{centering}
\caption{Noise density for different parameters of LLaMA 60M at Step 0.\label{fig:Normalized-noise-density-0}}
\end{figure}
\begin{figure}
\begin{centering}
\includegraphics[width=0.99\textwidth]{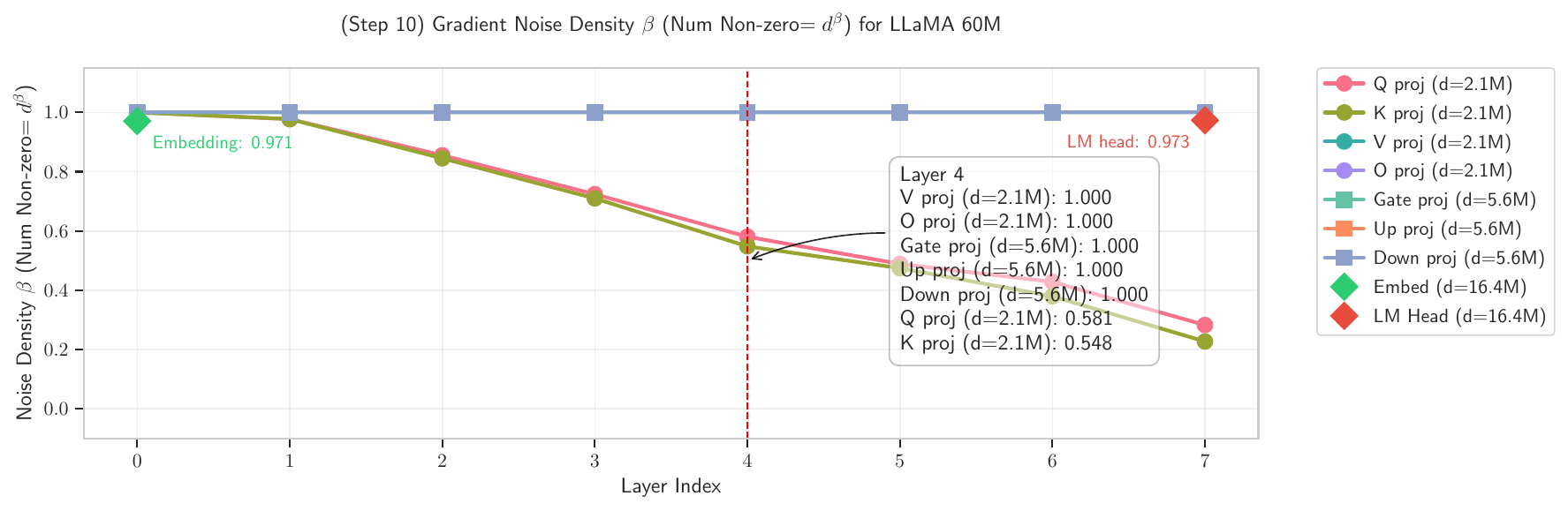}
\par\end{centering}
\caption{Noise density for different parameters of LLaMA 60M at Step 10.\label{fig:Noise-density-10}}
\end{figure}
\begin{figure}
\begin{centering}
\includegraphics[width=0.99\textwidth]{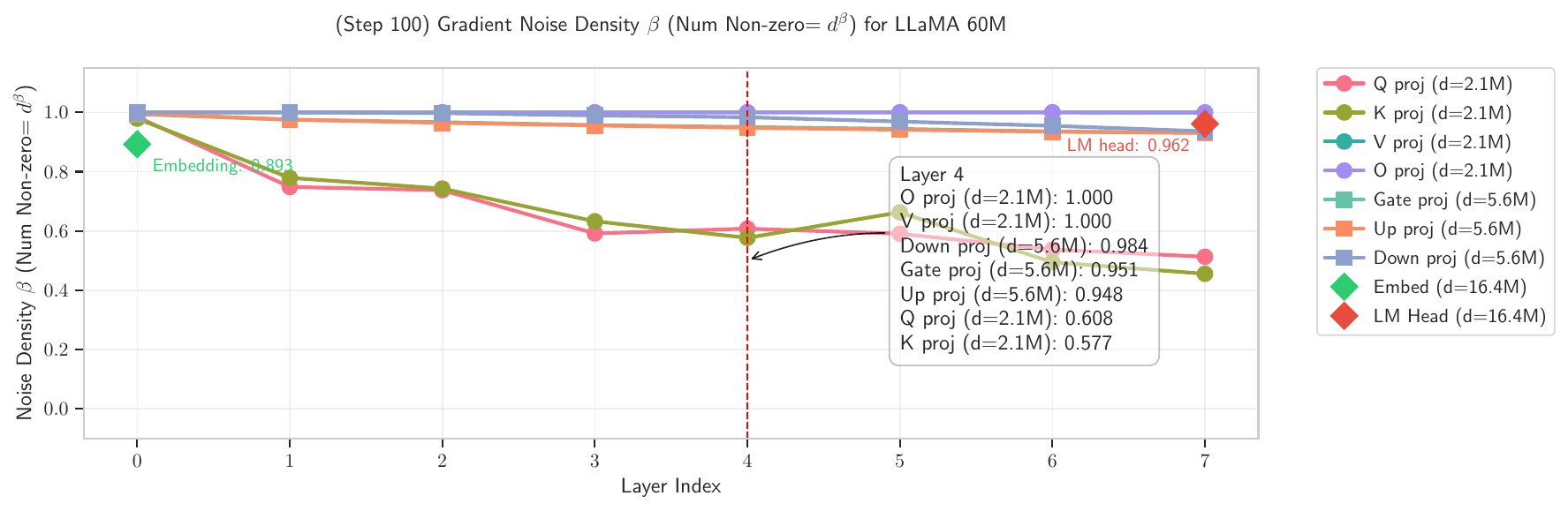}
\par\end{centering}
\caption{Noise density for different parameters of LLaMA 60M at Step 100.}
\end{figure}
\begin{figure}
\begin{centering}
\includegraphics[width=0.99\textwidth]{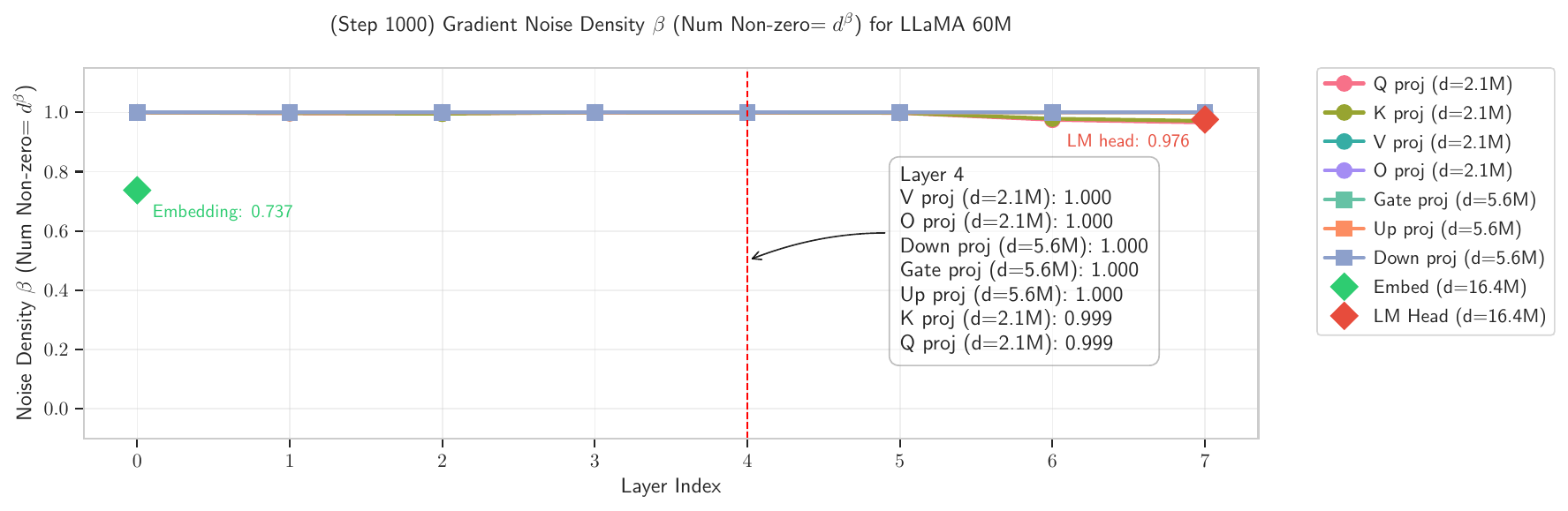}
\par\end{centering}
\caption{Noise density for different parameters of LLaMA 60M at Step 1000.}
\end{figure}
\begin{figure}
\begin{centering}
\includegraphics[width=0.99\textwidth]{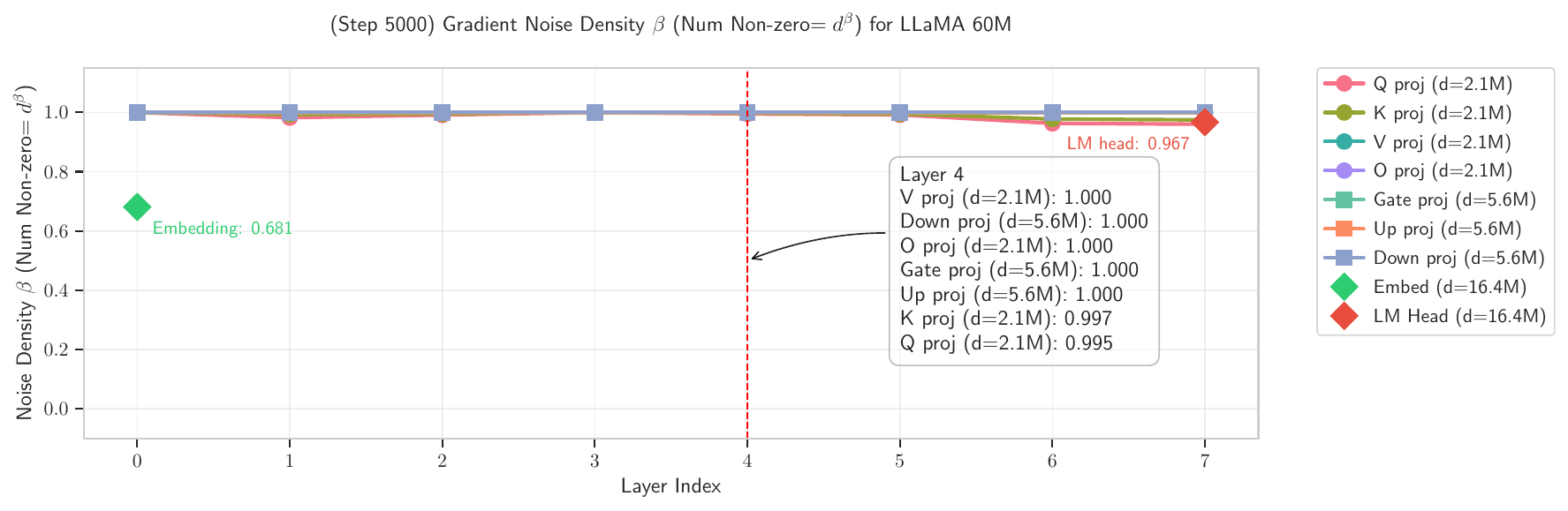}
\par\end{centering}
\caption{Noise density for different parameters of LLaMA 60M at Step 5000.\label{fig:Normalized-noise-density-5000}}
\end{figure}
\begin{figure}
\begin{centering}
\includegraphics[width=0.99\textwidth]{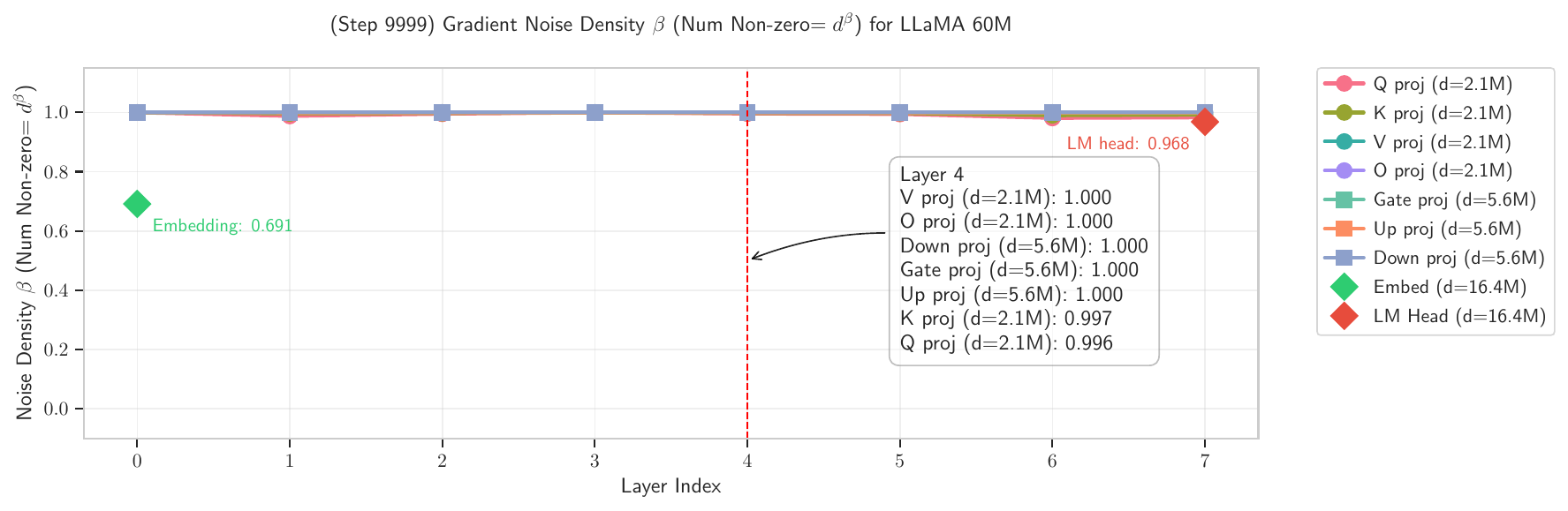}
\par\end{centering}
\caption{Noise density for different parameters of LLaMA 60M at Step 9999.
\label{fig:Noise-density-9999}}
\end{figure}

\subsection{Convergence Rate Derivation\label{sec:coordinate-noise-sparsity-derivation} }

We derive the dimensional dependency of convergence rates for different
AdaGrad variants below.

\paragraph{AdaGrad-Coordinate. }

For $c=d$ (AdaGrad-Coordinate), we get $\sum_{i=0}^{c-1}\norm{\sigma_{\Psi_{i}}}=\alpha d^{\beta}$,
$\norm{\sigma}_{2}^{2}=\alpha^{2}d^{\beta}$, and $\sum_{i=0}^{c-1}\norm{\sigma_{\Psi_{i}}}^{4}=\alpha^{4}d^{\beta}$,
so the bound from Theorem \ref{thm:main-thm-simplified} becomes 
\begin{align*}
\frac{1}{T}\sum_{t=1}^{T}\norm{\nabla_{t}}_{2}^{2} & \le\tilde{O}\left(\alpha^{4}d^{\beta}+\alpha^{3}d^{\beta}+dL+d^{1.5}\alpha\right)\cdot\tilde{O}\left(\frac{\alpha d^{\beta}}{\sqrt{T}}+\frac{\alpha^{2}d^{\beta}+\alpha d^{\beta}+Ld}{T}\right).
\end{align*}
The dependency on $d$ for the slow term $O(1/\sqrt{T})$ is $d^{1.5}d^{\beta}=d^{1.5+\beta}$.
The dependency on $d$ for the fast term $O(1/T)$ is $d^{1.5}d=d^{2.5}$.
Note that there is an inherent $d^{1.5}$ dependency for the slow
term that does not reduce as the coordinate-noise density decrease.

\paragraph{AdaGrad-Norm}

For $c=1$ (AdaGrad-Norm), we get $\norm{\sigma}_{2}^{2}=\sum_{i=0}^{d}\norm{\sigma_{i}}^{2}=\alpha^{2}d^{\beta}$,
$\norm{\sigma}_{2}=\alpha d^{\beta/2}$, and $\norm{\sigma}^{4}=\alpha^{4}d^{2\beta}$.
This means that our bound from Theorem \ref{thm:main-thm-simplified}
becomes 
\[
\frac{1}{T}\sum_{t=1}^{T}\norm{\nabla_{t}}_{2}^{2}\le\tilde{O}\left(\alpha^{4}d^{2\beta}+\alpha^{3}d^{\beta}+L+\alpha\right)\cdot\tilde{O}\left(\frac{\alpha d^{\beta/2}}{\sqrt{T}}+\frac{\alpha^{2}d^{\beta}+\alpha d^{\beta/2}+L}{T}\right).
\]
The dependency on $d$ for the slow term $O(1/\sqrt{T})$ is $d^{2\beta}\cdot d^{\beta/2}=d^{2.5\beta}$.
The dependency on $d$ for the fast term $O(1/T)$ is $d^{2\beta}\cdot d^{\beta}=d^{3\beta}$.
Note that when $\beta=0$, or when all the noise is on a single coordinate,
we recover the dimension-free results of previous works.

\paragraph{AdaGrad-Subset-Norm.}

Now, consider the following partition strategy, where we divide the
coordinates into $c=d^{1-\beta}k$ subsets of size $d^{\beta}/k$
each with the $d^{\beta}$ noisy coordinates into just $k$ subsets
so that the rest of the $c-k$ subsets do not contain any noisy coordinate. This is a reasonable choice due to the empirical validation from Section \ref{subsec:Empirical-validation-noise-sparse}: The noisy parameters seem to cluster in groups corresponding to the architecture.

With this strategy, we have $\norm{\sigma_{\Psi_{j}}}_{2}^{2}=\alpha^{2}d^{\beta}/k\implies\norm{\sigma_{\Psi_{j}}}_{2}=\alpha d^{\beta/2}/k^{0.5}$
if $j$ is a noisy subset. We can compute $\sum_{i=0}^{c-1}\norm{\sigma_{\Psi_{i}}}=\alpha d^{\beta/2}k^{0.5}$,
$\norm{\sigma}_{2}^{2}=\sum_{i=0}^{c-1}\norm{\sigma_{\Psi_{i}}}_{2}^{2}=\alpha^{2}d^{\beta}$,
and $\sum_{i=0}^{c-1}\norm{\sigma_{\Psi_{i}}}^{4}=\alpha^{4}d^{2\beta}/k$.
From Theorem \ref{thm:main-thm-simplified}, we get a bound of 
\begin{align*}
\frac{1}{T}\sum_{t=1}^{T}\norm{\nabla_{t}}_{2}^{2} & \le\tilde{O}\left(\alpha^{4}d^{2\beta}/k+\alpha^{3}d^{\beta}+d^{1-\beta}kL+\left(d^{1-\beta}k\right)^{3/2}\alpha\right)\cdot\\
 & \quad\quad\tilde{O}\left(\frac{\alpha d^{\beta/2}k^{0.5}}{\sqrt{T}}+\frac{\alpha^{2}d^{\beta}+\alpha d^{\beta/2}k^{0.5}+Ld^{1-\beta}k}{T}\right).
\end{align*}
Set $k=d^{7\beta/5-3/5}$ so that $\left(d^{1-\beta}k\right)^{3/2}=d^{2\beta}/k=d^{3\beta/5+3/5}$.
Then we can simplify 
\begin{align*}
\frac{1}{T}\sum_{t=1}^{T}\norm{\nabla_{t}}_{2}^{2} & \le\tilde{O}\left(\alpha^{4}d^{3(\beta+1)/5}+\alpha^{3}d^{\beta}+d^{2(\beta+1)/5}L+d^{3(\beta+1)/5}\alpha\right)\cdot\\
 & \quad\quad\tilde{O}\left(\frac{\alpha d^{(12\beta-3)/10}}{\sqrt{T}}+\frac{\alpha^{2}d^{\beta}+\alpha d^{(12\beta-3)/10}+Ld^{2(\beta+1)/5}}{T}\right).
\end{align*}
The dependency on $d$ for the slow term $O(1/\sqrt{T})$ is $d^{3(\beta+1)/5}\cdot d^{(12\beta-3)/10}=d^{3(1+6\beta)/10}=d^{0.3+1.8\beta}$.
The dependency on $d$ for the fast term $O(1/T)$ is a bit more complicated:
For $\beta\in[0,\frac{2}{3}]$, we have the dependency on $d$ is
$d^{3(\beta+1)/5}\cdot d^{2(\beta+1)/5}=d^{\beta+1}$. For $\beta\in[\frac{2}{3},1]$,
we have the dependency on $d$ is $d^{3(\beta+1)/5}\cdot d^{\beta}=d^{3(\beta+1)/5+\beta}=d^{1.6\beta+0.6}$.
Note that this is only a possible partition strategy where the subset
sizes are of equal size (which is probably the most natural and easiest
to implement). There, the optimal subset size is $k=d^{1.4\beta-0.6}$,
for which if we plug in $\beta\in[0,1]$ we get a range from $1$
to $d^{0.8}$. 

\subsection{From Theory to Practice}
Our theory provides an optimal grouping strategy that depends on the noise density. However, in practice, we must trade off the cost to figure out a good grouping and the performance gain from it.
The key from the theory improvement is to group the coordinates with similar noise magnitudes together. However, any expensive method to figure out these groups (e.g. the Hessian in Adam-mini) would have detrimental effects on the wall clock time and memory. Instead, our heuristic as in Section \ref{par:subset-size-heuristics} is meant to be a simple method to capture most of these groups.

Intuitively, coordinates in the same row/column either act on the same input or are used to compute the same output. The noise and normalization on each input and output would affect coordinates in the same row/columns in a correlated way. To provide some evidence, we perform the experiments in Section \ref{subsec:Empirical-validation-noise-sparse} again in Figure \ref{fig:noise-heuristic-group}, but with the noise grouped by the corresponding dimension according to the heuristics. 
\begin{figure}
    \centering
    \includegraphics[width=0.5\linewidth]{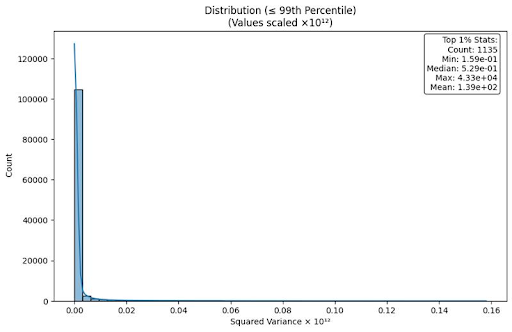}
    \caption{Coordinate noise grouped by the corresponding dimension according to the heuristics}
    \label{fig:noise-heuristic-group}
\end{figure}
There we see most groups have very low noise (very close to 0, namely less than $10^{-12}$) while a small number of groups (top 1 percentile in the annotation) have much larger noises.
Overall, our heuristics aim to capture the similar noise coming from the same inputs and outputs. Our experiments suggest that this is a major part of the gain. There might be other simple sources of correlation in the noise magnitudes, which we leave for future work.

\section{Subspace-Momentum Convergence Proof \label{sec:Subspace-Momentum-Convergence}}

In this section, we provide a high-probability convergence proof for SGD
with Subspace-Momentum for non-convex smooth objective under sub-gaussian gradient noise.

\subsection{Setup and intuition \label{subsec:Subspace-momentum-SGD-setup}}

\paragraph{Notations.}

Given a linear operator $P:\R^{d}\rightarrow\R^{k}$, we have $P^{*}:\R^{k}\rightarrow\R^{d}$
is $P$'s adjoint\footnote{In $\R^d$, the adjoint $P^*$ of a linear operator $P$ is the linear operator given by the \textit{transpose} of the matrix representation of $P$. We can also generalize Subspace-Momentum to general Hilbert spaces.}, and we consider $P^{*}P:\R^{d}\rightarrow\R^{d}$
is a projection operator i.e. $P^{*}P$ is a bounded linear operator
such that $\left(P^{*}P\right)^{2}=P^{*}P$. Given a space $V\subseteq\R^{d}$,
we denote its orthogonal subspace by $V^{\perp}:=\left\{ v\in\R^{d}:\left\langle v,u\right\rangle =0,\ \forall u\in V\right\} $. 

Let $U=\text{row}\left(P\right)\subseteq\R^{d}$ be the row span of
$P$. Let $\Psi:\R^{d}\rightarrow U$ be $\Psi(x)=P^{*}Px$ and $\Psi^{\perp}:\R^{d}\rightarrow U^{\perp}$
be $\Psi^{\perp}(x)=x-P^{*}Px$. Then for any vector $x$ in $\R^{d}$,
have the orthogonal decomposition
\[
x=\Psi(x)+\Psi^{\perp}(x).
\]

\paragraph{SGD with Subspace Momentum.}

Let $g_{t}:=\hn f(x_{t})$ denotes the stochastic gradient at time
$t$. Let $\hat{c}_{t}=Pg_{t}$, $g_{t}^{U}=\Psi g_{t}=P^{*}Pg_{t}\in U$,
and $g_{t}^{\perp}=g_{t}-g_{t}^{U}\in U^{\perp}$. Let $\nabla_{t}:=\nabla f(x_{t})$
be a short hand for the gradient at time $t$ and let $\nabla_{t}^{U}:=\Psi\left(\nabla f(x_{t})\right)\in U$
and $\nabla_{t}^{\perp}:=\Psi^{\perp}\left(\nabla f(x_{t})\right)\in U^{\perp}$
be the orthogonal decomposition of $\nabla f(x_{t})$ with respect
to $U$ and $U^{\perp}$, so that $\nabla_{t}=\nabla_{t}^{U}+\nabla_{t}^{\perp}$.
Note that the superscript of a variable tries to suggest the space
that it lives in (either $U$ or $U^{\perp}$). We have the following
update rule for subspace momentum: 
\begin{align*}
\hat{m}_{t} & =\beta\hat{m}_{t-1}+(1-\beta)Pg_{t}\\
g_{t}^{\perp} & =g_{t}-P^{*}Pg_{t}\\
m_{t} & =P^{*}\hat{m}_{t}\\
x_{t+1} & =x_{t}-\eta\left(m_{t}+g_{t}^{\perp}\right).
\end{align*}
Note that 
\begin{align*}
m_{t} & =\beta P^{*}\hat{m}_{t-1}+(1-\beta)P^{*}Pg_{t}\\
 & =\beta m_{t-1}+(1-\beta)g_{t}^{U}.
\end{align*}
Expanding the terms, we see that this is just momentum in $U$
\begin{align}
m_{t} & =\beta P^{*}\hat{m}_{t-1}+(1-\beta)P^{*}Pg_{t}\nonumber \\
 & =\beta P^{*}\hat{m}_{t-1}+(1-\beta)g_{t}^{U}\nonumber \\
 & =\beta^{2}P^{*}\hat{m}_{t-2}+(1-\beta)\beta g_{t-1}^{U}+(1-\beta)g_{t}^{U}\nonumber \\
 & =(1-\beta)\sum_{i=0}^{t}\beta^{i}g_{t-i}^{U}.\label{eq:subspace-momentumterm-expansion}
\end{align}
 Hence, we can think of the update of SGD-SM as performing two separate algorithms
in the orthogonal subspaces: momentum in the subspace $U$ and SGD in
the subspace $U^{\perp}$ (see also Figure \ref{fig:subspace-momentum}) i.e. if we decompose $x_{t}$ into its orthogonal
components $x_{t}=x_{t}^{U}+x_{t}^{\perp}$, then 
\begin{align*}
x_{t+1}^{U} & =x_{t}^{U}-\eta m_{t}\\
 & =x_{t}^{U}-\eta(1-\beta)\sum_{i=0}^{t}\beta^{i}g_{t-i}^{U}\\
x_{t+1}^{\perp} & =x_{t}^{\perp}-\eta g_{t}^{\perp}.
\end{align*}
For our analysis, let $\xi_{t}:=g_{t}-\nabla_{t}$ denote the stochastic
gradient error at time $t$. We can further decompose the error into
its subspace components:
\begin{align*}
\xi_{t} & =\xi_{t}^{U}+\xi_{t}^{\perp}\\
 & =\left(g_{t}^{U}-\nabla_{t}^{U}\right)+\left(g_{t}^{\perp}-\nabla_{t}^{\perp}\right).
\end{align*}

\paragraph{Basic facts.}

We establish some facts for subspace momentum.
\begin{enumerate}
\item Pythagorean: $\norm{g_{t}}^{2}=\norm{g_{t}^{U}}^{2}+\norm{g_{t}^{\perp}}^{2}$
and $\norm{\nabla_{t}}^{2}=\norm{\nabla_{t}^{U}}^{2}+\norm{\nabla_{t}^{\perp}}^{2}$
and so on for these decompositions. 
\item Subspace smoothness: If $f$ is smooth $\norm{\nabla f(x)-\nabla f(y)}\le L\norm{x-y}$,
then due to contraction property of the projection operator, we have
that the projected gradients of $f$ are also $L$-Lipschitz: 
\begin{align}
\norm{P^{*}P\nabla f(x)-P^{*}P\nabla f(y)}^{2} & =\norm{\nabla f(x)-\nabla f(y)}\label{eq:subspacesmooth}\\
 & \le L\norm{x-y}.\nonumber 
\end{align}
\item Subspace non-bias:
\begin{align*}
\E\left[g_{t}^{U}-\nabla_{t}^{U}\right] & =\E\left[\xi_{t}^{U}\right]\\
 & =\E\left[P^{*}P\left(g_{t}-\nabla_{t}\right)\right]\\
 & =0,
\end{align*}
and similarly for the orthogonal subspace
\begin{align*}
\E\left[g_{t}^{\perp}-\nabla_{t}^{\perp}\right] & =\E\left[\xi_{t}^{\perp}\right]\\
 & =\E\left[\xi_{t}-\xi_{t}^{U}\right]\\
 & =0.
\end{align*}
\item Subspace bounded variance: if the stochastic gradient's variance is
bounded, then its subspace components are also bounded $\E\left[\norm{\xi_{t}}^{2}\right]$:
\begin{align*}
\E\left[\norm{g_{t}^{U}-\nabla_{t}^{U}}^{2}\right] & =\E\left[\norm{\xi_{t}^{U}}^{2}\right]\\
 & =\E\left[\norm{\xi_{t}}^{2}-\norm{\xi_{t}^{\perp}}^{2}\right]\\
 & \le\sigma^{2}-\E\left[\norm{\xi_{t}^{\perp}}^{2}\right],
\end{align*}
and similarly,
\[
\E\left[\norm{\xi_{t}^{\perp}}^{2}\right]\le\sigma^{2}-\E\left[\norm{\xi_{t}^{U}}^{2}\right].
\]
\end{enumerate}

\subsection{Subspace-Momentum convergence proof \label{subsec:Subspace-Momentum-Convergence-proof}}

Suppose that $f:\R^{d}\rightarrow\R$ is $L$-smooth and stochastic
gradients $\hn f(x_{t})=g_{t}$ is unbiased, i.e. $\E\left[g_{t}\right]=\nabla f(x_{t})$,
and has $\sigma$-sub-gaussian noise, i.e. $\E [\exp(\lambda ^2\norm{g_{t}-\nabla f(x_{t})}^{2})]\le \exp(\lambda^2 \sigma^{2})$ for all $\lambda$ s.t. $|\lambda|\le 1/\sigma$.
First, we will show an error bound that is a starting point for the high-probability convergence results.
\begin{lemma}
\label{lem:SM-start-pt}If $f$ is $L$-smooth, then SGD with Subspace-Momentum
(Algorithm \ref{alg:Subspace-momentum}) yields
\[
f(x_{T+1})-f(x_{1})\le-\eta\sum_{t=1}^{T}\norm{\nabla_{t}}^{2}-\eta\sum_{t=1}^{T}\left\langle \nabla_{t},\xi_{t}\right\rangle +\frac{(3-\beta)L\eta^{2}}{2\left(1-\beta\right)}\sum_{t=1}^{T}\norm{g_{t}}^{2}.
\]
\end{lemma}
\begin{remark}
    Lemma \ref{lem:SM-start-pt} shows that the optimization error of SGD-SM is quite similar
to SGD-M.
\end{remark}

\begin{proof}
Note that $m_{t}\in U$ and $r_{t}\in U^{\perp}$. Starting with smoothness,
we have 

\begin{align*}
f(x_{t+1}) & \le f(x_{t})+\left\langle \nabla f(x_{t}),x_{t+1}-x_{t}\right\rangle +\frac{L}{2}\norm{x_{t+1}-x_{t}}^{2}\\
 & =f(x_{t})-\eta\left\langle \nabla f(x_{t}),m_{t}+g_{t}^{\perp}\right\rangle +\frac{\eta^{2}L}{2}\norm{m_{t}+g_{t}^{\perp}}^{2}\\
 & =f(x_{t})-\eta\left\langle \nabla f(x_{t}),m_{t}\right\rangle -\eta\left\langle \nabla f(x_{t}),g_{t}^{\perp}\right\rangle +\frac{\eta^{2}L}{2}\norm{m_{t}+g_{t}^{\perp}}^{2}.
\end{align*}
We have 
\begin{align*}
f(x_{t+1})-f(x_{t}) & \le-\eta\left\langle \nabla_{t}^{U},m_{t}\right\rangle -\eta\left\langle \nabla_{t}^{\perp},g_{t}^{\perp}\right\rangle +\frac{\eta^{2}L}{2}\norm{m_{t}+g_{t}^{\perp}}^{2}\\
 & =-\eta\left\langle \nabla_{t}^{U},m_{t}\right\rangle -\eta\left\langle \nabla_{t}^{\perp},g_{t}^{\perp}\right\rangle +\frac{\eta^{2}L}{2}\norm{m_{t}}^{2}+\frac{\eta^{2}L}{2}\norm{g_{t}^{\perp}}^{2}.\tag{Pythagorean}
\end{align*}
Summing it up, we get 
\begin{equation}
f(x_{T+1})-f(x_{1})\le\underbrace{-\eta\sum_{t=1}^{T}\left\langle \nabla_{t}^{U},m_{t}\right\rangle +\frac{\eta^{2}L}{2}\sum_{t=1}^{T}\norm{m_{t}}^{2}}_{\text{SGD with momentum error in \ensuremath{U}}}\underbrace{-\eta\sum_{t=1}^{T}\left\langle \nabla_{t}^{\perp},g_{t}^{\perp}\right\rangle +\frac{\eta^{2}L}{2}\sum_{t=1}^{T}\norm{g_{t}^{\perp}}^{2}}_{\text{vanilla SGD error in \ensuremath{U^{\perp}}}}.\label{eq:joint-error}
\end{equation}
We analyze $-\eta\left\langle \nabla_{t}^{U},m_{t}\right\rangle +\frac{\eta^{2}L}{2}\norm{m_{t}}^{2}$
and $-\eta\left\langle \nabla_{t}^{\perp},g_{t}^{\perp}\right\rangle +\frac{\eta^{2}L}{2}\norm{g_{t}^{\perp}}^{2}$
separately. Intuitively, the error within each subspace is controlled
by their respective algorithm. Investigating the momentum term, we
have
\begin{align*}
-\left\langle \nabla_{t}^{U},m_{t}\right\rangle  & =-\left\langle \nabla_{t}^{U},\beta m_{t-1}+(1-\beta)g_{t}^{U}\right\rangle \\
 & =-\beta\left\langle \nabla_{t}^{U},m_{t-1}\right\rangle -(1-\beta)\left\langle \nabla_{t}^{U},g_{t}^{U}\right\rangle \\
 & =-\beta\left\langle \nabla_{t}^{U}-\nabla_{t-1}^{U},m_{t-1}\right\rangle -\beta\left\langle \nabla_{t-1}^{U},m_{t-1}\right\rangle -(1-\beta)\left\langle \nabla_{t}^{U},g_{t}^{U}\right\rangle .
\end{align*}
We examine $-\left\langle \nabla_{t}^{U}-\nabla_{t-1}^{U},m_{t-1}\right\rangle $:
\begin{align*}
-\left\langle \nabla_{t}^{U}-\nabla_{t-1}^{U},m_{t-1}\right\rangle  & =-\left\langle \nabla_{t}-\nabla_{t-1},m_{t-1}\right\rangle +\left\langle \nabla_{t}^{\perp}-\nabla_{t-1}^{\perp},m_{t-1}\right\rangle \\
 & =-\left\langle \nabla_{t}-\nabla_{t-1},m_{t-1}\right\rangle \\
 & \le\norm{\nabla_{t}-\nabla_{t-1}}\norm{m_{t-1}}\\
 & \le L\norm{x_{t}-x_{t-1}}\norm{m_{t-1}}\\
 & =\eta L\norm{m_{t-1}+g_{t-1}^{\perp}}\norm{m_{t-1}}\\
 & \le\eta L\norm{m_{t-1}+g_{t-1}^{\perp}}^{2}.
\end{align*}
Now we have 
\begin{align*}
-\left\langle \nabla_{t}^{U},m_{t}\right\rangle  & \le\eta L\beta\norm{m_{t-1}+g_{t-1}^{\perp}}^{2}-\beta\left\langle \nabla_{t-1}^{U},m_{t-1}\right\rangle -(1-\beta)\left\langle \nabla_{t}^{U},g_{t}^{U}\right\rangle \\
 & \le\eta L\sum_{i=1}^{t-1}\beta^{t-i}\norm{m_{i}+g_{i}^{\perp}}^{2}-(1-\beta)\sum_{i=1}^{t}\beta^{t-i}\left\langle \nabla_{i}^{U},g_{i}^{U}\right\rangle .
\end{align*}
Summing over $t$, we have 
\begin{align*}
-\eta\sum_{t=1}^{T}\left\langle \nabla_{t}^{U},m_{t}\right\rangle  & \le\eta^{2}L\sum_{t=1}^{T}\sum_{i=1}^{t-1}\beta^{t-i}\norm{m_{i}+g_{i}^{\perp}}^{2}-(1-\beta)\eta\sum_{t=1}^{T}\sum_{i=1}^{t}\beta^{t-i}\left\langle \nabla_{i}^{U},g_{i}^{U}\right\rangle \\
 & =L\eta^{2}\sum_{i=1}^{T}\sum_{t=i}^{T}\beta^{t-i}\norm{m_{i}+g_{i}^{\perp}}^{2}-(1-\beta)\eta\sum_{i=1}^{T}\sum_{t=i}^{T}\beta^{t-i}\left\langle \nabla_{i}^{U},g_{i}^{U}\right\rangle \tag{swap the sum}\\
 & \le L\eta^{2}\sum_{i=1}^{T}\norm{m_{i}+g_{i}^{\perp}}^{2}\sum_{t=i}^{T}\beta^{t}-(1-\beta)\eta\sum_{i=1}^{T}\left\langle \nabla_{i}^{U},g_{i}^{U}\right\rangle \sum_{t=i}^{T}\beta^{t}\\
 & \le\frac{L\eta^{2}}{1-\beta}\sum_{i=1}^{T}\norm{m_{i}+g_{i}^{\perp}}^{2}-\eta\sum_{i=1}^{T}\left\langle \nabla_{i}^{U},g_{i}^{U}\right\rangle \\
 & =\frac{L\eta^{2}}{1-\beta}\sum_{i=1}^{T}\left(\norm{m_{i}}^{2}+\norm{g_{i}^{\perp}}^{2}\right)-\eta\sum_{i=1}^{T}\left\langle \nabla_{i}^{U},\xi_{i}^{U}\right\rangle -\eta\sum_{i=1}^{T}\norm{\nabla_{i}^{U}}^{2}.
\end{align*}
 Now, we look at $\sum_{i=1}^{T}\norm{m_{i}}^{2}$:
\begin{align*}
\sum_{t=1}^{T}\norm{m_{t}}^{2} & =\sum_{t=1}^{T}\norm{\beta m_{t-1}+(1-\beta)g_{t}^{U}}^{2}\\
 & \le\sum_{t=1}^{T}\beta\norm{m_{t-1}}^{2}+(1-\beta)\norm{g_{t}^{U}}^{2}\tag{convexity of \ensuremath{\norm{\cdot}^{2}}}\\
 & \le\sum_{t=1}^{T}\beta\norm{m_{t}}^{2}+(1-\beta)\sum_{t=1}^{T}\norm{g_{t}^{U}}^{2}\\
\implies\sum_{t=1}^{T}\norm{m_{t}}^{2} & \le\sum_{t=1}^{T}\norm{g_{t}^{U}}^{2}.
\end{align*}
Examining the momentum error terms, we get 
\begin{align}
 & -\eta\sum_{t=1}^{T}\left\langle \nabla_{t}^{U},m_{t}\right\rangle +\frac{\eta^{2}L}{2}\sum_{t=1}^{T}\norm{m_{t}}^{2}\nonumber \\
 & \le\left(\frac{L\eta^{2}}{1-\beta}+\frac{\eta^{2}L}{2}\right)\sum_{i=1}^{T}\norm{m_{i}}^{2}+\frac{L\eta^{2}}{1-\beta}\sum_{i=1}^{T}\norm{g_{i}^{\perp}}^{2}-\eta\sum_{i=1}^{T}\left\langle \nabla_{i}^{U},\xi_{i}^{U}\right\rangle -\eta\sum_{i=1}^{T}\norm{\nabla_{i}^{U}}^{2}\nonumber \\
 & \le\left(\frac{(3-\beta)L\eta^{2}}{2\left(1-\beta\right)}\right)\sum_{i=1}^{T}\norm{g_{i}^{U}}^{2}+\frac{L\eta^{2}}{1-\beta}\sum_{i=1}^{T}\norm{g_{i}^{\perp}}^{2}-\eta\sum_{i=1}^{T}\left\langle \nabla_{i}^{U},\xi_{i}^{U}\right\rangle -\eta\sum_{i=1}^{T}\norm{\nabla_{i}^{U}}^{2}.\label{eq:momentum-control}
\end{align}
We consider the SGD error terms in the orthogonal subspace:
\begin{align}
-\eta\left\langle \nabla_{t}^{\perp},g_{t}^{\perp}\right\rangle +\frac{\eta^{2}L}{2}\norm{g_{t}^{\perp}}^{2} & =-\eta\left\langle \nabla_{t}^{\perp},\nabla_{t}^{\perp}-g_{t}^{\perp}\right\rangle -\eta\norm{\nabla_{t}^{\perp}}^{2}+\frac{\eta^{2}L}{2}\norm{g_{t}^{\perp}}^{2}\nonumber \\
 & =-\eta\left\langle \nabla_{t}^{\perp},\xi_{t}^{\perp}\right\rangle -\eta\norm{\nabla_{t}^{\perp}}^{2}+\frac{\eta^{2}L}{2}\norm{g_{t}^{\perp}}^{2}.\label{eq:sgd-control}
\end{align}
Now we are ready to combine (\ref{eq:momentum-control}) and (\ref{eq:sgd-control}).
First note the common terms $\norm{g_{i}^{\perp}}^{2}$ in both equations
combine to a sum similarly to $\norm{g_{t}^{U}}^{2}$:
\[
\underbrace{\frac{L\eta^{2}}{1-\beta}\sum_{t=1}^{T}\norm{g_{t}^{\perp}}^{2}}_{\text{momentum}}+\underbrace{\frac{\eta^{2}L}{2}\sum_{t=1}^{T}\norm{g_{t}^{\perp}}^{2}}_{\text{SGD}}=\left(\frac{(3-\beta)L\eta^{2}}{2\left(1-\beta\right)}\right)\sum_{t=1}^{T}\norm{g_{t}^{\perp}}^{2}.
\]
 Combining both terms, we see that the terms are combined from both
subspaces (red from (\ref{eq:momentum-control}) and blue from (\ref{eq:sgd-control})):
\begin{align*}
-\red{\eta\sum_{t=1}^{T}\norm{\nabla_{t}^{U}}^{2}}-\eta\blue{\sum_{t=1}^{T}\norm{\nabla_{t}^{\perp}}^{2}} & =-\eta\sum_{t=1}^{T}\norm{\nabla_{t}}^{2}\\
-\eta\red{\sum_{t=1}^{T}\left\langle \nabla_{t}^{U},\xi_{t}^{U}\right\rangle }-\eta\blue{\sum_{t=1}^{T}\left\langle \nabla_{t}^{\perp},\xi_{t}^{\perp}\right\rangle } & =-\eta\sum_{t=1}^{T}\left\langle \nabla_{t},\xi_{t}\right\rangle \\
\red{\frac{(3-\beta)L\eta^{2}}{2\left(1-\beta\right)}\sum_{t=1}^{T}\norm{g_{t}^{U}}^{2}}+\left(\red{\frac{L\eta^{2}}{1-\beta}}+\blue{\frac{\eta^{2}L}{2}}\right)\sum_{t=1}^{T}\norm{g_{t}^{\perp}}^{2} & =\frac{(3-\beta)L\eta^{2}}{2\left(1-\beta\right)}\sum_{t=1}^{T}\norm{g_{t}}^{2}.
\end{align*}
Plugging everything back into (\ref{eq:joint-error}), we have 
\begin{equation}
f(x_{T+1})-f(x_{1})\le-\eta\sum_{t=1}^{T}\norm{\nabla_{t}}^{2}-\eta\sum_{t=1}^{T}\left\langle \nabla_{t},\xi_{t}\right\rangle +\frac{(3-\beta)L\eta^{2}}{2\left(1-\beta\right)}\sum_{t=1}^{T}\norm{g_{t}}^{2}.\label{eq:final-error-bound}
\end{equation}
\end{proof}

\subsubsection{Proof of Theorem \ref{thm:SGDSM-highprob}.}

Our proof uses the technical tools from \cite{liu2023high}, although
the strategy here has been simplified. We use Lemma A.1 from \cite{liu2023high}:
\begin{lemma}[Lemma A.1 from \cite{liu2023high}]
For any $a\geq0$, $0\leq b\leq\frac{1}{2\sigma}$ and a nonnegative
$\sigma$-subgaussian random variable $X$, we have
\[
\E\left[1+b^{2}X^{2}+\sum_{i=2}^{\infty}\frac{1}{i!}\left(aX+b^{2}X^{2}\right)^{i}\right]\leq\exp\left(3\sigma^{2}\left(a^{2}+b^{2}\right)\right).
\]
\end{lemma}
We present some useful tools:
\begin{corollary}
\label{cor:mgf-subgaussian}Suppose that $X$ is a mean zero random
vector in $\R^{d}$, where $\norm X$ is $\sigma$-subgaussian. For
$0\leq a\leq\frac{1}{4\sigma^{2}}$ and $B\in\R^{d}$ then
\[
\E\left[\exp\left(a\norm X^{2}+\left\langle B,X\right\rangle \right)\right]\leq\exp\left(3\sigma^{2}(a+\norm B^{2})\right).
\]
\end{corollary}
\begin{proof}
We have
\begin{align*}
\E\left[\exp\left(a\norm X^{2}+\left\langle B,X\right\rangle \right)\right] & =\E\left[1+a^{2}\norm X^{2}+\left\langle B,X\right\rangle +\sum_{k=2}^{\infty}\frac{1}{k!}\left(a\norm X^{2}+\left\langle B,X\right\rangle \right)^{k}\right]\\
 & =\E\left[1+a\norm X^{2}+\sum_{k=2}^{\infty}\frac{1}{k!}\left(a\norm X^{2}+\left\langle B,X\right\rangle \right)^{k}\right]\\
 & \leq\E\left[1+a\norm X^{2}+\sum_{k=2}^{\infty}\frac{1}{k!}\left(a\norm X^{2}+\norm B\norm X\right)^{k}\right]\\
 & \leq\exp\left(3\sigma^{2}(a+\norm B^{2})\right).
\end{align*}
\end{proof}

We can now control martingale via:
\begin{lemma}
\label{lem:martingale-diff-bound}If we have a sequence of random
variable $X_{t}$ with $\F_{t}=\sigma(X_{1},X_{2},\dots,X_{t-1})$
for $t=1,2,\dots,T$. If we can bound $\E\left[\exp\left(X_{t}\right)\mid\F_{t}\right]\leq\exp(Y_{t})$,
where $Y_{t}$ is $\F_{t}$-measurable, then 
\[
\sum_{t=1}^{T}X_{t}\leq\sum_{t=1}^{T}Y_{t}+\log\left(1/\delta\right)
\]
holds with probability at least $1-\delta$.
\end{lemma}
\begin{proof}
Define the $Z_{t}=X_{t}-Y_{t}$ and $S_{t}=\sum_{i=t}^{T}Z_{i}$.
Then 
\begin{align*}
\E\left[\exp\left(Z_{t}\right)\mid\F_{t}\right] & =\E\left[\exp\left(X_{t}-Y_{t}\right)\mid\F_{t}\right]\\
 & =\exp\left(-Y_{t}\right)\E\left[\exp\left(X_{t}\right)\mid\F_{t}\right]\tag{\ensuremath{Y_{t}} is \ensuremath{\F_{t}}-measurable}\\
 & \leq\exp(-Y_{t})\exp(Y_{t})=\exp(0)=1.
\end{align*}
Then we show $\E\left[\exp\left(S_{1}\right)\right]\leq1$ via an
induction: we have $\E\left[\exp\left(S_{T}\right)\mid\F_{T}\right]=\E\left[\exp\left(Z_{T}\right)\mid\F_{T}\right]\leq1.$
Suppose that $\E\left[\exp\left(S_{t+1}\right)\mid\F_{t+1}\right]$
\begin{align*}
\E\left[\exp(S_{t})\mid\F_{t}\right] & =\E\left[\exp(Z_{t})\exp(S_{t+1})\mid\F_{t}\right]\\
 & =\E\left[\exp\left(Z_{t}\right)\E\left[\exp\left(S_{t+1}\right)\mid\F_{t+1}\right]\mid\F_{t}\right]\\
 & \leq\E\left[\exp(Z_{t})\mid\F_{t}\right]\leq1.
\end{align*}
Hence, this implies that $\E\left[\exp(S_{1})\right]\leq1$. By Markov's
inequality, this means that $S_{1}\leq\log(\frac{1}{\delta})$ with
probability at least $1-\delta$:
\begin{align*}
S_{1} & =\sum_{t=1}^{T}Z_{t}=\sum_{t=1}^{T}X_{t}-Y_{t}\leq\log\left(1/\delta\right)\\
\implies & \sum_{t=1}^{T}X_{t}\leq\sum_{t=1}^{T}Y_{t}+\log\left(1/\delta\right).
\end{align*}
\end{proof}

We can now prove Theorem \ref{thm:SGDSM-highprob}:
\begin{proof}[Proof of Theorem \ref{thm:SGDSM-highprob}]
Starting from Lemma \ref{lem:SM-start-pt} and letting $\alpha=\frac{(3-\beta)L}{2\left(1-\beta\right)}$
and $\Delta_{1}:=f(x_{1})-f_{*}$ for simplicity, we have 
\begin{align*}
 & \Delta_{T+1}-\Delta_{1}\\
\le & -\eta\sum_{i=1}^{T}\norm{\nabla_{t}}^{2}-\eta\sum_{i=1}^{T}\left\langle \nabla_{t},\xi_{t}\right\rangle +\alpha\eta^{2}\sum_{t=1}^{T}\norm{g_{t}}^{2}\\
= & -\eta\sum_{i=1}^{T}\norm{\nabla_{t}}^{2}-\eta\sum_{i=1}^{T}\left\langle \nabla_{t},\xi_{t}\right\rangle +\alpha\eta^{2}\sum_{t=1}^{T}\norm{\xi_{t}+\nabla_{t}}^{2}\\
= & \eta\left(\alpha\eta-1\right)\sum_{i=1}^{T}\norm{\nabla_{t}}^{2}+\eta\left(\alpha\eta-1\right)\sum_{i=1}^{T}\left\langle \nabla_{t},\xi_{t}\right\rangle +\alpha\eta^{2}\sum_{t=1}^{T}\norm{\xi_{t}}^{2}.
\end{align*}
Rearranging and defining some weight $w>0$, we have 
\[
w\left(\Delta_{T+1}-\Delta_{1}\right)+\eta w\left(1-\alpha\eta\right)\sum_{i=1}^{T}\norm{\nabla_{t}}^{2}\le\eta w\left(\alpha\eta-1\right)\sum_{i=1}^{T}\left\langle \nabla_{t},\xi_{t}\right\rangle +\alpha w\eta^{2}\sum_{t=1}^{T}\norm{\xi_{t}}^{2}.
\]
Let $\F_{t}:=\sigma\left(\xi_{1},\dots,\xi_{t-1}\right)$ denote the
natural filtration. Now, since $\E\left[\sum_{t=1}^{T}\left\langle \nabla_{t},\xi_{t}\right\rangle \right]=0$
and $\xi_{t}$ is $\sigma$-sub-gaussian, we have that $\nabla_{t}\in\F_{t}$
and so if $0\le w\alpha\eta^{2}\le\frac{1}{4\sigma^{2}}$, Corollary
\ref{cor:mgf-subgaussian} implies 
\[
\E\left[\exp\left(w\eta\left(\alpha\eta-1\right)\left\langle \nabla_{t},\xi_{t}\right\rangle +w\alpha\eta^{2}\norm{\xi_{t}}^{2}\right)\mid\F_{t}\right]\le\exp\left(3\sigma^{2}\left(w\alpha\eta^{2}+w^{2}\eta^{2}\left(\alpha\eta-1\right)^{2}\norm{\nabla_{t}}^{2}\right)\right),
\]
Then Lemma \ref{lem:martingale-diff-bound} implies that with probability
at least $1-\delta$, we have
\begin{align*}
w\eta\left(\alpha\eta-1\right)\sum_{t=1}^{T}\left\langle \nabla_{t},\xi_{t}\right\rangle +w\alpha\eta^{2}\sum_{t=1}^{T}\norm{\xi_{t}}^{2} & \leq3\sigma^{2}\sum_{t=1}^{T}\left(w\alpha\eta^{2}+w^{2}\eta^{2}\left(\alpha\eta-1\right)^{2}\norm{\nabla_{t}}^{2}\right)+\log\left(1/\delta\right)\\
 & =3\sigma^{2}w\eta^{2}\alpha T+3\sigma^{2}w^{2}\eta^{2}\left(\alpha\eta-1\right)^{2}\sum_{t=1}^{T}\norm{\nabla_{t}}^{2}+\log\left(1/\delta\right).
\end{align*}
Then with probability at least $1-\delta$, we have 
\begin{align*}
w\left(\Delta_{T+1}-\Delta_{1}\right)+\eta w\left(1-\alpha\eta\right)\sum_{i=1}^{T}\norm{\nabla_{t}}^{2} & \le\eta w\left(\alpha\eta-1\right)\sum_{i=1}^{T}\left\langle \nabla_{t},\xi_{t}\right\rangle +\alpha w\eta^{2}\sum_{t=1}^{T}\norm{\xi_{t}}^{2}\\
 & \le3\sigma^{2}w\eta^{2}\alpha T+3\sigma^{2}w^{2}\eta^{2}\left(\alpha\eta-1\right)^{2}\sum_{t=1}^{T}\norm{\nabla_{t}}^{2}+\log\left(1/\delta\right)\\
\implies\eta w\left(1-\alpha\eta\right)\sum_{i=1}^{T}\norm{\nabla_{t}}^{2} & \le w\Delta_{1}+3\sigma^{2}w\eta^{2}\alpha T+3\sigma^{2}w^{2}\eta^{2}\left(\alpha\eta-1\right)^{2}\sum_{t=1}^{T}\norm{\nabla_{t}}^{2}+\log\left(1/\delta\right).
\end{align*}
Combining the $\norm{\nabla_{t}}^{2}$ terms, we get 
\begin{equation}
\left(\eta w\left(1-\alpha\eta\right)-3\sigma^{2}w^{2}\eta^{2}\left(\alpha\eta-1\right)^{2}\right)\sum_{i=1}^{T}\norm{\nabla_{t}}^{2}\le w\Delta_{1}+3\sigma^{2}w\eta^{2}\alpha T+\log\left(1/\delta\right).\label{eq:before-final-bound-sgdsm}
\end{equation}
Setting $w=\frac{1}{12\sigma^{2}\eta}$, then 
\begin{align*}
\eta w\left(1-\alpha\eta\right)-3\sigma^{2}w^{2}\eta^{2}\left(\alpha\eta-1\right)^{2} & =\eta w\left(1-\alpha\eta-3\sigma^{2}w\eta\left(\alpha\eta-1\right)^{2}\right)\\
 & =\eta w\left(1-\alpha\eta-\frac{1}{4}\left(\alpha\eta-1\right)^{2}\right)\\
 & \ge\eta w\frac{1}{4}.
\end{align*}
 if $1-\alpha\eta\ge1/2$. Furthermore, we have that $w\alpha\eta^{2}=\frac{\alpha\eta}{12\sigma^{2}}\le\frac{1}{4\sigma^{2}}$
if $\eta\le\frac{3}{\alpha}$, as required for Corr \ref{cor:mgf-subgaussian}.
Hence, if $\eta\le\frac{1}{2\alpha}$ then both requirements are satisfied.
Consider the LHS of \ref{eq:before-final-bound-sgdsm}, we can bound
\begin{align*}
\left(\eta w\left(1-\alpha\eta\right)-3\sigma^{2}w^{2}\eta^{2}\left(\alpha\eta-1\right)^{2}\right)\sum_{i=1}^{T}\norm{\nabla_{t}}^{2} & \ge\eta w\frac{1}{4}\sum_{i=1}^{T}\norm{\nabla_{t}}^{2}
\end{align*}
Finally, we have 
\begin{align*}
\frac{\eta w}{4}\sum_{i=1}^{T}\norm{\nabla_{t}}^{2} & \le w\Delta_{1}+3\sigma^{2}w\eta^{2}\alpha T+\log\left(1/\delta\right)\\
\sum_{i=1}^{T}\norm{\nabla_{t}}^{2} & \le\frac{4}{\eta}\Delta_{1}+3\sigma^{2}\eta\alpha T+48\sigma^{2}\log\left(1/\delta\right).
\end{align*}
Setting $\eta=\min\left\{ \frac{1}{2\alpha}; \sqrt{\frac{\Delta_1}{\sigma^2 \alpha T}} \right\}$, we have that with probability
at least $1-\delta$
\begin{align*}
\sum_{i=1}^{T}\norm{\nabla_{t}}^{2}&\le\frac{4}{\eta}\Delta_{1}+3\sigma^{2}\eta\alpha T+48\sigma^{2}\log\left(1/\delta\right)\\&=\frac{4}{\min\left\{ \frac{1}{2\alpha};\sqrt{\frac{\Delta_{1}}{\sigma^{2}\alpha T}}\right\} }\Delta_{1}+3\sigma^{2}\min\left\{ \frac{1}{2\alpha};\sqrt{\frac{\Delta_{1}}{\sigma^{2}\alpha T}}\right\} \alpha T+48\sigma^{2}\log\left(1/\delta\right)\\&\le4\left(2\alpha+\sqrt{\frac{\sigma^{2}\alpha T}{\Delta_{1}}}\right)\Delta_{1}+3\sigma^{2}\sqrt{\frac{\Delta_{1}}{\sigma^{2}\alpha T}}\alpha T+48\sigma^{2}\log\left(1/\delta\right)\\&=8\Delta_{1}\alpha+4\sigma\sqrt{\alpha T\Delta_{1}}+3\sigma\sqrt{\Delta_{1}\alpha T}+48\sigma^{2}\log\left(1/\delta\right)\\&=8\Delta_{1}\alpha+7\sigma\sqrt{\alpha T\Delta_{1}}+48\sigma^{2}\log\left(1/\delta\right)\\
\implies\frac{1}{T}\sum_{i=1}^{T}\norm{\nabla_{t}}^{2}&\le\frac{8\Delta_{1}\alpha}{T}+\frac{7\sigma\sqrt{\alpha\Delta_{1}}}{\sqrt{T}}+\frac{48\sigma^{2}\log\left(1/\delta\right)}{T}.
\end{align*}
We are done.
\end{proof}

\section{Subset-Norm adaptive step size full theorem and proof\label{sec:Full-Theorem-and}}

We show the full result in Theorem \ref{thm:full-thm} with all the
polylog terms omitted from Theorem \ref{thm:main-thm-simplified}. 
\begin{theorem}
\label{thm:full-thm}Suppose that $f:\R^{d}\rightarrow\R$ is $L$-smooth
and lower bounded by $f_{*}$. Given unbiased stochastic gradients
$\widehat{\nabla}f(x_{t})$ with stochastic gradient noise $\xi_{t}:=\widehat{\nabla}f(x_{t})-\nabla f(x_{t})$
being $\sigma_{i}$-per-coordinate subgaussian for $i\in[d]$. For
partitions of the parameters into disjoint subsets $[d]=\bigcup_{i=0}^{c-1}\Psi_{i}$
with $\Psi_{i}\cap\Psi_{j}=\emptyset,\ \text{if }i\neq j$, the iterates
$x_{t}$ given by (\ref{eq:adagrad-subset-norm}) satisfies the following
inequality with probability at least $1-6c\delta$ (for failure probability
$\delta>0$): 
\begin{align*}
\frac{1}{T}\sum_{t=1}^{T}\norm{\nabla_{t}}_{2}^{2} & \le G(\delta)\cdot\left(\frac{4\sum_{i=0}^{c-1}\norm{\sigma_{\Psi_{i}}}}{\sqrt{T}}+\frac{I(\delta)}{T}\right),\ \text{where \ensuremath{G(\delta)} and \ensuremath{I(\delta)} are polylog terms:}\\
G(\delta):= & \frac{\Delta_{1}}{\eta}+H(\delta)+\left(\ln T/\delta\norm{\sigma}_{2}^{2}+c\eta L+4c^{3/2}\sigma_{\max}\sqrt{\log\frac{1}{\delta}}\right)\log\left(\frac{4\sqrt{T}\sum_{i=0}^{c-1}\norm{\sigma_{\Psi_{i}}}+I(\delta)}{b_{0,\min}}\right)\\
I(\delta):= & \norm{b_{0}}_{1}+\frac{2\Delta_{1}}{\eta}+\frac{8\log\frac{1}{\delta}}{b_{0,\min}}\norm{\sigma}_{2}^{2}+\sqrt{\log\frac{1}{\delta}}\sum_{i=0}^{c-1}\norm{\sigma_{\Psi_{i}}}+8\eta Lc\log\frac{4\eta L}{b_{0,\min}}\\
H(\delta):= & \sum_{i=0}^{c-1}\left(\ln\left(T/\delta\right)\norm{\sigma_{\Psi_{i}}}^{2}+2\alpha\right)\left(\frac{8\norm{\sigma_{\Psi_{i}}}^{2}\log\frac{1}{\delta}}{b_{0,i}^{2}}+2\log\left(1+\norm{\sigma_{\Psi_{i}}}^{2}T+\norm{\sigma_{\Psi_{i}}}^{2}\log\frac{1}{\delta}\right)\right).
\end{align*}
where $\norm{\sigma}_{2}^{2}=\sum_{i=1}^{d}\sigma_{i}^{2}$, $\norm{\sigma_{\Psi_{i}}}^{2}=\sum_{j\in\Psi_{i}}\sigma_{j}^{2}$,
$\sigma_{\max}=\max_{i\in[d]}\sigma_{i}$, $\Delta_{1}=f(x_{1})-f_{*}$,
$b_{0,\min}=\min_{i\in[d]}b_{0,i}>0$. 
\end{theorem}

\subsection{Proof of Theorem \ref{thm:full-thm}}

For simplicity, in our analysis, we will use $\hn_{t,i}:=\widehat{\nabla}_{i}f(x_{t})$
and $\nabla_{t,i}:=\nabla_{i}f(x_{t})$ to denote the $i$-th coordinate
of the stochastic gradients and gradients at iterate $t$, respectively.
The proof utilizes techniques and follows the strategies \cite{liu2023high},
where the main effort is to adapt the techniques for handling subsets
from the AdaGrad-Norm and AdaGrad-Coordinate proofs in \cite{liu2023high}.
\begin{proof}
We write $\frac{\hn_{t}}{b_{t}}$ to denote $\left(\frac{\hn_{t}}{b_{t}}\right)_{k}=\frac{\hn_{k}f(x_{t})}{b_{t,i}}$
for $k\in\Psi_{i}$ (we will use this notation briefly to show some
steps and will not be crucial in the main analysis). We start with
the smoothness of $f$ and $\Delta_{t}:=f(x_{t})-f_{*}$.

\begin{align}
\Delta_{t+1}-\Delta_{t} & \leq\left\langle \nabla f(x_{t}),x_{t+1}-x_{t}\right\rangle +\frac{L}{2}\norm{x_{t+1}-x_{t}}^{2}\nonumber \\
 & =-\eta\left\langle \nabla_{t},\frac{\hn_{t}}{b_{t}}\right\rangle +\frac{\eta^{2}L}{2}\norm{\frac{\hn_{t}}{b_{t}}}^{2}\\
 & =-\eta\sum_{i=0}^{c-1}\sum_{j\in\Psi_{i}}\frac{\nabla_{t,j}\hn_{t,j}}{b_{t,i}}+\frac{\eta^{2}L}{2}\sum_{i=0}^{c-1}\sum_{j\in\Psi_{i}}\frac{\hn_{t,j}^{2}}{b_{t,i}^{2}}\nonumber \\
 & =-\eta\sum_{i=0}^{c-1}\sum_{j\in\Psi_{i}}\frac{\nabla_{t,j}\left(\xi_{t,j}+\nabla_{t,j}\right)}{b_{t,i}}+\frac{\eta^{2}L}{2}\sum_{i=0}^{c-1}\sum_{j\in\Psi_{i}}\frac{\hn_{t,j}^{2}}{b_{t,i}^{2}}\tag{\ensuremath{\xi_{t,i}=\hn_{t,i}-\nabla_{t,i}}}\nonumber \\
 & =-\eta\sum_{i=0}^{c-1}\sum_{j\in\Psi_{i}}\frac{\nabla_{t,j}^{2}}{b_{t,i}}-\eta\sum_{i=0}^{c-1}\sum_{j\in\Psi_{i}}\frac{\nabla_{t,j}\xi_{t,j}}{b_{t,i}}+\frac{\eta^{2}L}{2}\sum_{i=0}^{c-1}\sum_{j\in\Psi_{i}}\frac{\hn_{t,j}^{2}}{b_{t,i}^{2}}\nonumber \\
 & =-\eta\sum_{i=0}^{c-1}\sum_{j\in\Psi_{i}}\frac{\nabla_{t,j}^{2}}{b_{t,i}}-\eta\sum_{i=0}^{c-1}\sum_{j\in\Psi_{i}}\frac{\nabla_{t,j}\xi_{t,j}}{a_{t,i}}+\eta\sum_{i=0}^{c-1}\sum_{j\in\Psi_{i}}\left(\frac{1}{a_{t,i}}-\frac{1}{b_{t,i}}\right)\nabla_{t,j}\xi_{t,j}+\frac{\eta^{2}L}{2}\sum_{i=0}^{c-1}\sum_{j\in\Psi_{i}}\frac{\hn_{t,j}^{2}}{b_{t,i}^{2}}.\label{eq:basic-stop1-1}
\end{align}
Now, we analyze $\frac{1}{a_{t,i}}-\frac{1}{b_{t,i}}$ for $i=0,1,\dots,c-1$:
\begin{align*}
\left|\frac{1}{a_{t,i}}-\frac{1}{b_{t,i}}\right| & =\left|\frac{b_{t,i}-a_{t,i}}{a_{t,i}b_{t,i}}\right|\\
 & =\left|\frac{b_{t,i}^{2}-a_{t,i}^{2}}{a_{t,i}b_{t,i}\left(b_{t,i}+a_{t,i}\right)}\right|\\
 & =\left|\frac{b_{t-1,i}^{2}+\norm{\hn_{\Psi_{i}}f(x_{t})}^{2}-b_{t-1,i}^{2}-\norm{\nabla_{\Psi_{i}}f(x_{t})}^{2}}{a_{t,i}b_{t,i}\left(b_{t,i}+a_{t,i}\right)}\right|\\
 & =\left|\frac{\norm{\hn_{\Psi_{i}}f(x_{t})}^{2}-\norm{\nabla_{\Psi_{i}}f(x_{t})}^{2}}{a_{t,i}b_{t,i}\left(b_{t,i}+a_{t,i}\right)}\right|\\
 & =\left|\frac{\left(\norm{\hn_{\Psi_{i}}f(x_{t})}-\norm{\nabla_{\Psi_{i}}f(x_{t})}\right)\left(\norm{\hn_{\Psi_{i}}f(x_{t})}+\norm{\nabla_{\Psi_{i}}f(x_{t})}\right)}{a_{t,i}b_{t,i}\left(b_{t,i}+a_{t,i}\right)}\right|.
\end{align*}
Since $b_{t,i}=\sqrt{b_{t-1,i}^{2}+\norm{\hn_{\Psi_{i}}f(x_{t})}^{2}}\geq\norm{\hn_{\Psi_{i}}f(x_{t})}$
and $a_{t,i}=\sqrt{b_{t-1,i}^{2}+\norm{\nabla_{\Psi_{i}}f(x_{t})}^{2}}\geq\norm{\nabla_{\Psi_{i}}f(x_{t})}$,
we have 
\begin{align*}
\left|\frac{1}{a_{t,i}}-\frac{1}{b_{t,i}}\right| & \le\left|\frac{\left(\norm{\hn_{\Psi_{i}}f(x_{t})}-\norm{\nabla_{\Psi_{i}}f(x_{t})}\right)\left(\norm{\hn_{\Psi_{i}}f(x_{t})}+\norm{\nabla_{\Psi_{i}}f(x_{t})}\right)}{a_{t,i}b_{t,i}\left(\norm{\hn_{\Psi_{i}}f(x_{t})}+\norm{\nabla_{\Psi_{i}}f(x_{t})}\right)}\right|\\
 & \leq\left|\frac{\norm{\hn_{\Psi_{i}}f(x_{t})}-\norm{\nabla_{\Psi_{i}}f(x_{t})}}{a_{t,i}b_{t,i}}\right|\\
 & \leq\frac{\norm{\hn_{\Psi_{i}}f(x_{t})-\nabla_{\Psi_{i}}f(x_{t})}}{a_{t,i}b_{t,i}}\\
 & =\frac{\norm{\xi_{t,\Psi_{i}}}}{a_{t,i}b_{t,i}}.
\end{align*}
Hence, we have 
\[
\left|\frac{1}{a_{t,i}}-\frac{1}{b_{t,i}}\right|\leq\frac{\norm{\xi_{t,\Psi_{i}}}}{a_{t,i}b_{t,i}}.
\]
Then from \ref{eq:basic-stop1-1}, taking the absolute value of $\sum_{i=0}^{c-1}\sum_{j\in\Psi_{i}}\left(\frac{1}{a_{t,i}}-\frac{1}{b_{t,i}}\right)\nabla_{t,j}\xi_{t,j}$,
we can bound: 
\begin{align*}
\Delta_{t+1}-\Delta_{t} & \leq-\eta\sum_{i=0}^{c-1}\sum_{j\in\Psi_{i}}\frac{\nabla_{t,j}^{2}}{b_{t,i}}-\eta\sum_{i=0}^{c-1}\sum_{j\in\Psi_{i}}\frac{\nabla_{t,j}\xi_{t,j}}{a_{t,i}}+\eta\sum_{i=0}^{c-1}\sum_{j\in\Psi_{i}}\left|\frac{1}{a_{t,i}}-\frac{1}{b_{t,i}}\right|\left|\nabla_{t,j}\xi_{t,j}\right|+\frac{\eta^{2}L}{2}\sum_{i=0}^{c-1}\sum_{j\in\Psi_{i}}\frac{\hn_{t,j}^{2}}{b_{t,i}^{2}}\\
 & \leq-\eta\sum_{i=0}^{c-1}\sum_{j\in\Psi_{i}}\frac{\nabla_{t,j}^{2}}{b_{t,i}}-\eta\sum_{i=0}^{c-1}\sum_{j\in\Psi_{i}}\frac{\nabla_{t,j}\xi_{t,j}}{a_{t,i}}+\eta\sum_{i=0}^{c-1}\frac{\norm{\xi_{t,\Psi_{i}}}}{a_{t,i}b_{t,i}}\sum_{j\in\Psi_{i}}\left|\nabla_{t,j}\xi_{t,j}\right|+\frac{\eta^{2}L}{2}\sum_{i=0}^{c-1}\sum_{j\in\Psi_{i}}\frac{\hn_{t,j}^{2}}{b_{t,i}^{2}}\\
 & \overset{(1)}{\le}-\eta\sum_{i=0}^{c-1}\sum_{j\in\Psi_{i}}\frac{\nabla_{t,j}^{2}}{b_{t,i}}-\eta\sum_{i=0}^{c-1}\sum_{j\in\Psi_{i}}\frac{\nabla_{t,j}\xi_{t,j}}{a_{t,i}}+\eta\sum_{i=0}^{c-1}\frac{\norm{\xi_{t,\Psi_{i}}}}{a_{t,i}b_{t,i}}\norm{\nabla_{t,\Psi_{i}}}\norm{\xi_{t,\Psi_{i}}}+\frac{\eta^{2}L}{2}\sum_{i=0}^{c-1}\sum_{j\in\Psi_{i}}\frac{\hn_{t,j}^{2}}{b_{t,i}^{2}}\\
 & \le-\eta\sum_{i=0}^{c-1}\sum_{j\in\Psi_{i}}\frac{\nabla_{t,j}^{2}}{b_{t,i}}-\eta\sum_{i=0}^{c-1}\sum_{j\in\Psi_{i}}\frac{\nabla_{t,j}\xi_{t,j}}{a_{t,i}}\\
 & \quad\quad+\eta\sum_{i=0}^{c-1}\norm{\xi_{t,\Psi_{i}}}\left(\frac{\norm{\xi_{t,\Psi_{i}}}^{2}}{2b_{t,i}^{2}}+\frac{\norm{\nabla_{t,\Psi_{i}}}^{2}}{2a_{t,i}^{2}}\right)+\frac{\eta^{2}L}{2}\sum_{i=0}^{c-1}\sum_{j\in\Psi_{i}}\frac{\hn_{t,j}^{2}}{b_{t,i}^{2}},
\end{align*}
where (1) is due to $\sum_{j\in\Psi_{i}}\left|\nabla_{t,j}\xi_{t,j}\right|=\left\langle \left|\nabla_{t,\Psi_{i}}\right|,\left|\xi_{t,\Psi_{i}}\right|\right\rangle \le\norm{\nabla_{t,\Psi_{i}}}\norm{\xi_{t,\Psi_{i}}}$
and $\left|\cdot\right|$ denotes coordinate-wise absolute value when
we apply to vectors. The last inequality is due to $2ab\le a^{2}+b^{2}$.
Now, we can sum both sides for $t=1,\dots,T$ to telescope the LHS:
\begin{align*}
\Delta_{T+1}-\Delta_{1} & \le\sum_{t=1}^{T}\Bigl(-\eta\sum_{i=0}^{c-1}\sum_{j\in\Psi_{i}}\frac{\nabla_{t,j}^{2}}{b_{t,i}}-\eta\sum_{i=0}^{c-1}\sum_{j\in\Psi_{i}}\frac{\nabla_{t,j}\xi_{t,j}}{a_{t,i}}\\
 & +\eta\sum_{i=0}^{c-1}\norm{\xi_{t,\Psi_{i}}}\left(\frac{\norm{\xi_{t,\Psi_{i}}}^{2}}{2b_{t,i}^{2}}+\frac{\norm{\nabla_{t,\Psi_{i}}}^{2}}{2a_{t,i}^{2}}\right)+\frac{\eta^{2}L}{2}\sum_{i=0}^{c-1}\sum_{j\in\Psi_{i}}\frac{\hn_{t,j}^{2}}{b_{t,i}^{2}}\Bigr).
\end{align*}
Rearranging gives

\begin{align*}
\sum_{t=1}^{T}\sum_{i=0}^{c-1}\sum_{j\in\Psi_{i}}\frac{\nabla_{t,j}^{2}}{b_{t,i}} & \le\frac{\Delta_{1}-\Delta_{T+1}}{\eta}-\sum_{t=1}^{T}\underbrace{\sum_{i=0}^{c-1}\sum_{j\in\Psi_{i}}\frac{\nabla_{t,j}\xi_{t,j}}{a_{t,i}}}_{A}\\
 & \quad\quad+\underbrace{\sum_{t=1}^{T}\sum_{i=0}^{c-1}\norm{\xi_{t,\Psi_{i}}}\left(\frac{\norm{\xi_{t,\Psi_{i}}}^{2}}{2b_{t,i}^{2}}+\frac{\norm{\nabla_{t,\Psi_{i}}}^{2}}{2a_{t,i}^{2}}\right)}_{B}+\frac{\eta L}{2}\underbrace{\sum_{t=1}^{T}\sum_{i=0}^{c-1}\sum_{j\in\Psi_{i}}\frac{\hn_{t,j}^{2}}{b_{t,i}^{2}}}_{C}.
\end{align*}
On the LHS, we note that 
\[
\sum_{t=1}^{T}\sum_{i=0}^{c-1}\sum_{j\in\Psi_{i}}\frac{\nabla_{t,j}^{2}}{b_{t,i}}=\sum_{t=1}^{T}\sum_{i=0}^{c-1}\frac{\norm{\nabla_{t,\Psi_{i}}}^{2}}{b_{t,i}}.
\]
We now bound each term separately. It's easiest to bound $C$: $\sum_{t=1}^{T}\sum_{i=0}^{c-1}\sum_{j\in\Psi_{i}}\frac{\hn_{t,j}^{2}}{b_{t,i}^{2}}$:
\begin{align*}
\sum_{t=1}^{T}\sum_{i=0}^{c-1}\sum_{j\in\Psi_{i}}\frac{\hn_{t,j}^{2}}{b_{t,i}^{2}} & =\sum_{i=0}^{c-1}\sum_{t=1}^{T}\sum_{j\in\Psi_{i}}\frac{\hn_{t,j}^{2}}{b_{t,i}^{2}}=\sum_{i=1}^{d}\sum_{t=1}^{T}\frac{b_{t,i}^{2}-b_{t-1,i}^{2}}{b_{t,i}^{2}}\leq\sum_{i=1}^{d}2\log\frac{b_{T,i}}{b_{0,i}}.\\
 & =\sum_{i=0}^{c-1}\sum_{t=1}^{T}\frac{\norm{\hn_{t,\Psi_{i}}}^{2}}{b_{t,i}^{2}}\\
 & =\sum_{i=0}^{c-1}\sum_{t=1}^{T}\frac{b_{t,i}^{2}-b_{t-1,i}^{2}}{b_{t,i}^{2}}\\
 & =\sum_{i=0}^{c-1}\sum_{t=1}^{T}1-\frac{b_{t-1,i}^{2}}{b_{t,i}^{2}}\\
 & \le\sum_{i=0}^{c-1}\sum_{t=1}^{T}\log\frac{b_{t,i}^{2}}{b_{t-1,i}^{2}}\\
 & =2\sum_{i=0}^{c-1}\log\prod_{t=1}^{T}\frac{b_{t,i}}{b_{t-1,i}}\\
 & =2\sum_{i=0}^{c-1}\log\frac{b_{T,i}}{b_{0,i}}.
\end{align*}
We now have a useful inequality 
\begin{equation}
\sum_{t=1}^{T}\frac{\norm{\hn_{t,\Psi_{i}}}^{2}}{b_{t,i}^{2}}\le2\log\frac{b_{T,i}}{b_{0,i}},\ \forall i=0,\dots,c-1.\label{eq:C-bound}
\end{equation}
Next, we deal with $-\sum_{t=1}^{T}\sum_{i=0}^{c-1}\sum_{j\in\Psi_{i}}\frac{\nabla_{t,j}\xi_{t,j}}{a_{t,i}}$
via a martingale argument. Let $\F_{t}:=\sigma\left(\xi_{1},\dots,\xi_{t-1}\right)$
denote the natural filtration. Note that $x_{t}$ is $\F_{t}$-measurable.
For any $w>0$, we have for each $i\in[c]$: 
\begin{align*}
 & \E\left[\exp\left(-w\sum_{j\in\Psi_{i}}\frac{\nabla_{t,j}\xi_{t,j}}{a_{t,i}}-2w^{2}\sum_{j\in\Psi_{i}}\frac{\sigma_{j}^{2}\nabla_{t,j}^{2}}{a_{t,i}^{2}}\right)\mid\F_{t}\right]\\
 & =\exp\left(-2w^{2}\sum_{j\in\Psi_{i}}\frac{\sigma_{j}^{2}\nabla_{t,j}^{2}}{a_{t,i}^{2}}\right)\E\left[\exp\left(-w\sum_{j\in\Psi_{i}}\frac{\nabla_{t,j}\xi_{t,j}}{a_{t,i}}\right)\mid\F_{t}\right]\\
 & \leq1.
\end{align*}
Then a simple inductive argument and using Markov's inequality gives
with probability at least $1-\delta$: 
\[
-w\sum_{t=1}^{T}\sum_{j\in\Psi_{i}}\frac{\nabla_{t,j}\xi_{t,j}}{a_{t,i}}\leq2w^{2}\sum_{t=1}^{T}\sum_{j\in\Psi_{i}}\frac{\sigma_{j}^{2}\nabla_{t,j}^{2}}{a_{t,i}^{2}}+\log\frac{1}{\delta}.
\]
By a union bound across all $c$ subsets, we have w.p. at least $1-c\delta$:
\begin{equation}
-\sum_{t=1}^{T}\sum_{i=0}^{c-1}\sum_{j\in\Psi_{i}}\frac{\nabla_{t,j}\xi_{t,j}}{a_{t,i}}\leq\sum_{t=1}^{T}\sum_{i=0}^{c-1}\sum_{j\in\Psi_{i}}\frac{w\sigma_{j}^{2}\nabla_{t,j}^{2}}{a_{t,i}^{2}}+\frac{c}{w}\log\frac{1}{\delta}.\label{eq:bound-inner-prod-adagrad-sn-1}
\end{equation}
Let's call the event that (\ref{eq:bound-inner-prod-adagrad-sn-1})
happens $E_{1}$. Now, consider $\sum_{t=1}^{T}\sum_{i=0}^{c-1}\sum_{j\in\Psi_{i}}\frac{\nabla_{t,j}^{2}}{a_{t,i}^{2}}$.
We have 
\begin{align*}
\sum_{j\in\Psi_{i}}\frac{\nabla_{t,j}^{2}}{a_{t,i}^{2}} & =\frac{\norm{\nabla_{t,\Psi_{i}}}^{2}}{a_{t,i}^{2}}=\frac{\norm{\nabla_{t,\Psi_{i}}}^{2}}{b_{t-1,i}^{2}+\norm{\nabla_{t,\Psi_{i}}}^{2}}\\
 & \overset{(*)}{\leq}\frac{2\norm{\hn_{t,\Psi_{i}}}^{2}+2\norm{\xi_{t,\Psi_{i}}}^{2}}{b_{t-1,i}^{2}+2\norm{\hn_{t,\Psi_{i}}}^{2}+2\norm{\xi_{t,\Psi_{i}}}^{2}}\\
\frac{\norm{\nabla_{t,\Psi_{i}}}^{2}}{a_{t,i}^{2}} & \le2\frac{\norm{\hn_{t,\Psi_{i}}}^{2}}{b_{t,i}^{2}}+2\frac{\norm{\xi_{t,\Psi_{i}}}^{2}}{b_{t,i}^{2}}.
\end{align*}
For $(*)$ we use the fact that $\frac{x}{c+x}$ is an increasing
function and $\norm{\nabla_{t,\Psi_{i}}}^{2}=\norm{\hn_{t,\Psi_{i}}+\xi_{t,\Psi_{i}}}^{2}\le2\norm{\hn_{t,\Psi_{i}}}^{2}+2\norm{\xi_{t,\Psi_{i}}}^{2}$.
Let $\sigma_{\max}:=\max_{i\in[d]}\sigma_{i}$, then under event $E_{1}$,
we have with probability at least $1-c\delta$:

\begin{align*}
-\sum_{t=1}^{T}\sum_{i=0}^{c-1}\sum_{j\in\Psi_{i}}\frac{\nabla_{t,j}\xi_{t,j}}{a_{t,i}} & \leq\sum_{t=1}^{T}\sum_{i=0}^{c-1}\sum_{j\in\Psi_{i}}\frac{w\sigma_{j}^{2}\nabla_{t,j}^{2}}{a_{t,i}^{2}}+\frac{c}{w}\log\frac{1}{\delta}\\
 & \le w\sigma_{\max}^{2}\sum_{t=1}^{T}\sum_{i=0}^{c-1}\sum_{j\in\Psi_{i}}\frac{\nabla_{t,j}^{2}}{a_{t,i}^{2}}+\frac{c}{w}\log\frac{1}{\delta}\\
 & \le w\sigma_{\max}^{2}\sum_{t=1}^{T}\sum_{i=0}^{c-1}\left(2\frac{\norm{\hn_{t,\Psi_{i}}}^{2}}{b_{t,i}^{2}}+2\frac{\norm{\xi_{t,\Psi_{i}}}^{2}}{b_{t,i}^{2}}\right)+\frac{c}{w}\log\frac{1}{\delta}\\
 & =\underbrace{\sigma_{\max}\sqrt{c\log\frac{1}{\delta}}}_{=:\alpha}\sum_{t=1}^{T}\sum_{i=0}^{c-1}\left(2\frac{\norm{\hn_{t,\Psi_{i}}}^{2}}{b_{t,i}^{2}}+2\frac{\norm{\xi_{t,\Psi_{i}}}^{2}}{b_{t,i}^{2}}\right)+\sigma_{\max}\sqrt{c\log\frac{1}{\delta}}\tag{set \ensuremath{w:=\frac{\sqrt{c\log\frac{1}{\delta}}}{\sigma_{\max}}}}\\
 & =2\alpha\sum_{t=1}^{T}\sum_{i=0}^{c-1}\left(\frac{\norm{\hn_{t,\Psi_{i}}}^{2}}{b_{t,i}^{2}}+\frac{\norm{\xi_{t,\Psi_{i}}}^{2}}{b_{t,i}^{2}}\right)+\alpha.
\end{align*}
where the second to last equality is due to choosing $w=\frac{\sqrt{c\log\frac{1}{\delta}}}{\sigma_{\max}}$
and the last equality is letting $\alpha:=\sigma_{\max}\sqrt{c\log\frac{1}{\delta}}$
for readability.

Let $M_{T,i}=\max_{t\le T}\left|\xi_{t,i}\right|$. Using our notation,
we can define $M_{T,\Psi_{i}}:=\max_{t\le T}\norm{\xi_{t,\Psi_{i}}}$.
Under event $E_{1}$ (and our new bound for $C$), we have that with
probability at least $1-c\delta$: 
\begin{align}
\sum_{t=1}^{T}\sum_{i=0}^{c-1}\frac{\norm{\nabla_{t,\Psi_{i}}}^{2}}{b_{t,i}} & \overset{\text{(C)}}{\le}\frac{\Delta_{1}}{\eta}-\sum_{t=1}^{T}\sum_{i=0}^{c-1}\sum_{j\in\Psi_{i}}\frac{\nabla_{t,j}\xi_{t,j}}{a_{t,i}}+\sum_{t=1}^{T}\sum_{i=0}^{c-1}\norm{\xi_{t,\Psi_{i}}}\left(\frac{\norm{\xi_{t,\Psi_{i}}}^{2}}{2b_{t,i}^{2}}+\frac{\norm{\nabla_{t,\Psi_{i}}}^{2}}{2a_{t,i}^{2}}\right)+\eta L\sum_{i=0}^{c-1}\log\frac{b_{T,i}}{b_{0,i}}\nonumber \\
 & \le\frac{\Delta_{1}}{\eta}-\sum_{t=1}^{T}\sum_{i=0}^{c-1}\sum_{j\in\Psi_{i}}\frac{\nabla_{t,j}\xi_{t,j}}{a_{t,i}}\\
 & \quad\quad+\sum_{t=1}^{T}\sum_{i=0}^{c-1}M_{T,\Psi_{i}}\left(\frac{\norm{\xi_{t,\Psi_{i}}}^{2}}{2b_{t,i}^{2}}+\frac{\norm{\nabla_{t,\Psi_{i}}}^{2}}{2a_{t,i}^{2}}\right)+\eta L\sum_{i=0}^{c-1}\log\frac{b_{T,i}}{b_{0,i}}\tag{def of \ensuremath{M_{T,\Psi_{i}}}}\nonumber \\
 & \overset{\left(E_{1}\right)}{\le}\frac{\Delta_{1}}{\eta}+2\alpha\sum_{t=1}^{T}\sum_{i=0}^{c-1}\left(\underbrace{\frac{\norm{\hn_{t,\Psi_{i}}}^{2}}{b_{t,i}^{2}}}_{\text{bound with (C)}}+\frac{\norm{\xi_{t,\Psi_{i}}}^{2}}{b_{t,i}^{2}}\right)+\alpha+\nonumber \\
 & \quad\quad\quad\sum_{t=1}^{T}\sum_{i=0}^{c-1}M_{T,\Psi_{i}}\left(\frac{\norm{\xi_{t,\Psi_{i}}}^{2}}{2b_{t,i}^{2}}+\frac{\norm{\nabla_{t,\Psi_{i}}}^{2}}{2a_{t,i}^{2}}\right)+\eta L\sum_{i=0}^{c-1}\log\frac{b_{T,i}}{b_{0,i}}\\
 & \overset{\text{(C)}}{\le}\frac{\Delta_{1}}{\eta}+2\alpha\sum_{t=1}^{T}\sum_{i=0}^{c-1}\frac{\norm{\xi_{t,\Psi_{i}}}^{2}}{b_{t,i}^{2}}+\alpha+\nonumber \\
 & \quad\quad\quad\sum_{t=1}^{T}\sum_{i=0}^{c-1}M_{T,\Psi_{i}}\left(\frac{\norm{\xi_{t,\Psi_{i}}}^{2}}{2b_{t,i}^{2}}+\frac{\norm{\nabla_{t,\Psi_{i}}}^{2}}{2a_{t,i}^{2}}\right)+\left(\eta L+4\alpha\right)\sum_{i=0}^{c-1}\log\frac{b_{T,i}}{b_{0,i}}\\
 & \le\frac{\Delta_{1}}{\eta}+2\alpha\sum_{t=1}^{T}\sum_{i=0}^{c-1}\frac{\norm{\xi_{t,\Psi_{i}}}^{2}}{b_{t,i}^{2}}+\alpha+\label{eq:last-pt-before-bounding-mtpsi}\\
 & \quad\quad\quad\sum_{t=1}^{T}\sum_{i=0}^{c-1}M_{T,\Psi_{i}}\frac{\norm{\xi_{t,\Psi_{i}}}^{2}}{2b_{t,i}^{2}}+\sum_{t=1}^{T}\sum_{i=0}^{c-1}M_{T,\Psi_{i}}\frac{\norm{\nabla_{t,\Psi_{i}}}^{2}}{2a_{t,i}^{2}}+\left(\eta L+4\alpha\right)\sum_{i=0}^{c-1}\log\frac{b_{T,i}}{b_{0,i}}.
\end{align}
Let us turn our attention to $M_{T,\Psi_{i}}:=\max_{t\le T}\norm{\xi_{t,\Psi_{i}}}$.
Note that 
\begin{align*}
\Pr\left[\max_{t\in[T]}\norm{\xi_{t,\Psi_{i}}}^{2}\geq A\right] & =\Pr\left[\exp\left(\frac{\max_{t\in[T]}\norm{\xi_{t,\Psi_{i}}}^{2}}{w}\right)\geq\exp\left(\frac{A}{w}\right)\right]\tag{for \ensuremath{w>0}}\\
 & \leq\exp\left(-\frac{A}{w}\right)\E\left[\exp\left(\frac{\max_{t\in[T]}\norm{\xi_{t,\Psi_{i}}}^{2}}{w}\right)\right]\tag{Markov}\\
 & =\exp\left(-\frac{A}{w}\right)\E\left[\max_{t\in[T]}\exp\left(\frac{\norm{\xi_{t,\Psi_{i}}}^{2}}{w}\right)\right]\\
 & \leq\exp\left(-\frac{A}{w}\right)\sum_{t\in[T]}\E\left[\exp\left(\frac{\norm{\xi_{t,\Psi_{i}}}^{2}}{w}\right)\right].
\end{align*}
We have 
\begin{align*}
\E\left[\exp\left(\frac{\norm{\xi_{t,\Psi_{i}}}^{2}}{w}\right)\right] & =\E\left[\exp\left(\frac{\sum_{j\in\Psi_{i}}\xi_{t,j}^{2}}{w}\right)\right]\\
 & =\E\left[\exp\left(\frac{\sum_{j\in\Psi_{i}}\xi_{t,j}^{2}}{w}\right)\right]\\
 & =\E\left[\prod_{j\in\Psi_{i}}\exp\left(\frac{\xi_{t,j}^{2}}{w}\right)\right]\\
 & =\prod_{j\in\Psi_{i}}\E\left[\exp\left(\frac{\xi_{t,j}^{2}}{w}\right)\right].\tag{independence}
\end{align*}
Since sub-gaussianity give us 
\[
\E\left[\exp\left(\lambda^{2}\xi_{t,i}^{2}\right)\right]\leq\exp\left(\lambda^{2}\sigma_{i}^{2}\right),\forall\left|\lambda\right|\leq\frac{1}{\sigma_{i}},\forall i\in\left[d\right],
\]
we have $\E\left[\exp\left(\frac{\xi_{t,j}^{2}}{w}\right)\right]\le\exp\left(\frac{\sigma_{j}^{2}}{w}\right)$
if $\sqrt{\frac{1}{w}}\le\frac{1}{\sigma_{j}}$. We pick $w:=\norm{\sigma_{\Psi_{i}}}^{2}=\sum_{j\in\Psi_{i}}\sigma_{j}^{2}\ge\sigma_{j}^{2},\ \forall j\in\Psi_{i}$
. Hence, we have 
\begin{align}
\E\left[\exp\left(\frac{\norm{\xi_{t,\Psi_{i}}}^{2}}{\norm{\sigma_{\Psi_{i}}}^{2}}\right)\right] & \le\prod_{j\in\Psi_{i}}\exp\left(\frac{\sigma_{j}^{2}}{\norm{\sigma_{\Psi_{i}}}^{2}}\right)\nonumber \\
 & =\exp\left(\frac{\norm{\sigma_{\Psi_{i}}}^{2}}{\norm{\sigma_{\Psi_{i}}}^{2}}\right)=1.\label{eq:subgaussianity-of-xi-t}
\end{align}
We have actually shown that $\xi_{t,\Psi_{i}}$ is a $\norm{\sigma_{\Psi_{i}}}^{2}$-subgaussian
random variable in $\R^{k}$ (see Proposition 2.5.2 in \cite{vershynin2018high}).
This fact will come in handy later. Now, we have 
\begin{align*}
\Pr\left[\max_{t\in[T]}\norm{\xi_{t,\Psi_{i}}}^{2}\geq A\right] & \le\exp\left(-\frac{A}{\norm{\sigma_{\Psi_{i}}}^{2}}\right)\sum_{t\in[T]}\E\left[\exp\left(\frac{\norm{\xi_{t,\Psi_{i}}}^{2}}{\norm{\sigma_{\Psi_{i}}}^{2}}\right)\right]\\
 & =\exp\left(-\frac{A}{\norm{\sigma_{\Psi_{i}}}^{2}}\right)T.
\end{align*}
Setting $\exp\left(-\frac{A}{\norm{\sigma_{\Psi_{i}}}^{2}}\right)T=\delta$
gives $A=\norm{\sigma_{\Psi_{i}}}^{2}\ln T/\delta$. Hence, we have
with probability at least $1-\delta$, 
\begin{equation}
M_{T,\Psi_{i}}=\max_{t\in[T]}\norm{\xi_{t,\Psi_{i}}}^{2}\le\norm{\sigma_{\Psi_{i}}}^{2}\ln T/\delta.\label{eq:bound-on-max-noise-sn}
\end{equation}
Union bounding across all $i=0,1,\dots,c-1$, we have that with probability
at least $1-c\delta$, 
\begin{equation}
M_{T,\Psi_{i}}\le\norm{\sigma_{\Psi_{i}}}^{2}\ln T/\delta,\ \forall i=0,1,\dots,c-1.\label{eq:eq:bound-on-all-max-noise-sn}
\end{equation}
Let us denote the event in (\ref{eq:eq:bound-on-all-max-noise-sn})
by $E_{2}$. Combining it with event $E_{1}$ and starting from \ref{eq:last-pt-before-bounding-mtpsi},
we have that with probability $1-c\delta$: 
\begin{align*}
\sum_{t=1}^{T}\sum_{i=0}^{c-1}\frac{\norm{\nabla_{t,\Psi_{i}}}^{2}}{b_{t,i}} & \le\frac{\Delta_{1}}{\eta}+2\alpha\sum_{t=1}^{T}\sum_{i=0}^{c-1}\frac{\norm{\xi_{t,\Psi_{i}}}^{2}}{b_{t,i}^{2}}+\alpha+\sum_{t=1}^{T}\sum_{i=0}^{c-1}M_{T,\Psi_{i}}\frac{\norm{\xi_{t,\Psi_{i}}}^{2}}{2b_{t,i}^{2}}+\\
 & \quad\quad\sum_{t=1}^{T}\sum_{i=0}^{c-1}M_{T,\Psi_{i}}\frac{\norm{\nabla_{t,\Psi_{i}}}^{2}}{2a_{t,i}^{2}}+\left(\eta L+4\alpha\right)\sum_{i=0}^{c-1}\log\frac{b_{T,i}}{b_{0,i}}\\
 & \le\frac{\Delta_{1}}{\eta}+2\alpha\sum_{t=1}^{T}\sum_{i=0}^{c-1}\frac{\norm{\xi_{t,\Psi_{i}}}^{2}}{b_{t,i}^{2}}+\ln T/\delta\sum_{t=1}^{T}\sum_{i=0}^{c-1}\norm{\sigma_{\Psi_{i}}}^{2}\frac{\norm{\xi_{t,\Psi_{i}}}^{2}}{2b_{t,i}^{2}}+\alpha+\\
 & \quad\quad\ln T/\delta\sum_{t=1}^{T}\sum_{i=0}^{c-1}\norm{\sigma_{\Psi_{i}}}^{2}\frac{\norm{\nabla_{t,\Psi_{i}}}^{2}}{2a_{t,i}^{2}}+\left(\eta L+4\alpha\right)\sum_{i=0}^{c-1}\log\frac{b_{T,i}}{b_{0,i}}\\
 & =\frac{\Delta_{1}}{\eta}+\sum_{i=0}^{c-1}\left(\ln T/\delta\frac{\norm{\sigma_{\Psi_{i}}}^{2}}{2}+2\alpha\right)\sum_{t=1}^{T}\frac{\norm{\xi_{t,\Psi_{i}}}^{2}}{b_{t,i}^{2}}+\alpha+\\
 & \quad\quad\ln T/\delta\sum_{i=0}^{c-1}\frac{\norm{\sigma_{\Psi_{i}}}^{2}}{2}\sum_{t=1}^{T}\frac{\norm{\nabla_{t,\Psi_{i}}}^{2}}{a_{t,i}^{2}}+\left(\eta L+4\alpha\right)\sum_{i=0}^{c-1}\log\frac{b_{T,i}}{b_{0,i}}.
\end{align*}
Recall that $\frac{\norm{\nabla_{t,\Psi_{i}}}^{2}}{a_{t,i}^{2}}\le2\frac{\norm{\hn_{t,\Psi_{i}}}^{2}}{b_{t,i}^{2}}+2\frac{\norm{\xi_{t,\Psi_{i}}}^{2}}{b_{t,i}^{2}}$,
we then have 
\begin{align*}
\ln T/\delta\sum_{i=0}^{c-1}\frac{\norm{\sigma_{\Psi_{i}}}^{2}}{2}\sum_{t=1}^{T}\frac{\norm{\nabla_{t,\Psi_{i}}}^{2}}{a_{t,i}^{2}} & \le\ln T/\delta\sum_{i=0}^{c-1}\frac{\norm{\sigma_{\Psi_{i}}}^{2}}{2}\sum_{t=1}^{T}\left(2\frac{\norm{\hn_{t,\Psi_{i}}}^{2}}{b_{t,i}^{2}}+2\frac{\norm{\xi_{t,\Psi_{i}}}^{2}}{b_{t,i}^{2}}\right)\\
 & =\ln T/\delta\sum_{i=0}^{c-1}\norm{\sigma_{\Psi_{i}}}^{2}\sum_{t=1}^{T}\frac{\norm{\hn_{t,\Psi_{i}}}^{2}}{b_{t,i}^{2}}+\ln T/\delta\sum_{i=0}^{c-1}\norm{\sigma_{\Psi_{i}}}^{2}\sum_{t=1}^{T}\frac{\norm{\xi_{t,\Psi_{i}}}^{2}}{b_{t,i}^{2}}\\
 & \le\ln T/\delta\sum_{i=0}^{c-1}\norm{\sigma_{\Psi_{i}}}^{2}\log\frac{b_{T,i}}{b_{0,i}}+\ln T/\delta\sum_{i=0}^{c-1}\norm{\sigma_{\Psi_{i}}}^{2}\sum_{t=1}^{T}\frac{\norm{\xi_{t,\Psi_{i}}}^{2}}{b_{t,i}^{2}}.\tag{from \ref{eq:C-bound}}
\end{align*}
Hence, we have with probability at least $1-2c\delta$: 
\begin{align}
\sum_{t=1}^{T}\sum_{i=0}^{c-1}\frac{\norm{\nabla_{t,\Psi_{i}}}^{2}}{b_{t,i}} & \le\frac{\Delta_{1}}{\eta}+\sum_{i=0}^{c-1}\left(\ln T/\delta\norm{\sigma_{\Psi_{i}}}^{2}+2\alpha\right)\sum_{t=1}^{T}\frac{\norm{\xi_{t,\Psi_{i}}}^{2}}{b_{t,i}^{2}}\\
 & \quad\quad+\alpha+\sum_{i=0}^{c-1}\ln T/\delta\norm{\sigma_{\Psi_{i}}}^{2}\log\frac{b_{T,i}}{b_{0,i}}+\sum_{i=0}^{c-1}\left(\eta L+4\alpha\right)\log\frac{b_{T,i}}{b_{0,i}}\nonumber \\
 & =\frac{\Delta_{1}}{\eta}+\sum_{i=0}^{c-1}\left(\ln T/\delta\norm{\sigma_{\Psi_{i}}}^{2}+2\alpha\right)\sum_{t=1}^{T}\frac{\norm{\xi_{t,\Psi_{i}}}^{2}}{b_{t,i}^{2}}\\
 & \quad\quad+\alpha+\sum_{i=0}^{c-1}\left(\ln T/\delta\norm{\sigma_{\Psi_{i}}}^{2}+\eta L+4\alpha\right)\log\frac{b_{T,i}}{b_{0,i}}.\label{eq:last-pt-before-bounding-xit-btsq}
\end{align}
Now, we bound $\sum_{t=1}^{T}\frac{\norm{\xi_{t,\Psi_{i}}}^{2}}{b_{t,i}^{2}}$
and $\log\frac{b_{T,i}}{b_{0,i}}$. We need to first lower bound $\sum_{s=1}^{t}\norm{\hn_{t,\Psi_{i}}}^{2}$.
We proceed by noting that 
\begin{align*}
\|\hn_{t,\Psi_{i}}\|^{2} & =\|\nabla_{t,\Psi_{i}}+\xi_{t,\Psi_{i}}\|^{2}\\
 & =\|\nabla_{t,\Psi_{i}}\|^{2}+2\langle\xi_{t,\Psi_{i}},\nabla_{t,\Psi_{i}}\rangle+\|\xi_{t,\Psi_{i}}\|^{2}\\
\Rightarrow\|\nabla_{t,\Psi_{i}}\|-\|\hn_{t,\Psi_{i}}\|^{2}+\|\xi_{t,\Psi_{i}}\|^{2} & =2\langle\xi_{t,\Psi_{i}},\nabla_{t,\Psi_{i}}\rangle.
\end{align*}
Define for $t\in\left\{ 0,1,\cdots,T\right\} $ and some constant
$v_{s}$ to be specified later: 
\begin{align*}
U_{t+1} & =\exp\left(\sum_{s=1}^{t}w_{s}\left(\|\nabla_{s,\Psi_{i}}\|-\|\hn_{s,\Psi_{i}}\|^{2}+\|\xi_{s,\Psi_{i}}\|^{2}\right)-v_{s}\|\nabla_{s,\Psi_{i}}\|^{2}\right)\\
 & =U_{t}\cdot\exp\left(w_{t}\left(\|\nabla_{t,\Psi_{i}}\|-\|\hn_{t,\Psi_{i}}\|^{2}+\|\xi_{t,\Psi_{i}}\|^{2}\right)-v_{t}\|\nabla_{t,\Psi_{i}}\|^{2}\right)\\
 & =U_{t}\cdot\exp\left(w_{t}\left(2\langle\xi_{t,\Psi_{i}},\nabla_{t,\Psi_{i}}\rangle\right)-v_{t}\|\nabla_{t,\Psi_{i}}\|^{2}\right).
\end{align*}
First, note that $U_{t}\in\F_{t}$. We show that $U_{t}$ is a supermartingale
\begin{align*}
\E\left[U_{t+1}\mid\F_{t}\right] & =\E\left[U_{t}\cdot\exp\left(w_{t}\left(2\langle\xi_{t,\Psi_{i}},\nabla_{t,\Psi_{i}}\rangle\right)-v_{t}\|\nabla_{t,\Psi_{i}}\|^{2}\right)\mid\F_{t}\right]\\
 & =U_{t}\exp\left(-v_{t}\|\nabla_{t,\Psi_{i}}\|^{2}\right)\E\left[\exp\left(2w_{t}\langle\xi_{t,\Psi_{i}},\nabla_{t,\Psi_{i}}\rangle\right)\mid\F_{t}\right]\\
 & \overset{(*)}{\le}U_{t}\exp\left(-v_{t}\|\nabla_{t,\Psi_{i}}\|^{2}\right)\E\left[\exp\left(4w_{t}^{2}\norm{\sigma_{\Psi_{i}}}^{2}\|\nabla_{t,\Psi_{i}}\|^{2}\right)\mid\F_{t}\right]\\
 & =U_{t},\tag{\ensuremath{v_{t}}=4\ensuremath{w_{t}^{2}\norm{\sigma_{\Psi_{i}}}^{2}}}
\end{align*}
where $(*)$ is due to Lemma 2.2 of \cite{liu2023high} and the fact
that $\xi_{t,\Psi_{i}}$ is $\norm{\sigma_{\Psi_{i}}}^{2}$-subgaussian
from (\ref{eq:subgaussianity-of-xi-t}). Hence, by Ville's supermartingale
inequality, we have 
\[
\Pr\left[\max_{t\in\left[T+1\right]}U_{t}\geq\delta^{-1}\right]\leq\delta\E\left[U_{1}\right]=\delta.
\]
This implies w.p. $\geq1-\delta$, $\forall0\leq t\leq T$: 
\begin{align*}
\sum_{s=1}^{t}w_{s}\left(\|\nabla_{s,\Psi_{i}}\|-\|\hn_{s,\Psi_{i}}\|^{2}+\|\xi_{s,\Psi_{i}}\|^{2}\right)-v_{s}\|\nabla_{s,\Psi_{i}}\|^{2} & \leq\log\frac{1}{\delta}\\
\implies\sum_{s=1}^{t}\left(w_{s}-4w_{s}^{2}\norm{\sigma_{\Psi_{i}}}^{2}\right)\|\nabla_{s,\Psi_{i}}\|^{2}+\sum_{s=1}^{t}w_{s}\|\xi_{s,\Psi_{i}}\|^{2} & \leq\sum_{s=1}^{t}w_{s}\|\hn_{s,\Psi_{i}}\|^{2}+\log\frac{1}{\delta}\\
\iff\sum_{s=1}^{t}\left(1-4w_{s}\norm{\sigma_{\Psi_{i}}}^{2}\right)\|\nabla_{s,\Psi_{i}}\|^{2}+\sum_{s=1}^{t}\|\xi_{s,\Psi_{i}}\|^{2} & \leq\sum_{s=1}^{t}\|\hn_{s,\Psi_{i}}\|^{2}+\frac{1}{w_{s}}\log\frac{1}{\delta}.
\end{align*}
Set $w_{s}=\frac{1}{4\norm{\sigma_{\Psi_{i}}}^{2}}$ to get 
\begin{equation}
\sum_{s=1}^{t}\|\xi_{s,\Psi_{i}}\|^{2}\leq\sum_{s=1}^{t}\|\hn_{s,\Psi_{i}}\|^{2}+4\norm{\sigma_{\Psi_{i}}}^{2}\log\frac{1}{\delta},\ \forall t\le T.\label{eq:bound-on-sum-xi_t}
\end{equation}
We are now ready to bound $\sum_{t=1}^{T}\frac{\norm{\xi_{t,\Psi_{i}}}^{2}}{b_{t,i}^{2}}$.
Starting by applying (\ref{eq:bound-on-sum-xi_t}), we have that with
probability at least $1-\delta$ 
\begin{align*}
\sum_{t=1}^{T}\frac{\norm{\xi_{t,\Psi_{i}}}^{2}}{b_{t,i}^{2}} & =\sum_{t=1}^{T}\frac{\norm{\xi_{t,\Psi_{i}}}^{2}}{b_{0,i}^{2}+\sum_{s=1}^{t}\norm{\hn_{t,\Psi_{i}}}^{2}}\\
 & \le\sum_{t=1}^{T}\frac{\norm{\xi_{t,\Psi_{i}}}^{2}}{b_{0,i}^{2}+\left(\sum_{s=1}^{t}\|\xi_{s,\Psi_{i}}\|^{2}-4\norm{\sigma_{\Psi_{i}}}^{2}\log\frac{1}{\delta}\right)^{+}}
\end{align*}
where $\left(x\right)^{+}=\max\left\{ x,0\right\} $. Let $\tau=\max\left(\left\{ 0\right\} \cup\left\{ t\in\mathbb{N}_{\leq T}\mid\sum_{s=1}^{t}\norm{\xi_{s,\Psi_{i}}}^{2}\leq2C\right\} \right)$
for some $C\ge0$. We have 
\begin{align*}
\sum_{t=1}^{T}\frac{\norm{\xi_{t,\Psi_{i}}}^{2}}{b_{t,i}^{2}} & =\sum_{t=1}^{\tau}\frac{\norm{\xi_{t,\Psi_{i}}}^{2}}{b_{t,i}^{2}}+\sum_{t=\tau+1}^{T}\frac{\norm{\xi_{t,\Psi_{i}}}^{2}}{b_{0,i}^{2}+\sum_{s=1}^{t}\norm{\hn_{t,\Psi_{i}}}^{2}}\\
 & \le\frac{1}{b_{0,i}^{2}}\sum_{t=1}^{\tau}\norm{\xi_{t,\Psi_{i}}}^{2}+\sum_{t=\tau+1}^{T}\frac{\norm{\xi_{t,\Psi_{i}}}^{2}}{b_{0,i}^{2}+\sum_{s=1}^{t}\|\xi_{s,\Psi_{i}}\|^{2}-4\norm{\sigma_{\Psi_{i}}}^{2}\log\frac{1}{\delta}}\\
 & \le\frac{2C}{b_{0,i}^{2}}+\sum_{t=\tau+1}^{T}\frac{\norm{\xi_{t,\Psi_{i}}}^{2}}{b_{0,i}^{2}+\sum_{s=1}^{t}\|\xi_{s,\Psi_{i}}\|^{2}-4\norm{\sigma_{\Psi_{i}}}^{2}\log\frac{1}{\delta}}.
\end{align*}
Now, since $\frac{\sum_{s=1}^{t}\norm{\xi_{s,\Psi_{i}}}^{2}}{2}\ge C$
for $t>\tau$, we have $b_{0,i}^{2}+\sum_{s=1}^{t}\|\xi_{s,\Psi_{i}}\|^{2}-4\norm{\sigma_{\Psi_{i}}}^{2}\log\frac{1}{\delta}\ge b_{0,i}^{2}-4\norm{\sigma_{\Psi_{i}}}^{2}\log\frac{1}{\delta}+C+\frac{1}{2}\sum_{s=1}^{t}\|\xi_{s,\Psi_{i}}\|^{2}$.
If $b_{0,i}^{2}-4\norm{\sigma_{\Psi_{i}}}^{2}\log\frac{1}{\delta}\ge0$,
then we pick $C=0$ and $b_{0,i}^{2}-4\norm{\sigma_{\Psi_{i}}}^{2}\log\frac{1}{\delta}+C+\frac{1}{2}\sum_{s=1}^{t}\|\xi_{s,\Psi_{i}}\|^{2}\ge\frac{1}{2}\sum_{s=1}^{t}\|\xi_{s,\Psi_{i}}\|^{2}$.
If $b_{0,i}^{2}-4\norm{\sigma_{\Psi_{i}}}^{2}\log\frac{1}{\delta}<0$,
we pick $C=4\norm{\sigma_{\Psi_{i}}}^{2}\log\frac{1}{\delta}-b_{0,i}^{2}>0$,
which gives $b_{0,i}^{2}-4\norm{\sigma_{\Psi_{i}}}^{2}\log\frac{1}{\delta}+C+\frac{1}{2}\sum_{s=1}^{t}\|\xi_{s,\Psi_{i}}\|^{2}\ge\frac{1}{2}\sum_{s=1}^{t}\|\xi_{s,\Psi_{i}}\|^{2}$.
In either case, we have $b_{0,i}^{2}-4\norm{\sigma_{\Psi_{i}}}^{2}\log\frac{1}{\delta}+C+\frac{1}{2}\sum_{s=1}^{t}\|\xi_{s,\Psi_{i}}\|^{2}\ge\frac{1}{2}\sum_{s=1}^{t}\|\xi_{s,\Psi_{i}}\|^{2}$.
Hence, letting $C=\max\left(0,4\norm{\sigma_{\Psi_{i}}}^{2}\log\frac{1}{\delta}-b_{0,i}^{2}\right)\le4\norm{\sigma_{\Psi_{i}}}^{2}\log\frac{1}{\delta}$,
we have w.p. at least $1-\delta$: 
\begin{align*}
\sum_{t=1}^{T}\frac{\norm{\xi_{t,\Psi_{i}}}^{2}}{b_{t,i}^{2}} & \le\frac{2C}{b_{0,i}^{2}}+2\sum_{t=\tau+1}^{T}\frac{\norm{\xi_{t,\Psi_{i}}}^{2}}{\sum_{s=1}^{t}\|\xi_{s,\Psi_{i}}\|^{2}}\\
 & \le\frac{2C}{b_{0,i}^{2}}+2\sum_{t=1}^{T}\frac{\norm{\xi_{t,\Psi_{i}}}^{2}}{\sum_{s=1}^{t}\|\xi_{s,\Psi_{i}}\|^{2}}\\
 & \le\frac{8\norm{\sigma_{\Psi_{i}}}^{2}\log\frac{1}{\delta}}{b_{0,i}^{2}}+2\sum_{t=1}^{T}\frac{\norm{\xi_{t,\Psi_{i}}}^{2}}{\sum_{s=1}^{t}\|\xi_{s,\Psi_{i}}\|^{2}}.
\end{align*}
Let $X_{t}=1+\sum_{s=1}^{t}\norm{\xi_{s,\Psi_{i}}}^{2}=X_{t-1}+\norm{\xi_{t,\Psi_{i}}}^{2}$,
where $X_{0}=1$. Then, 
\begin{align*}
\sum_{t=1}^{T}\frac{\norm{\xi_{t,\Psi_{i}}}^{2}}{\sum_{s=1}^{t}\|\xi_{s,\Psi_{i}}\|^{2}} & =\sum_{t=1}^{T}\frac{X_{t}-X_{t-1}}{X_{t}}=\sum_{t=1}^{T}1-\frac{X_{t-1}}{X_{t}}\\
 & \le\sum_{t=1}^{T}\log\left(\frac{X_{t}}{X_{t-1}}\right)\\
 & =\log\left(\prod_{t=1}^{T}\frac{X_{t}}{X_{t-1}}\right)\\
 & =\log\left(\frac{X_{T}}{X_{0}}\right)=\log\left(1+\sum_{t=1}^{T}\norm{\xi_{s,\Psi_{i}}}^{2}\right).
\end{align*}
Hence, with probability at least $1-\delta$: 
\begin{equation}
\sum_{t=1}^{T}\frac{\norm{\xi_{t,\Psi_{i}}}^{2}}{b_{t,i}^{2}}\le\frac{8\norm{\sigma_{\Psi_{i}}}^{2}\log\frac{1}{\delta}}{b_{0,i}^{2}}+2\log\left(1+\sum_{t=1}^{T}\norm{\xi_{s,\Psi_{i}}}^{2}\right).\label{eq:bound-xit-btsq-before-finish}
\end{equation}
It remains to bound $\sum_{t=1}^{T}\norm{\xi_{s,\Psi_{i}}}^{2}$.
Note that 
\begin{align*}
\Pr\left[\sum_{t=1}^{T}\norm{\xi_{s,\Psi_{i}}}^{2}\geq u\right] & =\Pr\left[\exp\left(\sum_{t=1}^{T}\frac{\|\xi_{s,\Psi_{i}}\|^{2}}{\norm{\sigma_{\Psi_{i}}}^{2}}\right)\geq\exp\left(\frac{u}{\norm{\sigma_{\Psi_{i}}}^{2}}\right)\right]\\
 & \leq\frac{\E\left[\exp\left(\sum_{t=1}^{T}\frac{\|\xi_{s,\Psi_{i}}\|^{2}}{\norm{\sigma_{\Psi_{i}}}^{2}}\right)\right]}{\exp\left(\frac{u}{\norm{\sigma_{\Psi_{i}}}^{2}}\right)}\\
 & \leq\frac{\exp(T)}{\exp\left(\frac{u}{\norm{\sigma_{\Psi_{i}}}^{2}}\right)}\tag{\text{\ensuremath{\xi_{s,\Psi_{i}}}}is \ensuremath{\norm{\sigma_{\Psi_{i}}}^{2}}-subgaussian}
\end{align*}
Choosing $u=\norm{\sigma_{\Psi_{i}}}^{2}T+\norm{\sigma_{\Psi_{i}}}^{2}\log\frac{1}{\delta}$
gives that with probability at least $1-\delta$, we have 
\begin{equation}
\sum_{t=1}^{T}\norm{\xi_{s,\Psi_{i}}}^{2}\le\norm{\sigma_{\Psi_{i}}}^{2}T+\norm{\sigma_{\Psi_{i}}}^{2}\log\frac{1}{\delta}.\label{eq:bound-sum-error-naive}
\end{equation}
Having a high probability bound on the sum of the stochastic error
of the subset-norm, we can combine both events from (\ref{eq:bound-xit-btsq-before-finish})
and (\ref{eq:bound-sum-error-naive}) to get that with probability
at least $1-2\delta$: 
\begin{equation}
\sum_{t=1}^{T}\frac{\norm{\xi_{t,\Psi_{i}}}^{2}}{b_{t,i}^{2}}\le\frac{8\norm{\sigma_{\Psi_{i}}}^{2}\log\frac{1}{\delta}}{b_{0,i}^{2}}+2\log\left(1+\norm{\sigma_{\Psi_{i}}}^{2}T+\norm{\sigma_{\Psi_{i}}}^{2}\log\frac{1}{\delta}\right).\label{eq:final-bound-xit-btsq}
\end{equation}
Then we can also condition on the event that (\ref{eq:final-bound-xit-btsq})
happens and combine it with the event in (\ref{eq:last-pt-before-bounding-xit-btsq})
to get that with probability at least $1-2c\delta$ (assuming $c\ge2$),
we have 
\begin{align}
\sum_{t=1}^{T}\sum_{i=0}^{c-1}\frac{\norm{\nabla_{t,\Psi_{i}}}_{2}^{2}}{b_{t,i}} & \le\frac{\Delta_{1}}{\eta}+\sum_{i=0}^{c-1}\left(\ln T/\delta\norm{\sigma_{\Psi_{i}}}^{2}+2\alpha\right)\sum_{t=1}^{T}\frac{\norm{\xi_{t,\Psi_{i}}}^{2}}{b_{t,i}^{2}}\\
 & \quad\quad+\alpha+\sum_{i=0}^{c-1}\left(\ln T/\delta\norm{\sigma_{\Psi_{i}}}^{2}+\eta L+4\alpha\right)\log\frac{b_{T,i}}{b_{0,i}}\label{eq:prepare-for-final}\\
 & \le\frac{\Delta_{1}}{\eta}+\underbrace{\sum_{i=0}^{c-1}\left(\ln T/\delta\norm{\sigma_{\Psi_{i}}}^{2}+2\alpha\right)\left(\frac{8\norm{\sigma_{\Psi_{i}}}^{2}\log\frac{1}{\delta}}{b_{0,i}^{2}}+2\log\left(1+\norm{\sigma_{\Psi_{i}}}^{2}T+\norm{\sigma_{\Psi_{i}}}^{2}\log\frac{1}{\delta}\right)\right)}_{=:H(\delta)}\label{eq:def-H-delta}\\
 & \quad+\alpha+\sum_{i=0}^{c-1}\left(\ln T/\delta\norm{\sigma_{\Psi_{i}}}^{2}+\eta L+4\alpha\right)\log\frac{b_{T,i}}{b_{0,i}}\nonumber \\
 & =\frac{\Delta_{1}}{\eta}+H(\delta)+\alpha+\sum_{i=0}^{c-1}\left(\ln T/\delta\norm{\sigma_{\Psi_{i}}}^{2}+\eta L+4\alpha\right)\log\frac{b_{T,i}}{b_{0,i}}.
\end{align}
First, note that $b_{T,i}\le\norm{b_{T}}_{1}=\sum_{i=0}^{c-1}b_{T,i}$.
Letting $b_{0,\min}:=\min_{i}b_{0,i}$, we then have 
\begin{align*}
\sum_{i=0}^{c-1}\left(\ln T/\delta\norm{\sigma_{\Psi_{i}}}^{2}+\eta L+4\alpha\right)\log\frac{b_{T,i}}{b_{0,i}} & \le\log\frac{\norm{b_{T}}_{1}}{b_{0,\min}}\sum_{i=0}^{c-1}\left(\ln T/\delta\norm{\sigma_{\Psi_{i}}}^{2}+\eta L+4\alpha\right)\\
 & =\log\frac{\norm{b_{T}}_{1}}{b_{0,\min}}\left(\ln T/\delta\norm{\sigma}_{2}^{2}+c\eta L+4c\alpha\right).
\end{align*}
Now, note the LHS term $\sum_{t=1}^{T}\sum_{i=0}^{c-1}\frac{\norm{\nabla_{t,\Psi_{i}}}_{2}^{2}}{b_{t,i}}$
of (\ref{eq:prepare-for-final}): 
\begin{align*}
\left(\sum_{i=0}^{c-1}\frac{\norm{\nabla_{t,\Psi_{i}}}_{2}^{2}}{b_{t,i}}\right)\left(\sum_{i=0}^{c-1}b_{t,i}\right) & \geq\left(\sum_{i=0}^{c-1}\left\Vert \nabla_{t,\Psi_{i}}\right\Vert _{2}\right)^{2}\ge\sum_{i=0}^{c-1}\left\Vert \nabla_{t,\Psi_{i}}\right\Vert _{2}^{2}=\norm{\nabla_{t}}_{2}^{2}\\
\implies\frac{\norm{\nabla_{t}}_{2}^{2}}{\left(\sum_{i=0}^{c-1}b_{t,i}\right)} & \le\sum_{i=0}^{c-1}\frac{\norm{\nabla_{t,\Psi_{i}}}_{2}^{2}}{b_{t,i}}.
\end{align*}
Now, $\sum_{i=0}^{c-1}b_{t,i}=\sum_{i=0}^{c-1}\left|b_{t,i}\right|=\norm{b_{t}}_{1}$,
so with probability $1-2c\delta$: 
\begin{align}
\sum_{t=1}^{T}\frac{\norm{\nabla_{t}}_{2}^{2}}{\norm{b_{T}}_{1}} & \le\sum_{t=1}^{T}\frac{\norm{\nabla_{t}}_{2}^{2}}{\norm{b_{t}}_{1}}\le\sum_{t=1}^{T}\sum_{i=0}^{c-1}\frac{\norm{\nabla_{t,\Psi_{i}}}_{2}^{2}}{b_{t,i}}\nonumber \\
\implies\sum_{t=1}^{T}\norm{\nabla_{t}}_{2}^{2} & \le\norm{b_{T}}_{1}\sum_{t=1}^{T}\sum_{i=0}^{c-1}\frac{\norm{\nabla_{t,\Psi_{i}}}_{2}^{2}}{b_{t,i}}\nonumber \\
 & \le\norm{b_{T}}_{1}\left(\frac{\Delta_{1}}{\eta}+cH(\delta)+\left(\ln T/\delta\norm{\sigma}_{2}^{2}+c\eta L+4c\alpha\right)\log\frac{\norm{b_{T}}_{1}}{b_{0,\min}}\right)\\
 & \le\norm{b_{T}}_{1}\left(\frac{\Delta_{1}}{\eta}+cH(\delta)+\left(\ln T/\delta\norm{\sigma}_{2}^{2}+c\eta L+4c\alpha\right)\log\frac{\norm{b_{T}}_{1}}{b_{0,\min}}\right).\label{eq:just-before-bound-bT}
\end{align}
It remains to bound $\norm{b_{T}}_{1}$. We start again from smoothness
of $f$:

\begin{align}
\Delta_{t+1}-\Delta_{t} & \leq\left\langle \nabla_{t},x_{t+1}-x_{t}\right\rangle +\frac{L}{2}\norm{x_{t+1}-x_{t}}^{2}\nonumber \\
 & =-\eta\left\langle \nabla_{t},\frac{\hn_{t}}{b_{t}}\right\rangle +\frac{\eta^{2}L}{2}\norm{\frac{\hn_{t}}{b_{t}}}^{2}\nonumber \\
 & =-\eta\left\langle \hn_{t}-\xi_{t},\frac{\hn_{t}}{b_{t}}\right\rangle +\frac{\eta^{2}L}{2}\sum_{i=0}^{c-1}\sum_{j\in\Psi_{i}}\frac{\hn_{t,\Psi_{j}}^{2}}{b_{t,i}^{2}}\nonumber \\
 & =-\eta\left\langle \hn_{t},\frac{\hn_{t}}{b_{t}}\right\rangle +\eta\left\langle \xi_{t},\frac{\hn_{t}}{b_{t}}\right\rangle +\frac{\eta^{2}L}{2}\sum_{i=0}^{c-1}\frac{\norm{\hn_{t,\Psi_{i}}}^{2}}{b_{t,i}^{2}}\nonumber \\
 & =-\eta\sum_{i=0}^{c-1}\sum_{j\in\Psi_{i}}\frac{\hn_{t,j}^{2}}{b_{t,i}}+\eta\sum_{i=0}^{c-1}\sum_{j\in\Psi_{i}}\frac{\xi_{t,j}\hn_{t,j}}{b_{t,i}}+\frac{\eta^{2}L}{2}\sum_{i=0}^{c-1}\frac{\norm{\hn_{t,\Psi_{i}}}^{2}}{b_{t,i}^{2}}\nonumber \\
 & =-\eta\sum_{i=0}^{c-1}\frac{\norm{\hn_{t,\Psi_{i}}}^{2}}{b_{t,i}}+\frac{\eta^{2}L}{2}\sum_{i=0}^{c-1}\frac{\norm{\hn_{t,\Psi_{i}}}^{2}}{b_{t,i}^{2}}+\eta\sum_{i=0}^{c-1}\sum_{j\in\Psi_{i}}\frac{\xi_{t,j}\hn_{t,j}}{b_{t,i}}.
\end{align}
Note that 
\begin{align*}
\sum_{i=0}^{c-1}\sum_{j\in\Psi_{i}}\frac{\xi_{t,j}\hn_{t,j}}{b_{t,i}} & \le\frac{1}{2}\sum_{i=0}^{c-1}\sum_{j\in\Psi_{i}}\frac{\xi_{t,j}^{2}}{b_{t,i}}+\frac{1}{2}\sum_{i=0}^{c-1}\sum_{j\in\Psi_{i}}\frac{\hn_{t,j}^{2}}{b_{t,i}}\\
 & =\frac{1}{2}\sum_{i=0}^{c-1}\sum_{j\in\Psi_{i}}\frac{\xi_{t,j}^{2}}{b_{t,i}}+\frac{1}{2}\sum_{i=0}^{c-1}\frac{\norm{\hn_{t,\Psi_{i}}}^{2}}{b_{t,i}}.
\end{align*}
Plugging back in, we have

\begin{align*}
\Delta_{t+1}-\Delta_{t} & \le-\frac{\eta}{2}\sum_{i=0}^{c-1}\frac{\norm{\hn_{t,\Psi_{i}}}^{2}}{b_{t,i}}+\eta^{2}L\sum_{i=0}^{c-1}\frac{\norm{\hn_{t,\Psi_{i}}}^{2}}{b_{t,i}^{2}}+\frac{\eta}{2}\sum_{i=0}^{c-1}\frac{\norm{\xi_{t,\Psi_{i}}}^{2}}{b_{t,i}}.
\end{align*}
Summing over $T$ and rearranging, we get 
\begin{align*}
\sum_{t=1}^{T}\sum_{i=0}^{c-1}\frac{\norm{\hn_{t,\Psi_{i}}}^{2}}{b_{t,i}} & \leq\frac{2\Delta_{1}}{\eta}+\sum_{t=1}^{T}\sum_{i=0}^{c-1}\frac{\norm{\xi_{t,\Psi_{i}}}^{2}}{b_{t,i}}+2\eta L\sum_{t=1}^{T}\sum_{i=0}^{c-1}\frac{\norm{\hn_{t,\Psi_{i}}}^{2}}{b_{t,i}^{2}}\\
\implies\sum_{t=1}^{T}\sum_{i=0}^{c-1}\frac{\norm{\hn_{t,\Psi_{i}}}^{2}}{b_{t,i}} & \le\frac{4\Delta_{1}}{\eta}+2\sum_{t=1}^{T}\sum_{i=0}^{c-1}\frac{\norm{\xi_{t,\Psi_{i}}}^{2}}{b_{t,i}}+\sum_{t=1}^{T}\sum_{i=0}^{c-1}\left(\frac{4\eta L}{b_{t,i}^{2}}-\frac{1}{b_{t,i}}\right)\norm{\hn_{t,\Psi_{i}}}^{2}.
\end{align*}
We can bound $\sum_{t=1}^{T}\sum_{i=0}^{c-1}\left(\frac{4\eta L}{b_{t,i}^{2}}-\frac{1}{b_{t,i}}\right)\norm{\hn_{t,\Psi_{i}}}^{2}$
as follows. Consider $i\in[c]$. Let $\tau_{i}=\max\left\{ t\leq T\mid b_{t,i}\leq4\eta L\right\} $
so that $t\geq\tau_{i}$ implies $b_{t,i}>4\eta L\iff\frac{4\eta L}{b_{t,i}^{2}}<\frac{1}{b_{t,i}}$:
\begin{align*}
\sum_{t=1}^{T}\left(\frac{4\eta L}{b_{t,i}^{2}}-\frac{1}{b_{t,i}}\right)\norm{\hn_{t,\Psi_{i}}}^{2} & =\sum_{t=1}^{\tau_{i}}\left(\frac{4\eta L}{b_{t,i}^{2}}-\frac{1}{b_{t,i}}\right)\norm{\hn_{t,\Psi_{i}}}^{2}+\sum_{t=\tau_{i}+1}^{T}\left(\underbrace{\frac{4\eta L}{b_{t,i}^{2}}-\frac{1}{b_{t,i}}}_{<0}\right)\norm{\hn_{t,\Psi_{i}}}^{2}\\
 & \le\sum_{t=1}^{\tau_{i}}\left(\frac{4\eta L}{b_{t,i}^{2}}-\frac{1}{b_{t,i}}\right)\norm{\hn_{t,\Psi_{i}}}^{2}\\
 & \le4\eta L\sum_{t=1}^{\tau_{i}}\frac{\norm{\hn_{t,\Psi_{i}}}^{2}}{b_{t,i}^{2}}\\
 & \le8\eta L\log\frac{b_{\tau_{i},i}}{b_{0,i}}\le8\eta L\log\frac{4\eta L}{b_{0,i}}.
\end{align*}
Hence, we have 
\[
\sum_{t=1}^{T}\sum_{i=0}^{c-1}\frac{\norm{\hn_{t,\Psi_{i}}}^{2}}{b_{t,i}}\le\frac{4\Delta_{1}}{\eta}+2\sum_{t=1}^{T}\sum_{i=0}^{c-1}\frac{\norm{\xi_{t,\Psi_{i}}}^{2}}{b_{t,i}}+8\eta L\sum_{i=0}^{c-1}\log\frac{4\eta L}{b_{0,i}}.
\]
Consider the LHS 
\begin{align*}
\sum_{t=1}^{T}\sum_{i=0}^{c-1}\frac{\norm{\hn_{t,\Psi_{i}}}^{2}}{b_{t,i}} & =\sum_{t=1}^{T}\sum_{i=0}^{c-1}\frac{b_{t,i}^{2}-b_{t-1,i}^{2}}{b_{t,i}}=\sum_{t=1}^{T}\sum_{i=0}^{c-1}b_{t,i}-\frac{b_{t-1,i}^{2}}{b_{t,i}}\\
 & \ge\sum_{t=1}^{T}\sum_{i=0}^{c-1}b_{t,i}-\frac{b_{t-1,i}^{2}}{b_{t-1,i}}=\sum_{t=1}^{T}\sum_{i=0}^{c-1}b_{t,i}-b_{t-1,i}\\
 & =\sum_{i=0}^{c-1}\sum_{t=1}^{T}b_{t,i}-b_{t-1,i}=\sum_{i=0}^{c-1}b_{T,i}-b_{0,i}\\
 & =\norm{b_{T}}_{1}-\norm{b_{0}}_{1}.
\end{align*}
Hence, we have 
\[
\norm{b_{T}}_{1}\le\norm{b_{0}}_{1}+\frac{2\Delta_{1}}{\eta}+\sum_{i=0}^{c-1}\sum_{t=1}^{T}\frac{\norm{\xi_{t,\Psi_{i}}}^{2}}{b_{t,i}}+8\eta Lc\log\frac{4\eta L}{b_{0,\min}}.
\]
It remains to bound $\sum_{t=1}^{T}\frac{\norm{\xi_{t,\Psi_{i}}}^{2}}{b_{t,i}}$
for each $i\in[c]$. Recall from (\ref{eq:final-bound-xit-btsq}),
with probability at least $1-\delta$

\begin{align*}
\sum_{s=1}^{t}\|\xi_{t,\Psi_{i}}\|^{2} & \leq\sum_{s=1}^{t}\|\hn_{t,\Psi_{i}}\|^{2}+4\norm{\sigma_{\Psi_{i}}}^{2}\log\frac{1}{\delta},\ \forall t\le T.
\end{align*}
We have with probability at least $1-2c\delta$, 
\begin{align*}
\sum_{t=1}^{T}\frac{\norm{\xi_{t,\Psi_{i}}}^{2}}{b_{t,i}} & =\sum_{t=1}^{T}\frac{\norm{\xi_{t,\Psi_{i}}}^{2}}{\sqrt{b_{0,i}^{2}+\sum_{s=1}^{t}\|\hn_{s,\Psi_{i}}\|^{2}}}\\
 & \overset{(1)}{\leq}\sum_{t=1}^{T}\frac{\xi_{t,i}^{2}}{\sqrt{b_{0,i}^{2}+\left(\sum_{s=1}^{t}\|\xi_{s,\Psi_{i}}\|^{2}-4\norm{\sigma_{\Psi_{i}}}^{2}\log\frac{1}{\delta}\right)^{+}}}\\
 & \leq\frac{8\norm{\sigma_{\Psi_{i}}}^{2}\log\frac{1}{\delta}}{b_{0,i}}+2\sqrt{2}\sqrt{\sum_{s=1}^{T}\|\xi_{s,\Psi_{i}}\|^{2}}\\
 & \overset{(2)}{\leq}\frac{8\norm{\sigma_{\Psi_{i}}}^{2}\log\frac{1}{\delta}}{b_{0,i}}+4\sqrt{\norm{\sigma_{\Psi_{i}}}^{2}T+\norm{\sigma_{\Psi_{i}}}^{2}\log\frac{1}{\delta}},
\end{align*}
where (1) is due to (\ref{eq:bound-on-sum-xi_t}) and (2) is due to
Lemma (\ref{eq:bound-sum-error-naive}). Hence, we have that with
probability at least $1-2c\delta$, 
\begin{align*}
\norm{b_{T}}_{1} & \le\norm{b_{0}}_{1}+\frac{2\Delta_{1}}{\eta}+\sum_{i=0}^{c-1}\frac{8\norm{\sigma_{\Psi_{i}}}^{2}\log\frac{1}{\delta}}{b_{0,i}}+\sum_{i=0}^{c-1}4\sqrt{\norm{\sigma_{\Psi_{i}}}^{2}T+\norm{\sigma_{\Psi_{i}}}^{2}\log\frac{1}{\delta}}+8\eta Lc\log\frac{4\eta L}{b_{0,\min}}\\
 & \le\norm{b_{0}}_{1}+\frac{2\Delta_{1}}{\eta}+\frac{8\log\frac{1}{\delta}}{b_{0,\min}}\sum_{i=0}^{c-1}\norm{\sigma_{\Psi_{i}}}^{2}+4\sqrt{T}\sum_{i=0}^{c-1}\norm{\sigma_{\Psi_{i}}}+\sqrt{\log\frac{1}{\delta}}\sum_{i=0}^{c-1}\norm{\sigma_{\Psi_{i}}}+8\eta Lc\log\frac{4\eta L}{b_{0,\min}}\\
 & =4\sqrt{T}\sum_{i=0}^{c-1}\norm{\sigma_{\Psi_{i}}}+\underbrace{\norm{b_{0}}_{1}+\frac{2\Delta_{1}}{\eta}+\frac{8\log\frac{1}{\delta}}{b_{0,\min}}\norm{\sigma}_{2}^{2}+\sqrt{\log\frac{1}{\delta}}\sum_{i=0}^{c-1}\norm{\sigma_{\Psi_{i}}}+8\eta Lc\log\frac{4\eta L}{b_{0,\min}}}_{=:I(\delta)}.
\end{align*}
Hence, we can combine (\ref{eq:just-before-bound-bT}) with the bound
for $\norm{b_{T}}_{1}$ to get that with probability $1-6c\delta$:
\begin{align*}
\sum_{t=1}^{T}\norm{\nabla_{t}}_{2}^{2} & \le\norm{b_{T}}_{1}\left(\frac{\Delta_{1}}{\eta}+H(\delta)+\left(\ln T/\delta\norm{\sigma}_{2}^{2}+c\eta L+4c\sigma_{\max}\sqrt{c\log\frac{1}{\delta}}\right)\log\frac{\norm{b_{T}}_{1}}{b_{0,\min}}\right)\\
 & \le\left(4\sqrt{T}\sum_{i=0}^{c-1}\norm{\sigma_{\Psi_{i}}}+I(\delta)\right)\cdot\\
 & \quad\left(\frac{\Delta_{1}}{\eta}+H(\delta)+\left(\ln T/\delta\norm{\sigma}_{2}^{2}+c\eta L+4c^{3/2}\sigma_{\max}\sqrt{\log\frac{1}{\delta}}\right)\log\left(\frac{4\sqrt{T}\sum_{i=0}^{c-1}\norm{\sigma_{\Psi_{i}}}+I(\delta)}{b_{0,\min}}\right)\right).
\end{align*}
Dividing both sides by $T$, we get the theorem that with probability
$1-6c\delta$: 
\begin{align*}
\frac{1}{T}\sum_{t=1}^{T}\norm{\nabla_{t}}_{2}^{2} & \le G(\delta)\cdot\left(\frac{4\sum_{i=0}^{c-1}\norm{\sigma_{\Psi_{i}}}}{\sqrt{T}}+\frac{I(\delta)}{T}\right),\ \text{where \ensuremath{G(\delta)} and \ensuremath{I(\delta)} are polylog terms:}\\
G(\delta):= & \frac{\Delta_{1}}{\eta}+H(\delta)+\left(\ln T/\delta\norm{\sigma}_{2}^{2}+c\eta L+4c^{3/2}\sigma_{\max}\sqrt{\log\frac{1}{\delta}}\right)\log\left(\frac{4\sqrt{T}\sum_{i=0}^{c-1}\norm{\sigma_{\Psi_{i}}}+I(\delta)}{b_{0,\min}}\right)\\
I(\delta):= & \norm{b_{0}}_{1}+\frac{2\Delta_{1}}{\eta}+\frac{8\log\frac{1}{\delta}}{b_{0,\min}}\norm{\sigma}_{2}^{2}+\sqrt{\log\frac{1}{\delta}}\sum_{i=0}^{c-1}\norm{\sigma_{\Psi_{i}}}+8\eta Lc\log\frac{4\eta L}{b_{0,\min}}\\
H(\delta):= & \sum_{i=0}^{c-1}\left(\ln\left(T/\delta\right)\norm{\sigma_{\Psi_{i}}}^{2}+2\alpha\right)\left(\frac{8\norm{\sigma_{\Psi_{i}}}^{2}\log\frac{1}{\delta}}{b_{0,i}^{2}}+2\log\left(1+\norm{\sigma_{\Psi_{i}}}^{2}T+\norm{\sigma_{\Psi_{i}}}^{2}\log\frac{1}{\delta}\right)\right).\qedhere
\end{align*}\end{proof}

\end{document}